%% file: 0_main_iclr2023.tex
\documentclass{article} 
\usepackage{iclr2023_conference,times}

\input{math_commands.tex}

\usepackage[utf8]{inputenc} 
\usepackage[T1]{fontenc}    
\usepackage[colorlinks=true,linkcolor=blue,citecolor=green,urlcolor=blue]{hyperref}
\usepackage{url}            
\usepackage{booktabs}       
\usepackage{amsfonts}       
\usepackage{nicefrac}       
\usepackage{microtype}      
\usepackage{xcolor}         
\usepackage{soul}
\usepackage{xspace}
\usepackage{algorithm}
\usepackage[noend]{algpseudocode}
\usepackage{graphicx}

 \usepackage{bbm}
\usepackage{siunitx}
\usepackage{booktabs}
\usepackage{subcaption}
\usepackage{wrapfig}
\usepackage[font=small,labelfont=bf]{caption}

\usepackage[utf8]{inputenc}
\usepackage{pgfplots}
\DeclareUnicodeCharacter{2212}{−}
\usepgfplotslibrary{groupplots,dateplot}
\usetikzlibrary{patterns,shapes.arrows}
\pgfplotsset{compat=newest}

\usepackage{amsmath}
\usepackage{amsthm}
\newtheorem{theorem}{Theorem}
\newtheorem*{theorem*}{Theorem}
\newtheorem{proposition}[theorem]{Proposition}
\newtheorem*{proposition*}{Proposition}
\newtheorem{lemma}[theorem]{Lemma}
\newtheorem*{lemma*}{Lemma}

\newtheorem{definition}[theorem]{Definition}
\newtheorem{assumption}[theorem]{Assumption}

\newtheorem*{remark*}{Remark}

\newcommand{\algoName}{{\texttt{DyART}}\xspace}
\newcommand{\algoNameFull}{Dynamics-Aware Robust Training\xspace}

\title{Exploring and Exploiting Decision Boundary Dynamics for Adversarial Robustness}

\author{
Yuancheng Xu\textsuperscript{$\dag$} 
\quad
Yanchao Sun\textsuperscript{$\dag$} 
\quad
Micah Goldblum\textsuperscript{$\ddag$} 
\quad
Tom Goldstein\textsuperscript{$\dag$}
\quad
Furong Huang\textsuperscript{$\dag$}  \\
\textsuperscript{$\dag$} \text{University of Maryland, College Park} 
\qquad
\textsuperscript{$\ddag$} New York University \\
\texttt{\textsuperscript{$\dag$}\{ycxu,ycs,tomg,furongh\}@umd.edu}
\qquad
\texttt{\textsuperscript{$\ddag$}goldblum@nyu.edu}
}

\iclrfinalcopy 
\begin{document}

\maketitle

\begin{abstract}
\input{s0_abstract}

\end{abstract}

\input{s1_intro}

\input{s5_related}
\input{s2_prelim}
\input{s3_dynamics4DB}

\input{s4_improveRadiusDistribution}
\input{s6_exp}
\input{s7_conc}

\input{s8_acknowledgement.tex}


\bibliography{references}
\bibliographystyle{iclr2023_conference}

\clearpage
\newpage
\appendix
\begin{center}{
\bf \LARGE
    Supplementary Material
}
\end{center}
\input{a4_related}

\input{a1_speed_derivation}
\input{a2_method_details}
\input{a3_experiment}


\newpage

\end{document}

%% file: math_commands.tex
\usepackage{amsmath,amsfonts,bm}

\def\eqref#1{equation~\ref{#1}}

\def\1{\bm{1}}

\DeclareMathAlphabet{\mathsfit}{\encodingdefault}{\sfdefault}{m}{sl}
\SetMathAlphabet{\mathsfit}{bold}{\encodingdefault}{\sfdefault}{bx}{n}

\DeclareMathOperator*{\argmax}{arg\,max}
\DeclareMathOperator*{\argmin}{arg\,min}

\DeclareMathOperator{\sign}{sign}

%% file: s0_abstract.tex
The robustness of a deep classifier can be characterized by its \textit{margins}:
the decision boundary's distances to natural data points. However, it is unclear whether existing robust training methods effectively increase the margin for each vulnerable point during training. 
To understand this, we propose a continuous-time framework for quantifying the relative speed of the decision boundary with respect to each individual point. 
Through visualizing the moving speed of the decision boundary under Adversarial Training, one of the most effective robust training algorithms, a surprising moving-behavior is revealed: the decision boundary moves away from some vulnerable points but simultaneously moves closer to others, decreasing their margins. 
To alleviate these \textit{conflicting dynamics} of the decision boundary, we propose \textit{\algoNameFull} (\algoName), which encourages the decision boundary to engage in movement that prioritizes increasing smaller margins.
In contrast to prior works, \algoName directly operates on the margins rather than their indirect approximations, allowing for more targeted and effective robustness improvement. 
Experiments on the CIFAR-10 and Tiny-ImageNet datasets verify that \algoName alleviates the conflicting dynamics of the decision boundary and obtains improved robustness under various perturbation sizes compared to the state-of-the-art defenses.
Our code is available at \url{https://github.com/Yuancheng-Xu/Dynamics-Aware-Robust-Training}.

%% file: s1_intro.tex
\section{Introduction}\label{sec:intro}

Deep neural networks have exhibited impressive performance in a wide range of applications~\citep{krizhevsky2012imagenet,goodfellow2014generative,he2016deep}. 
However, they have also been shown to be susceptible to adversarial examples, leading to issues in security-critical applications such as autonomous driving and medicine~\citep{szegedy2013intriguing,nguyen2015deep}.  
To alleviate this problem, adversarial training (AT)~\citep{madry2017towards,shafahi2019adversarial,zhang2019theoretically,gowal2020uncovering} was proposed and is one of the most prevalent methods against adversarial attacks. 
Specifically, AT aims to find the worst-case adversarial examples based on some surrogate loss and adds them to the training dataset in order to improve robustness. 

Despite the success of AT, it has been shown that over-parameterized neural networks
still have insufficient model capacity for fitting adversarial training data, partly because AT does not consider the vulnerability difference among data points~\citep{zhang2021geometryaware}. The vulnerability of a data point can be measured by its margin: its distance to the decision boundary. As depicted in Figure~\ref{sfig:demo1}, some data points have smaller margins
and are thus more vulnerable to attacks. Since AT does not directly operate on the margins and it uses a pre-defined perturbation bound for all data points regardless of their vulnerability difference, it is unclear whether the learning algorithm can effectively increase the margin for each vulnerable point. 
Geometrically, we would like to know if the decision boundary moves away from the data points, especially the vulnerable ones. 
As illustrated in Figure~\ref{sfig:demo2}, there can exist conflicting dynamics of the decision boundary:  it moves away from some vulnerable points but simultaneously moves closer to other vulnerable ones during training. This motivates us to ask:

\textbf{Question 1}\quad \textit{Given a training algorithm, how can we analyze the dynamics of the decision boundary with respect to the data points?}

\begin{figure}[!t]
    \centering
    \begin{subfigure}[t]{0.33\textwidth}
        \includegraphics[width=\textwidth]{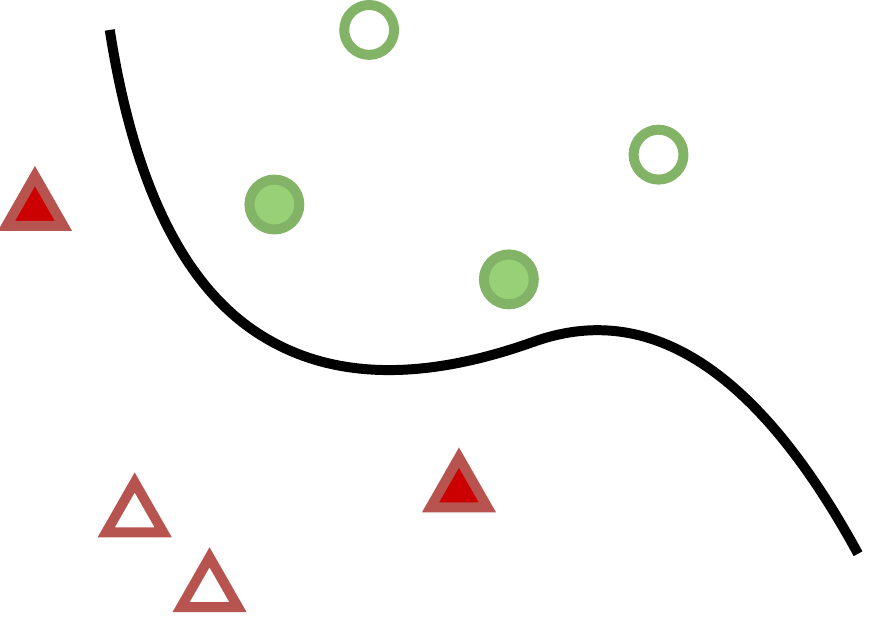}
        \caption{The decision boundary, vulnerable (solid) and robust (hollow) points.}
    \label{sfig:demo1}
    \end{subfigure}
    \hfill
    \begin{subfigure}[t]{0.32\textwidth}
        \includegraphics[width=\textwidth]{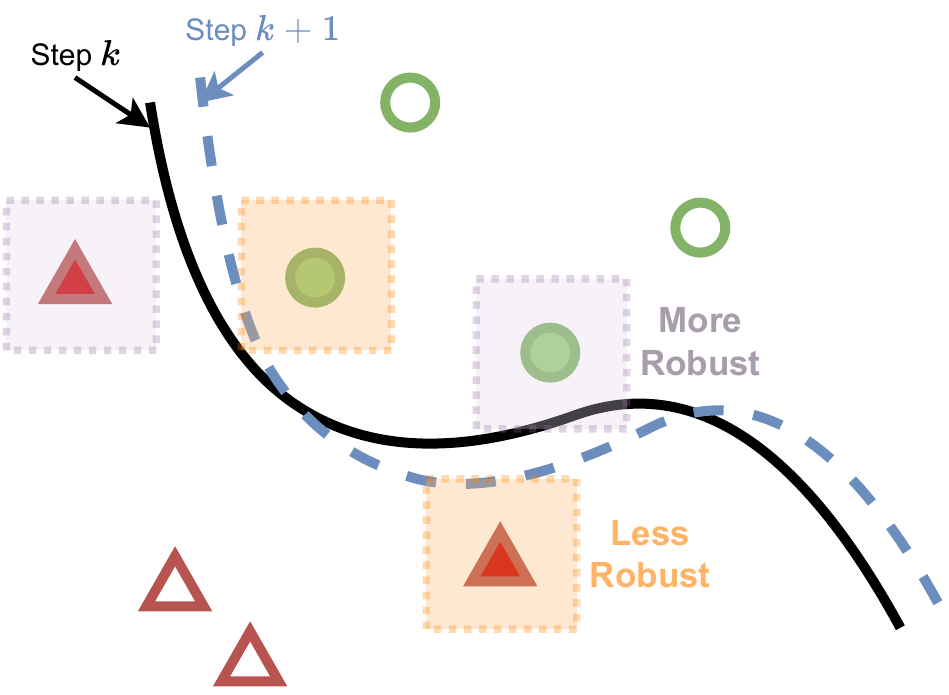}
        \caption{An update with conflicting impacts on robustness.}
    \label{sfig:demo2}
    \end{subfigure}
    \hfill
    \begin{subfigure}[t]{0.33\textwidth}
        \includegraphics[width=\textwidth]{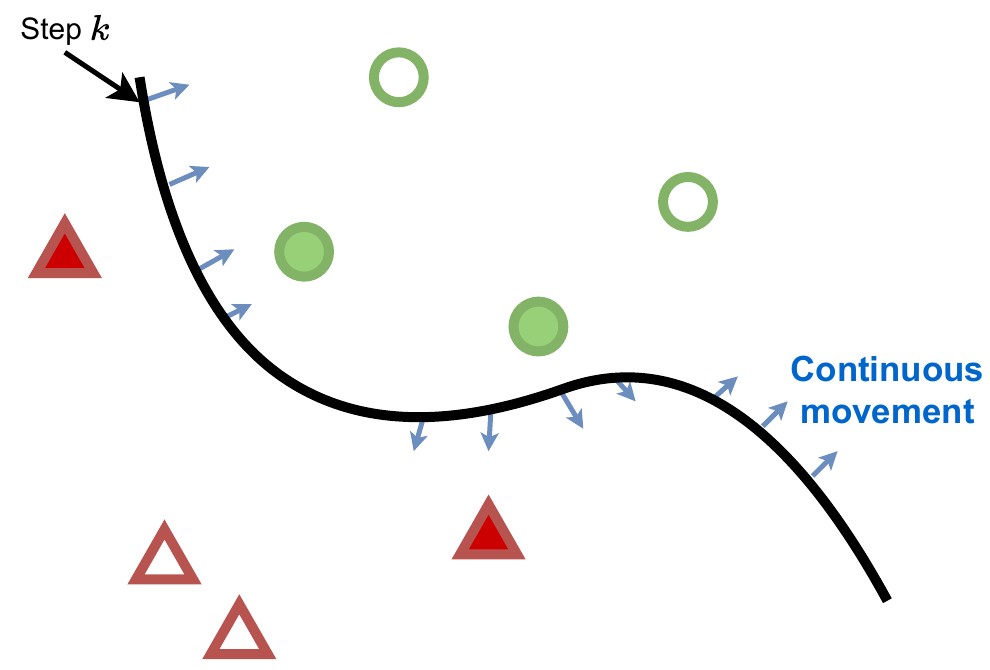}
        \caption{Continuous movement of the decision boundary.}
    \label{sfig:demo3}
    \end{subfigure}
    \caption{The movement of the decision boundary. Red triangles and green circles are data points from two classes. Figure~\ref{sfig:demo1} shows the vulnerability difference among the data points: some are closer to the decision boundary, whereas others are farther from it. In Figure~\ref{sfig:demo2}, the decision boundary after an update 
    moves away from some vulnerable points (made more robust) but simultaneously moves closer to other vulnerable ones (made less robust). 
    Figure~\ref{sfig:demo3} describes the continuous movement of the decision boundary in Figure~\ref{sfig:demo2}.}
    \label{fig:demos}
    \vspace{-1.5em}
\end{figure}

To answer the above question, we propose a continuous-time framework that quantifies the instantaneous movement of the decision boundary as shown in Figure~\ref{sfig:demo3}. 
Specifically, we define the relative speed of the decision boundary w.r.t. a point to be the time derivative of its margin, which can be interpreted as the speed of its closest adversarial example moving away from it. 
We show that the speed can be derived from the training algorithm using a closed-form expression. 

Using the proposed framework, we empirically compute the speed of the decision boundary w.r.t. data points for AT. 
As will be shown in Figure~\ref{fig: AT_VR_scatter_intepretation}, the aforementioned conflicting dynamics of the decision boundary (Figure~\ref{sfig:demo2},\ref{sfig:demo3}) is revealed: the decision boundary  moves \emph{towards} many vulnerable points during training and decrease their margins, directly counteracting the objective of robust training.
The desirable dynamics of the decision boundary, on the other hand, should increase the margins of all vulnerable points. 
This leads to another question:

\textbf{Question 2}\quad \textit{How to design algorithms that encourage the decision boundary to engage in movements that increase margins for vulnerable points, and not decrease them?}

To this end, we propose \textit{\algoNameFull (\algoName)}, which prioritizes moving the decision boundary away from more vulnerable points and increasing their margins. 
Specifically, \algoName directly operates on margins of training data and carefully designs its cost function on margins for more desirable dynamics. 
Note that directly optimizing margins in the input space is technically challenging since it was previously unclear how to compute the gradient of the margin.
In this work, we derive the closed-form expression for the gradient of the margin and present an efficient algorithm to compute it, making gradient descent viable for \algoName. 
In addition, since \algoName directly operates on margins instead of using a pre-defined uniform perturbation bound for training as in AT,  \algoName is naturally robust for a wide range of perturbation sizes $\epsilon$. 
Experimentally, we demonstrate that \algoName mitigates the conflicting dynamics of the decision boundary and achieves improved robustness performance on diverse attacking budgets.

\textbf{Summary of contributions.} \textbf{(1)} We propose a continuous-time framework to study the relative speed of the decision boundary w.r.t. each individual data point and provide a closed-form expression for the speed. 
\textbf{(2)} We visualize the speed of the decision boundary for AT and identify the conflicting dynamics of the decision boundary. 
\textbf{(3)} We present a close-form expression for the gradient of the margin, allowing for direct manipulation of the margin.
\textbf{(4)} We introduce an efficient alternative to compute the margin gradient by replacing the margin with our proposed \emph{soft margin}, a lower bound of the margin whose approximation gap is controllable. 
\textbf{(5)} We propose \algoNameFull (\algoName), which alleviates the conflicting dynamics by carefully designing a cost function on soft margins to prioritize increasing smaller margins. Experiments show that \algoName obtains improved robustness over state-of-the-art defenses on various perturbation sizes.

%% file: s5_related.tex
\section{Related Work}\label{sec:related}

\textbf{Decision boundary analysis.}\quad
Prior works on decision boundary of deep classifiers have studied the small margins in adversarial directions \citep{karimi2019characterizing}, the topology of classification regions \citep{fawzi2018empirical}, the relationship between dataset features and margins \citep{ortiz2020hold} and improved robust training by decreasing the unwarranted increase in the margin along adversarial directions \citep{rade2022reducing}. 
While these works study the static decision boundary of trained models, our work focuses on the decision boundary dynamics during training. 

\textbf{Weighted adversarial training.}\quad 
Adversarial training and its variants \citep{madry2017towards,zhang2019theoretically,wang2019improving,zhang2020attacks} have been proposed to alleviate the adversarial vulnerability of deep learning models. To better utilize the model capacity, weighted adversarial training methods are proposed \citep{zeng2020adversarial,liu2021probabilistic,zhang2021geometryaware} aiming to assign larger weights to more vulnerable points closer to the decision boundary. However, these methods rely on indirect approximations of margins that are not optimal. 
For example, GAIRAT \citep{zhang2021geometryaware} uses the least number of iterations needed to flip the label of an clean example as an surrogate to its margin, which is shown to be 
likely to make wrong judgement on the robustness \citep{liu2021probabilistic}. As another approximation, the logit margin \citep{liu2021probabilistic,zeng2020adversarial} is used but larger logit margin values do not necessarily correspond to larger margins. 
In contrast, our proposed \algoName directly uses margins to characterize the vulnerability of data points.

\textbf{Margin maximization.}\quad 
Increasing the distance between the decision boundary and data points has been discussed in prior works. The authors of \citet{elsayed2018large} propose to maximize the first order Taylor's expansion approximation 
of the margin  at the clean data point, which is inaccurate and computationally prohibitive since it requires computing the Hessian of the classifier. The authors of \citet{Atzmon2019control} propose to maximize the distance between each data point and some point on the decision boundary, which is not the closest one and thus does not increase the margin directly. 
MMA~\citep{Ding2020MMA} uses the uniform average of cross-entropy loss on the closest adversarial examples as the objective function, indirectly increasing the average margins. 
All of these methods maximize the average margin indirectly and do not consider the vulnerability differences among points. 
In contrast, our proposed \algoName will utilize our derived closed-form expression for margin gradient to directly operate on margins and moreover, prioritize increasing smaller margins.

%% file: s2_prelim.tex
\section{Preliminaries and Notations} \label{sec:prelim}
\textbf{Data and model.} \quad We consider a classification task with inputs $x \in \mathcal{X}$ and corresponding labels $y\in\mathcal{Y} = \{1,2,...,K\}$.
A deep classifier parameterized by $ \theta $ is denoted by $f_{\theta}(x) = \argmax_{c\in\mathcal{Y}}z_{\theta}^c(x)$ where $z_{\theta}^c(x)$ is the logit for class $c$. 

\textbf{Decision boundary.} \quad Denote the logit margin for class $ y $ as follows: 

\vspace{-1.0em}

\begin{equation}\label{eq: logit_margin}
    \phi^{y}_{\theta}(x) = z_{\theta}^y(x) - \max_{y'\neq y} z_{\theta}^{y'}(x)
\end{equation}
\vspace{-1.0em}

In this paper, we will use $\phi^{y}_\theta(x)$ and $\phi^{y}(x,\theta)$ interchangeably. Observe that $ x $ is classified as  $ y $ if and only if $ \phi^{y}_{\theta}(x) > 0 $. Therefore,  the zero level set of $ \phi^{y}_{\theta}(x) $, denoted by $ \Gamma_y=\{x: \phi^{y}_{\theta}(x) = 0\} $, is the decision boundary for class $ y $. 
For a correctly classified point $(x,y)$, its margin $R_\theta(x)$ is defined to be the distance from $x$ to the decision boundary for class $y$. That is,
\begin{equation} \label{eq: robustness_radius}
    R_\theta(x) = \min_{\hat{x}} \|\hat{x}-x\|_p \quad \textrm{s.t.}  \quad \phi_\theta^y(\hat{x}) = 0
\end{equation}
where $\|\cdot\|_p$ is the $\ell_p$ norm with $1\leq p \leq \infty$.

\textbf{Difference between logit margin and margin.}\quad  The logit margin $\phi^{y}_{\theta}(x)$ is the gap between the logits values that are \emph{output} by the neural network. On the other hand, the margin $R_\theta(x)$ is the distance from the data point to the decision boundary, and is measured in the \emph{input} space $\mathcal{X}$. 

\textbf{Continuous-time formulation.} \quad To study the instantaneous movement of the decision boundary in Section~\ref{sec: dynamics_of_bdr}, we will use the continuous-time formulation for the optimization on the parameters $\theta$, denoted as $\theta(t)$. 
Let $\theta'(t)$ be the continuous-time description of the update rule of the model parameters. When using gradient descent on a loss function $L$, we have $\theta'(t) = -\nabla_\theta L(\theta(t))$. 

%% file: s3_dynamics4DB.tex
\section{Dynamics of the decision boundary} \label{sec: dynamics_of_bdr}
In this section, we will study the dynamics of the decision boundary in continuous time. We first define its speed w.r.t. each data point, and then provide a closed-form expression for it.
Finally, we visualize the speed of the decision boundary under Adversarial Training and analyze its dynamics. 

\subsection{Speed of the decision boundary}
Consider a correctly classified clean example $(x_i,y_i)$. Our goal is to capture the movement of the decision boundary $\Gamma_{y_i} (t) = \{x:\phi^{y_i}(x,\theta(t)) = 0\}$ w.r.t. $x_i$ as $t$ varies continuously. To this end, we consider the curve of the closest boundary point $\hat{x}_i(t)$ on $\Gamma_{y_i} (t)$ to $x_i$:

\begin{definition}[Curve of the closest boundary point $\hat{x}_i(\cdot)$] \label{def: closest_pt}
Suppose that $(x_i,y_i)$ is correctly classified by $f_{\theta(t)}$ in some time interval $I$. Define the curve of the closest boundary point $\hat{x}_i(\cdot): I \to \mathcal{X}$ as
\begin{equation} \label{eq: closest_pt}
    \hat{x}_i(t) = \argmin\nolimits_{\hat{x}} \|\hat{x}-x_i\|_p \quad \textrm{s.t.}  \quad \phi^{y_i}(\hat{x},\theta(t)) = 0.
\end{equation}
Define the margin of $x_i$ at time $t$ to be $R(x_i,t) = \|\hat{x}_i(t)-x_i\|_p $. 
\end{definition}

\begin{wrapfigure}{r}{0.3\textwidth}
\vspace{-2.5em}
    \centering
    \includegraphics[width=0.3\columnwidth]{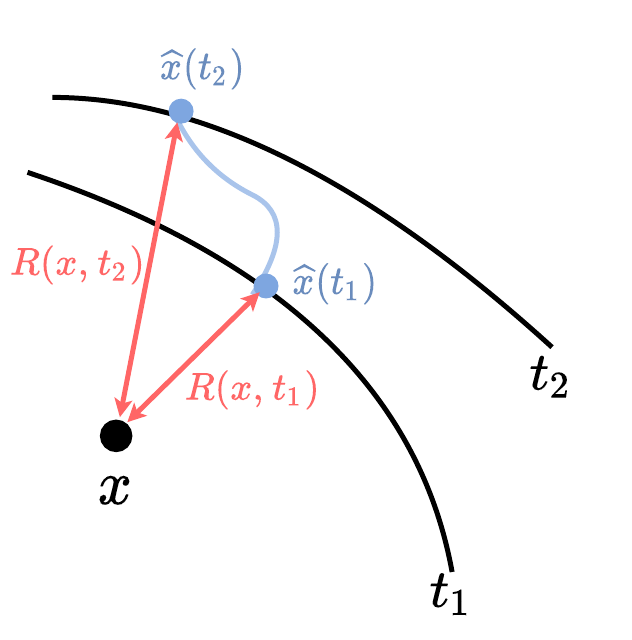}
    \vspace{-1.5em}
    \caption{The curve of the closest boundary point $\hat{x}(t)$ (in blue) of the data point $x$.}
    \label{fig:curve}
\vspace{-2em}
\end{wrapfigure}

An example of the curve of the closest boundary point is depicted in Figure~\ref{fig:curve}.
To understand how the distance between the decision boundary $\Gamma_{y_i} (t)$ and $x_i$ changes, it suffices to focus on the curve of the closest boundary point $\hat{x}_i(t)$. We define the speed of the decision boundary to be the time derivative of the margin as follows:

\begin{definition}[Speed of the decision boundary $s(x_i,t)$] \label{def: speed_of_bdr}
Under the setting of definition \ref{def: closest_pt}, define the speed of the decision boundary w.r.t. $x_i$ as $s(x_i,t) = \frac{d}{dt} R(x_i,t) = \frac{d}{dt}  \|\hat{x}_i(t)-x_i\|_p$.
\end{definition}

Note that the speed $s(x_i,t) > 0$ means that the robustness is improving for $x_i$ at time $t$, which is desirable during robust training. The following proposition gives a closed-form expression for the speed, given a training algorithm $\theta'(t)$.

\begin{proposition}[Closed-form expression of the speed $s(x_i,t)$]\label{prop: speed_expression} Let $\hat{x}_i(t)$ be the curve of the closest boundary point w.r.t. $x_i$. For $1 \leq p \leq \infty$, the speed of decision boundary w.r.t. $x_i$ under $\ell_p$ norm is
\begin{equation} \label{eq: expression_for_speed}
    s(x_i,t) = \frac{1}{\|\nabla_x \phi^{y_i}(\hat{x}_i(t),\theta(t))\|_q} \nabla_\theta \phi^{y_i} (\hat{x}_i(t),\theta(t)) \cdot \theta'(t)
\end{equation}
where $q$ satisfies that $1/q+1/p=1$. In particular, 
$q=1$ when $p=\infty$.
\end{proposition}

\paragraph{Remark.}
Note that \eqref{eq: expression_for_speed} is still valid when $\hat{x}_i(t)$ is just a locally closest boundary point (i.e. a local optimum of~\eqref{eq: closest_pt}). In this case, $s(x_i,t)$ is interpreted as the moving speed of decision boundary around the locally closest boundary point $\hat{x}_i(t)$.
The main assumption for~\eqref{eq: expression_for_speed} is the local differentiability of $\phi^{y_i}(\cdot,\theta(t))$ at $\hat{x}_i(t)$.  The full assumptions, proof and discussions are left to Appendix~\ref{app: closed_expression_speed}. 
Special care has been taken for $p=\infty$ since $\ell_\infty$ norm is not differentiable.

According to~\eqref{eq: expression_for_speed}, the speed $s(x_i,t_0) $ is positive at time $t_0$ when $\nabla_\theta \phi^{y_i} (\hat{x}_i(t_0),\theta(t_0)) \cdot \theta'(t_0) >0$, i.e., $\phi^{y_i}(\hat{x}_i(t),\theta(t))$ increases at time $t_0$, meaning that the boundary point $\hat{x}_i(t_0)$ will be correctly classified after the update. Also, the magnitude of the speed tends to be larger if $\|\nabla_x \phi^{y_i}(\hat{x}_i(t),\theta(t))\|_q$ is smaller, i.e., the margin function $\phi^{y_i}$ is flatter around $\hat{x}_i(t)$. In the remaining part of the paper, we will denote $s(x_i,t)$ and $R(x_i,t)$ as $s(x_i)$ and $R(x_i)$ when the indication is clear.

\textbf{Computing the closest boundary point.}\quad 
We use the Fast Adaptive Boundary Attack (FAB) \citep{croce2020minimally} to compute $\hat{x}_i(t)$ in \eqref{eq: expression_for_speed}.
Specifically, FAB iteratively projects onto the linearly approximated decision boundary with a bias towards the original data point, so that the resulting boundary point is close to the original point $x_i$. Note that FAB only serves as an algorithm to find $\hat{x}_i(t)$, and can be decoupled from the remaining part of the framework. In our experiments we find that FAB can reliably find locally closest boundary points given enough iterations, where the speed expression in~\eqref{eq: expression_for_speed} is still valid.
We give more details of how to check the local optimality condition of~\eqref{eq: closest_pt} and the performance of FAB in Appendix~\ref{app: compute_bdr_exact}. 
Note that in Section~\ref{ssec:soft_bdr}, we will see that directly using FAB is computationally prohibitive for robust training and we will propose a more efficient solution. 
In the next section, we will still use FAB to find closest boundary points for more accurate visualization of the dynamics during adversarial training. 

\subsection{Dynamics of adversarial training} \label{ssec: Dynamics of adversarial training}

In this section, we numerically investigate the dynamics of the decision boundary during adversarial training. We visualize the speed and identify the conflicting dynamics of the decision boundary.

\textbf{Experiment setting.}\quad To study the dynamics of AT in different stages of training where models have different levels of robustness,  we train a ResNet-18 \citep{he2016deep} with group normalization (GN) \citep{wu2018group} on CIFAR-10 using 10-step PGD under $\ell_\infty$ perturbation with $\epsilon=\frac{8}{255}$ from two pretrained models: (1) a partially trained model using natural training with $85\%$ clean accuracy and $0\%$ robust accuracy;  (2) a partially trained model using AT with $75\%$ clean accuracy and $42\%$ robust accuracy under 20-step PGD attack. Note that we replace the batch normalization (BN) layers with GN layers since the decision boundaries are not the same during training and evaluation when BN is used, which can cause confusion when studying the dynamics of the decision boundary. On both pretrained models, we run one iteration of AT on a batch of training data. For correctly classified points in the batch of data, we compute the margins as well as the speed of the decision boundary. 

\input{fig_intepret_train.tex}

\textbf{Conflicting dynamics.}\quad The dynamics of the decision boundary on both pretrained models under AT is shown in Figure~\ref{fig: AT_VR_scatter_intepretation}. The speed values are normalized
so that the maximum absolute value is 1 
for better visualization of their relative magnitude. We can observe that on both pretrained models, the decision boundary has negative speed w.r.t. a significant proportion of non-robust points with $R(x_i)<\frac{8}{255}$. That is, the margins of many vulnerable points \emph{decrease} during adversarial training even though the current update of the model is computed on these points, which counteracts the objective of robust training. In the next section, we will design a dynamics-aware robust training method to mitigate such conflicting dynamics issue.


%% file: fig_intepret_train.tex
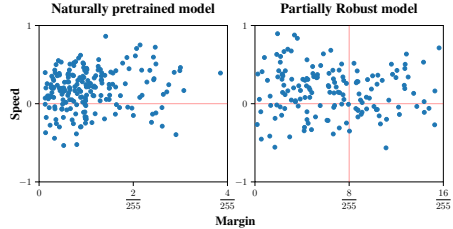
\begin{wrapfigure}{r}{0.43\textwidth}
\vspace{-1.6em}
    \resizebox{\linewidth}{!}{\input{ICLR_fig/AT_trainBatch_scatter.tex}}
    \vspace{-1.8em}
    \caption{Margin-speed plot of AT on a training batch. Among points with margins smaller than $\frac{8}{255}$, there are $28.8\%$ and $29.4\%$ points with negative speed on each pretrained model.}
    \label{fig: AT_VR_scatter_intepretation}
    \vspace{-1.5em}
\end{wrapfigure}

%% file: ICLR_fig/AT_trainBatch_scatter.tex
\pgfplotsset{scaled x ticks=false}
\begin{tikzpicture}

\definecolor{darkgray176}{RGB}{176,176,176}
\definecolor{steelblue31119180}{RGB}{31,119,180}

\begin{groupplot}[group style={group size=2 by 1}]
\nextgroupplot[
tick align=outside,
tick pos=left,
title={\textbf{\Large{Naturally pretrained model}}},
x grid style={darkgray176},
xmin=0, xmax=0.0156862745098039,
xtick style={color=black},
xtick={0,0.00784313725490196,0.0156862745098039},
xticklabels={
  \(\displaystyle 0\),
  \(\displaystyle \frac{2}{255}\),
  \(\displaystyle \frac{4}{255}\)
},
y grid style={darkgray176},
ymin=-1, ymax=1,
ytick style={color=black},
ytick={-1,0,1},
]
\addplot [draw=steelblue31119180, fill=steelblue31119180, mark=*, only marks]
table{%
x  y
0.00365513563156128 0.522027730941772
0.00130671262741089 -0.0385531783103943
0.00169408321380615 -0.297465324401855
0.00405830144882202 0.295788049697876
0.000594556331634521 -0.175970003008842
0.00772950053215027 -0.0565696507692337
0.00204837322235107 -0.534348130226135
0.00392740964889526 0.39963373541832
0.00253874063491821 0.205646127462387
0.00160866975784302 0.551275551319122
0.00412929058074951 0.216001331806183
0.00132414698600769 -0.0656941533088684
0.00584059953689575 0.199532225728035
0.00523719191551208 0.263211578130722
0.00583732128143311 0.0541317202150822
0.00349634885787964 -0.294234603643417
0.00119674205780029 0.20957638323307
0.00371229648590088 0.462108194828033
0.0044865608215332 0.105554953217506
0.00370454788208008 -0.109962940216064
0.00612485408782959 0.248662501573563
0.000983625650405884 0.0896076932549477
0.00355884432792664 0.0858711451292038
0.00246000289916992 0.135955259203911
0.00854279100894928 0.429065853357315
0.00119608640670776 0.197012275457382
0.000508546829223633 0.0715829730033875
0.00169861316680908 0.366393685340881
0.00128442049026489 -0.44206491112709
0.00344550609588623 -0.297440737485886
0.00382387638092041 -0.0087030790746212
0.00944781303405762 -0.133580237627029
0.00922256708145142 0.426731556653976
0.0097278356552124 0.50795590877533
0.0013994574546814 -0.203174158930779
0.00195884704589844 -0.00593761401250958
0.00239330530166626 0.106462702155113
0.00521266460418701 0.395328491926193
0.00515365600585938 0.105280235409737
0.00339275598526001 0.333713620901108
0.00204986333847046 0.202622443437576
0.00174927711486816 0.165166676044464
0.00387933850288391 0.037504069507122
0.00370496511459351 0.509366154670715
0.00144243240356445 -0.254597067832947
0.00272095203399658 0.04994847625494
0.00312724709510803 -0.0133873783051968
0.000842571258544922 -0.0508927404880524
0.00380492210388184 0.425496160984039
0.00520908832550049 0.30387482047081
0.0032273530960083 0.298017352819443
0.00670018792152405 -0.0494971387088299
0.00967150926589966 0.722590804100037
0.00121355056762695 0.080668605864048
0.000508248805999756 0.39771831035614
0.000979602336883545 0.243374913930893
0.01176917552948 0.462575942277908
0.0113976001739502 -0.395810753107071
0.00738269090652466 -0.0443714745342731
0.00252401828765869 0.431876957416534
0.00205540657043457 0.157065227627754
0.00389242172241211 -0.0252180658280849
0.0058315098285675 0.247161418199539
0.00167995691299438 0.387356877326965
0.0102527141571045 -0.28757107257843
0.00972756743431091 -0.118130624294281
0.00951904058456421 0.59536612033844
0.00151592493057251 0.235549867153168
0.00130635499954224 0.429629951715469
0.00289237499237061 0.0307113248854876
0.00277411937713623 0.366536915302277
0.000670135021209717 0.25575989484787
0.00297415256500244 0.217294692993164
0.00141501426696777 0.499899476766586
0.00331538915634155 0.183177590370178
0.00121504068374634 0.297812014818192
0.00384998321533203 -0.190234616398811
0.00286799669265747 0.358793258666992
0.00452080368995667 0.212755888700485
0.00257539749145508 0.178876772522926
0.000508368015289307 -0.0071677933447063
0.00133943557739258 0.393188923597336
0.00458601117134094 0.225286811590195
0.00423163175582886 0.256804883480072
0.0103445649147034 0.119605012238026
0.00639420747756958 0.421515464782715
0.00364762544631958 0.35179728269577
0.00106793642044067 0.142020732164383
0.00381618738174438 -0.105946682393551
0.0120518803596497 0.168387547135353
0.0039023756980896 0.244112238287926
0.00239872932434082 0.37379264831543
0.00840854644775391 0.751062273979187
0.0112490952014923 0.4009889960289
0.00298076868057251 0.166088089346886
0.00393521785736084 0.214325129985809
0.00304365158081055 0.538155555725098
0.00211650133132935 0.274379879236221
0.00190788507461548 0.121314026415348
0.0036376416683197 0.0199595596641302
0.00345897674560547 0.0815150812268257
0.0028873085975647 0.522381722927094
0.00391268730163574 -0.170990139245987
0.00084221363067627 0.0985407754778862
0.00510156154632568 0.437849909067154
0.0022331178188324 0.0166243109852076
0.00167614221572876 0.199978068470955
0.00313889980316162 -0.522422134876251
0.00637674331665039 0.213323876261711
0.00717729330062866 0.255166709423065
0.00183796882629395 0.196522027254105
0.00456094741821289 0.100828923285007
0.00427573919296265 0.257688581943512
0.00429677963256836 0.156259641051292
0.0010027289390564 0.20862390100956
0.00244808197021484 0.386565357446671
0.00398576259613037 0.437589704990387
0.00199347734451294 0.048969104886055
0.00392806529998779 -0.198252648115158
0.00643947720527649 -0.0117085576057434
0.00346148014068604 0.157525643706322
0.00492149591445923 0.477729797363281
0.00731205940246582 0.5908442735672
0.00315779447555542 0.247211933135986
0.00202262401580811 -0.175397917628288
0.0107992887496948 0.242584720253944
0.00183850526809692 0.183054089546204
0.00440126657485962 -0.241792723536491
0.00382190942764282 -0.165464997291565
0.00091630220413208 -0.274809747934341
0.000710129737854004 0.0999437272548676
0.00151711702346802 0.357701629400253
0.000771462917327881 0.0628874972462654
0.0151019692420959 0.391878485679626
0.00256872177124023 0.506556570529938
0.00243467092514038 0.337438881397247
0.00784975290298462 -0.150330156087875
0.00105792284011841 0.351747870445251
0.00677365064620972 0.102934032678604
0.00395190715789795 0.109681598842144
0.00596785545349121 0.404447346925735
0.00957447290420532 0.320337325334549
0.00317665934562683 0.129041865468025
0.00056767463684082 0.201427802443504
0.00230759382247925 -0.0795274972915649
0.00355333089828491 0.319621413946152
0.0028952956199646 0.487619817256927
0.00491750240325928 0.431442379951477
0.00107598304748535 0.200459524989128
0.00695931911468506 0.519033789634705
0.00516733527183533 0.302784502506256
0.00246518850326538 0.27751350402832
0.00954228639602661 0.305725276470184
0.00469499826431274 0.210116386413574
0.00857508182525635 0.107928849756718
0.0049629807472229 0.563814163208008
0.00310665369033813 -0.38573169708252
0.00585436820983887 0.199621304869652
0.00615060329437256 0.391608595848083
0.00261026620864868 -0.0630083158612251
0.00386375188827515 -0.248828828334808
0.00372862815856934 0.335906982421875
0.00216096639633179 0.171767935156822
0.00190842151641846 -0.0646050497889519
0.00547569990158081 0.363124668598175
0.0058780312538147 0.108765907585621
0.00254547595977783 0.0512651987373829
0.0117205381393433 0.422732979059219
0.00189447402954102 -0.0707003101706505
0.00395917892456055 0.452868908643723
0.00329321622848511 0.583771288394928
0.00635173916816711 -0.134121805429459
0.0102325081825256 -0.0499032624065876
0.00554513931274414 0.86217999458313
0.0089765191078186 0.276916146278381
0.00838238000869751 0.0808567404747009
0.00466644763946533 -0.128633946180344
0.00811812281608582 0.7059485912323
0.00116950273513794 0.319729834794998
0.00471752882003784 0.357984155416489
0.00285583734512329 -0.255413472652435
0.00749015808105469 0.447043597698212
0.0038607120513916 0.0685498788952827
0.00936007499694824 -0.238160908222198
0.00560194253921509 0.0552679859101772
0.00508272647857666 0.0370783060789108
0.00235158205032349 0.27973335981369
0.00360283255577087 0.280287951231003
0.000541210174560547 -0.370860695838928
0.00518518686294556 0.353448480367661
0.0034528374671936 0.619617342948914
0.00544685125350952 -0.194261565804482
0.00184226036071777 0.532029211521149
0.00275152921676636 -0.206992596387863
0.00459063053131104 0.329086989164352
0.0106337070465088 0.150884822010994
0.00801569223403931 0.611551940441132
0.00305017828941345 -0.130748763680458
};
\addplot [semithick, red, opacity=0.4]
table {%
0 0
0.0156862745098039 0
};

\nextgroupplot[
tick align=outside,
tick pos=left,
title={\textbf{\Large{Partially Robust model}}},
x grid style={darkgray176},
xmin=0, xmax=0.0627450980392157,
xtick style={color=black},
xtick={0,0.0313725490196078,0.0627450980392157},
xticklabels={
  \(\displaystyle 0\),
  \(\displaystyle \frac{8}{255}\),
  \(\displaystyle \frac{16}{255}\)
},
y grid style={darkgray176},
ymin=-1, ymax=1,
ytick style={color=black},
ytick={-1,0,1},
]
\addplot [draw=steelblue31119180, fill=steelblue31119180, mark=*, only marks]
table{%
x  y
0.0120787620544434 0.674289047718048
0.00752228498458862 0.895295381546021
0.0173213481903076 -0.103290967643261
0.0173263549804688 0.56147962808609
0.0563383102416992 0.0265122763812542
0.00194507837295532 -0.103116601705551
0.0041617751121521 0.298849999904633
0.0335975289344788 -0.0617453046143055
0.0346554517745972 0.124957114458084
0.0227861404418945 0.63747650384903
0.0205302834510803 0.332097738981247
0.0143898129463196 -0.0305895339697599
0.0579545497894287 -0.0591495186090469
0.0093001127243042 0.355203866958618
0.000759243965148926 -0.266114711761475
0.0183109641075134 0.180211946368217
0.0407352447509766 -0.0738517940044403
0.0304546356201172 0.259497284889221
0.0371867418289185 -0.0665352717041969
0.0259443819522858 0.427894443273544
0.026517778635025 -0.0230144243687391
0.0311389565467834 -0.00789366196841002
0.0182619988918304 0.0811871364712715
0.0164907574653625 -0.302245914936066
0.0208012461662292 0.372285842895508
0.0014670193195343 0.0562338344752789
0.0134779810905457 0.023836636915803
0.0171606540679932 0.228304415941238
0.00991690158843994 -0.0616621002554893
0.00493943691253662 -0.414471715688705
0.050433337688446 0.634823799133301
0.0497239828109741 0.235450744628906
0.00920069217681885 0.143730089068413
0.0145842432975769 0.228869587182999
0.0177146196365356 -0.250288903713226
0.000945448875427246 0.379655569791794
0.0176169872283936 0.310396999120712
0.0166523456573486 0.299226760864258
0.0252554416656494 0.683125019073486
0.0343238115310669 -0.0856688097119331
0.0217269062995911 -0.0240955352783203
0.00842905044555664 0.530959784984589
0.0270448327064514 0.213502436876297
0.0280755162239075 -0.31022247672081
0.0262386798858643 -0.0468107014894485
0.0413601994514465 -0.155704766511917
0.0454379320144653 0.505486965179443
0.0109276175498962 0.69685685634613
0.0329769253730774 0.517087817192078
0.03421550989151 -0.0787589401006699
0.0500233173370361 0.50403243303299
0.0265262126922607 0.272267073392868
0.0481072664260864 -0.137751236557961
0.00505244731903076 0.0965651571750641
0.0374917984008789 0.30837070941925
0.0172470808029175 0.322719365358353
0.0201632976531982 -0.168686792254448
0.0406553745269775 0.249719142913818
0.0499173849821091 0.450171172618866
0.0134746432304382 0.0379396975040436
0.0510182976722717 0.341217041015625
0.0289039611816406 -0.00510647147893906
0.00889265537261963 -0.147434651851654
0.046957790851593 0.333819329738617
0.0246496200561523 0.388248592615128
0.0101080536842346 0.421099752187729
0.056393176317215 -0.110466592013836
0.0096631646156311 0.666519343852997
0.0141111016273499 0.836408674716949
0.0180205702781677 0.204066425561905
0.0288708806037903 0.507671892642975
0.0131513476371765 0.877965867519379
0.0109407901763916 0.164353460073471
0.026078462600708 -0.387259244918823
0.0281993746757507 0.141197517514229
0.0150520801544189 0.059478797018528
0.0396338105201721 -0.0353399328887463
0.0192951560020447 0.385244429111481
0.0115717649459839 -0.424117267131805
0.0351501107215881 0.0465284399688244
0.00839865207672119 -0.0405612736940384
0.00905865430831909 0.379240840673447
0.0472913980484009 0.425879597663879
0.00899308919906616 0.312629848718643
0.0528552532196045 0.132481917738914
0.0458186566829681 0.375701993703842
0.0128678977489471 0.202200084924698
0.0151546001434326 0.168911084532738
0.0261233448982239 0.318100243806839
0.0445794463157654 0.0455336906015873
0.0245963037014008 0.217880859971046
0.0188260674476624 -0.331865221261978
0.0342843532562256 -0.169293761253357
0.0296277403831482 0.0540156029164791
0.02399080991745 -0.21193128824234
0.0218909382820129 0.488309562206268
0.0598553903400898 -0.270011305809021
0.0387842655181885 0.187496706843376
0.0189304351806641 0.14123497903347
0.0244303345680237 0.231759086251259
0.0192072987556458 0.475536555051804
0.0339722037315369 0.282884567975998
0.039178192615509 0.356558591127396
0.0404327511787415 -0.076293833553791
0.0343047976493835 0.281413286924362
0.0100252628326416 0.115930676460266
0.0280327796936035 0.429253548383713
0.00663000345230103 -0.55614447593689
0.00852799415588379 0.604398906230927
0.0340432524681091 0.419088929891586
0.0121119618415833 0.0465480349957943
0.0159433484077454 0.311990886926651
0.041784405708313 -0.258487164974213
0.00211071968078613 -0.452877402305603
0.0122518539428711 0.313101291656494
0.0263164937496185 0.137401819229126
0.00725048780441284 0.630468130111694
0.0314773917198181 -0.350294560194016
0.0436456799507141 -0.564812242984772
0.0539036393165588 -0.0348610393702984
0.0409272909164429 0.364612460136414
0.0293255150318146 0.114148110151291
0.0023687481880188 0.40926268696785
0.00314748287200928 0.59053909778595
0.00759753584861755 0.342819541692734
0.00791352987289429 0.421508699655533
0.0438152551651001 -0.199223190546036
0.0423831939697266 0.298526108264923
0.0353854894638062 -0.135289192199707
0.0276824235916138 0.147400766611099
0.0337398052215576 0.267516672611237
0.0243437886238098 -0.214668199419975
0.00909298658370972 0.375927001237869
0.00669723749160767 -0.187032148241997
0.0173200368881226 -0.20353801548481
0.0195884704589844 0.345617085695267
0.0166680216789246 0.0796019956469536
0.0498999953269958 0.19355034828186
0.059927761554718 0.162762433290482
0.0561398267745972 0.53116512298584
0.0096169114112854 0.236524730920792
0.0207808613777161 0.337593674659729
0.00726032257080078 -0.0166352707892656
0.0478167533874512 -0.0436776056885719
0.0304744839668274 -0.24714033305645
0.0295634269714355 0.296645045280457
0.0163841843605042 -0.346481651067734
0.0356666743755341 -0.108997128903866
0.0613210201263428 0.713615596294403
0.0292864739894867 0.342380672693253
0.0381020307540894 0.0712453573942184
0.0249340534210205 0.145622923970222
0.0370959937572479 0.0832581073045731
0.00523161888122559 -0.0218310914933681
0.0357050895690918 0.015103773213923
};
\addplot [semithick, red, opacity=0.4]
table {%
1.38777878078145e-17 0
0.0627450980392157 0
};
\addplot [semithick, red, opacity=0.4]
table {%
0.0313725490196079 -1
0.0313725490196079 1
};
\end{groupplot}

\draw ({$(current bounding box.south west)!0.5!(current bounding box.south east)$}|-{$(current bounding box.south west)$}) node[
  scale=0.75,
  text=black,
  rotate=0.0,
  yshift=-0.5cm
]{\textbf{\LARGE{Margin}}};
\draw ({$(current bounding box.south west)!0!(current bounding box.south east)$}|-{$(current bounding box.south west)!0.55!(current bounding box.north west)$}) node[
  scale=0.75,
  text=black,
  rotate=90.0
]{\textbf{\LARGE{Speed}}};
\end{tikzpicture}

%% file: s4_improveRadiusDistribution.tex
\vspace{-0.5em}
\section{\algoName: Dynamics-aware Robust Training}\label{sec:Robust_training}
\vspace{-0.3em}

In this section, we propose \algoNameFull (\algoName) to mitigate the conflicting dynamics issue.
In Section~\ref{ssec:design_h}, we show how to design an objective function to prioritize improving smaller margins and how to compute the gradient of such objective. In Section~\ref{ssec:soft_bdr}, we overcome the expensive cost of finding the closest boundary points and present the full \algoName algorithm. 

\vspace{-0.3em}
\subsection{Objective for desirable dynamics}\label{ssec:design_h}

We aim to design a loss function $L^{R}(\theta)$ to directly increase the overall margins for effective robustness improvement. 
We propose to use the robustness loss $L^{R}(\theta) := \mathbb{E}_{x}[h(R_\theta(x))]$, where $h: \mathbb{R}\to\mathbb{R}$ is a carefully selected \textit{cost function} that assigns a cost value $h(R)$ to a margin $R$. When designing $h(\cdot)$, it is crucial that 
minimizing $L^{R}(\theta) = \mathbb{E}_{x}[h(R_\theta(x))]$
encourages the \emph{desirable dynamics} of the decision boundary: the decision boundary has positive speed w.r.t. vulnerable points with small margins.

\textbf{Dynamics-aware loss function.}\quad To design such a \emph{dynamics-aware} loss function, the following two properties of the cost function $h$ are desired. 
\textbf{(1) Decreasing} (i.e., $h^\prime(\cdot) < 0$): a point with a smaller margin should be assigned a higher cost value since it is more vulnerable.
\textbf{(2) Convex} (i.e., $h^{\prime\prime}(\cdot) > 0$): the convexity condition helps prioritize improving smaller margins. 
To see this, consider minimizing the loss function $L^{R}(\theta)$ on $m$ points $\{x_i,y_i\}_{i=1}^m$ with margins $\{R_i\}_{i=1}^m$, where the objective becomes $\frac{1}{m}\sum_{i=1}^{m}h(R_{\theta}(x_i))$. At each iteration, the optimizer should update the model to decrease the objective value. Therefore, in continuous-time we have that $\frac{d}{dt}\sum_{i=1}^{m}h(R(x_i,t)) < 0 $. 
Using the chain rule and the definition that the speed $s(x_i,t)=\frac{d}{dt}R(x_i,t)$, we obtain that $\sum_{i=1}^{m}h'(R_i)s(x_i,t) < 0$. 
Given that $h^\prime(\cdot) < 0$, the ideal case is that $s(x_i,t)>0$ for all $x_i$ and thus the sum $\sum_{i=1}^{m}h'(R_i)s(x_i,t) < 0$. In this case, the margins of all data points increase. 
However, due to the existence of conflicting dynamics as described in Section~\ref{ssec: Dynamics of adversarial training}, some points may have negative speed $s(x_i,t)<0$ while $\sum_{i=1}^{m}h'(R_i)s(x_i,t)$ stays negative. 
In the presence of such conflicting dynamics, if $|h^\prime(R_i)|$ is large (i.e., $h^\prime(R_i)$ is small since $h'(\cdot)<0$), it is more likely that $s(x_i,t)>0$ since otherwise it is harder to make $\sum_{i=1}^{m}h'(R_i)s(x_i,t)$ negative. When $h''(\cdot)>0$, a smaller margin $R_i$ has smaller $h'(R_i)$ and thus $s(x_i,t)$ tends to be positive. Therefore, requiring $h^{\prime\prime}(\cdot) > 0$ incentivizes the decision boundary to have positive speed w.r.t. points with smaller margins. 

How to design the optimal $h(\cdot)$ is still an open problem. In this paper, we propose to use 
\begin{equation}\label{eq:exp_h}
    h(R) =\Big\{\begin{matrix}
      \frac{1}{\alpha}\exp(-\alpha R), &   R < r_0\\
      0, & \text{otherwise}
    \end{matrix} 
\end{equation}
where the hyperparameters $\alpha >0$ and $r_0 >0$. Larger $\alpha$ values prioritize improving smaller margins. The threshold $r_0$ is used to avoid training on points that are too far away from the clean data points. 

\textbf{Difficulties of computing margin gradient.}\quad
Directly minimizing $\mathbb{E}_{x}[h(R_\theta(x))]$ through gradient-based optimization methods requires computing the gradient $\nabla_\theta h(R_{\theta}(x_i))$ w.r.t. the model parameters. However, it was previously unclear how to compute $\nabla_\theta h(R_{\theta}(x_i))$, which partly explains why previous works did not directly operate on the margins. The difficulty of computing $\nabla_\theta h(R_{\theta}(x_i))$  lies in the fact that $R_{\theta}(x_i)$, as defined in \eqref{eq: robustness_radius}, involves a constrained optimization problem and thus its gradient $\nabla_\theta R_{\theta}(x_i)$ cannot be computed straightforwardly. An additional challenge is dealing with the non-smoothness of the $\ell_\infty$ norm, which is widely used in the robust training literature.

\textbf{Our solution.}\quad We overcome the above challenges and provide the following close-form expression for the gradient of any smooth function of the margin. The proof is provided in Appendix~\ref{app: closed_expression_speed}.




\begin{theorem}[The gradient $\nabla_\theta h(R_{\theta}(x_i))$ of any smooth function of the margin] \label{thm:closed-form-gradient} For $1 \leq p \leq \infty$,
\begin{equation} \label{eq: expression_for_grad_hR}
    \nabla_\theta h(R_{\theta}(x_i)) = \frac{h'(R_{\theta}(x_i))}{\|\nabla_x \phi^{y_i}(\hat{x}_i,\theta)\|_q} \nabla_\theta \phi^{y_i} (\hat{x}_i,\theta)
\end{equation}
where $q$ satisfies that $1/q+1/p=1$. In particular, $q=1$ when $p=\infty$.
\end{theorem}


Note that another expression for the margin gradient (i.e., $h$ is the identity function in \eqref{eq: expression_for_grad_hR}) was derived in MMA \citep{Ding2020MMA}, with the following distinctions from us: (a) The expression in MMA does not apply to the $\ell_\infty$ norm while ours does. (b) The coefficient $\frac{1}{\|\nabla_x \phi^{y_i}(\hat{x}_i,\theta)\|_q}$ in our expression is more informative and simpler to compute. (c) MMA treats the aforementioned coefficient as a constant during training, and therefore does not properly follow the margin gradient. 

Computing $\nabla_\theta h(R_{\theta}(x_i))$ requires computing the closest boundary points $\hat{x}_i$, which can be computationally prohibitive for robust training. In the next section, we propose to use the closest point $\hat{x}_i^{\text{soft}}$ on the \emph{soft} decision boundary instead, whose quality of approximation to the exact decision boundary is controllable and computational cost is tractable. We will then present the full \algoName algorithm.

\subsection{Efficient Robust Training}\label{ssec:soft_bdr}

\textbf{Directly finding the closest boundary points is expensive.}\quad  
Since the closest boundary point $\hat{x}_i$ can be on the decision boundary between the true class and any other class, FAB needs to form a linear approximation of the decision boundary between the true class and every other class at each iteration. This requires computing the Jacobian of the classifier, and the computational cost scales linearly with the number of classes $K$ \citep{croce2020reliable}. Therefore, finding the closest points on the exact decision boundary is computationally prohibitive for robust training in multi-class classification settings, especially when $K$ is large. To remedy this, we propose to instead use the closest points on the \emph{soft decision boundary} as elaborated below.

\textbf{Soft decision boundary.}
We replace the maximum operator in logit margin (\eqref{eq: logit_margin}) with a smoothed maximum controlled by the temperature $\beta>0$. Specifically, we define the soft logit margin of the class $ y $ as 

\begin{equation} \label{eq: smoothed_margin}
    \Phi^{y}_{\theta} (x;\beta) = z_{\theta}^y(x) - \frac{1}{\beta} \log\sum_{y'\neq y} \exp(\beta z_{\theta}^{y'}(x))
\end{equation}

The soft decision boundary is defined as the zero level set of the soft logit margin: $ \Gamma_y^{\text{soft}}=\{x: \Phi^{y}_{\theta}(x;\beta) = 0\} $.  For $x_i$ with $\Phi^{y_i}_{\theta}(x_i;\beta)>0$, the closest soft boundary point is defined as 

\begin{equation} \label{eq: soft_robustness_radius}
    \hat{x}_i^{\text{soft}} = \argmin\nolimits_{\hat{x}} \|\hat{x}-x_i\|_p \quad \textrm{s.t.}  \quad\Phi^{y}_{\theta}(\hat{x};\beta) = 0,
\end{equation}

and the \emph{soft margin} is defined as $R_{\theta}^{\text{soft}}(x_i) = \|\hat{x}_i^{\text{soft}}-x_i\|_p $. Note that we do not define $R^{\text{soft}}(x_i)$ when $\Phi^{y_i}_{\theta}(x_i;\beta)<0$. The relationship between the exact and soft decision boundary is characterized by the following proposition:

\begin{proposition}
    If $x$ is on the soft decision boundary $\Gamma_y^{\text{soft}}$, i.e. $\Phi^{y}_{\theta} (x;\beta) = 0$, then $\frac{\log(K-1)}{\beta} \geq \phi^{y}_{\theta} (x) \geq 0$. Moreover, when $\Phi^{y_i}_{\theta}(x_i;\beta)>0$, then $R_{\theta}^{\text{soft}}(x_i) \leq R_{\theta}(x_i)$.
\end{proposition}

In other words, the soft decision boundary is always closer to $x_i$ than the exact decision boundary as shown in Figure~\ref{fig:soft_bdr}. Moreover,
the quality of approximation to the exact decision boundary is controllable: the gap between the two decreases as $\beta$ increases and vanishes when $\beta \to \infty$.
Therefore, increasing the soft margins will increase the exact margins as well. 

\begin{wrapfigure}{r}{0.3\textwidth}
\vspace{-2.5em}
    \centering
    \includegraphics[width=0.3\columnwidth]{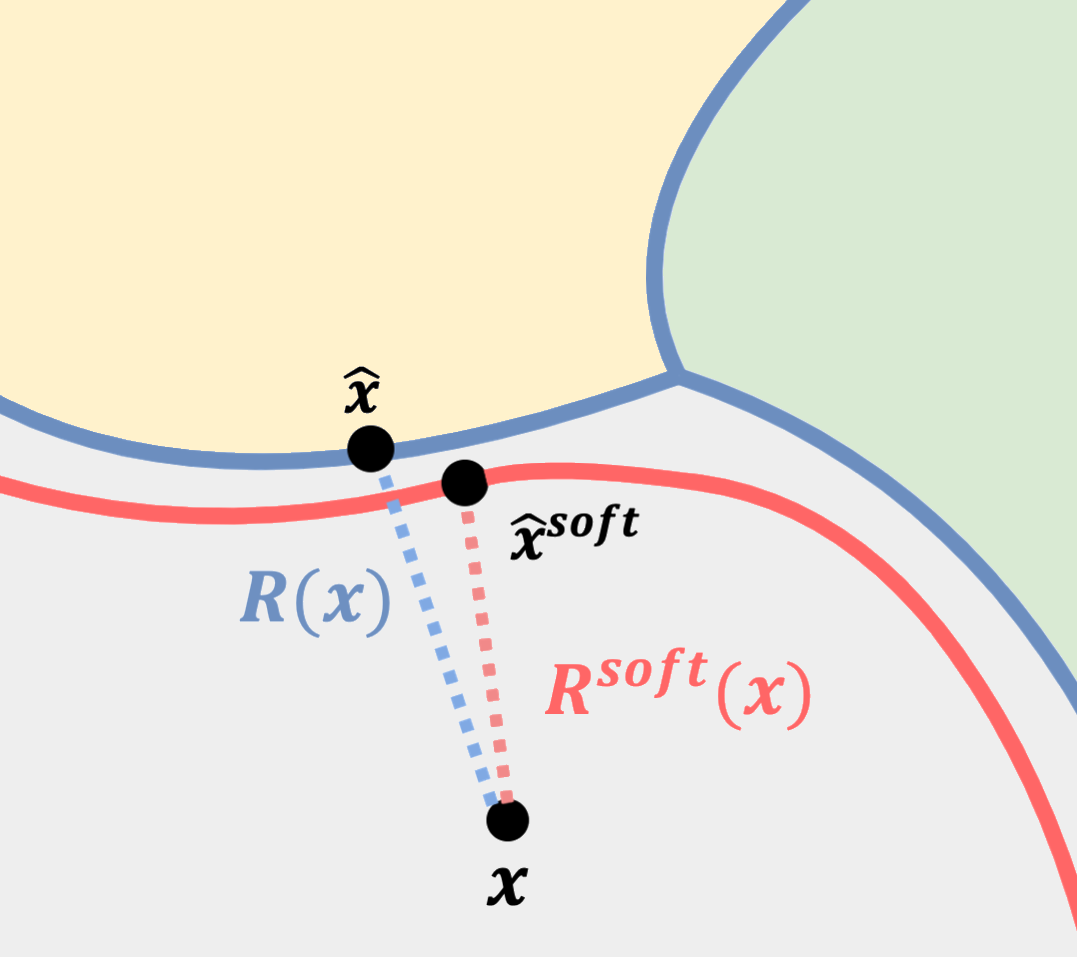}
    \vspace{-1.5em}
    \caption{Exact decision boundary (in blue) for three classes (yellow, green and grey regions) and the soft decision boundary  (in red) for the class of $x$.}
    \label{fig:soft_bdr}
\vspace{-1em}
\end{wrapfigure}

\textbf{Benefits of the soft decision boundary.}
\emph{(1) Computational efficiency.} Using the smoothed max operator, finding the closest soft boundary point does not require forming linear approximations for the decision boundary between the true class and all the other classes anymore. We adapt the FAB method to solve for $\hat{x}_i^{\text{soft}}$ (see details in Appendix~\ref{app: compute_bdr_soft}).  
Its computational cost for each iteration is independent of the number of classes $K$, which is the same as the PGD training. 
\emph{(2) Effective information usage.} Another benefit of using the smoothed max operator in soft logit margin is that, unlike the logit margin $\phi^{y_i}_{\theta}(x_i)$, the soft logit margin $\Phi^{y_i}_{\theta} (x_i;\beta)$ contains information of logit values of all classes. Therefore, the information of all classes is used at each iteration when finding $\hat{x}_i^{\text{soft}}$.

\textbf{Loss function and its gradient.} The overall objective of \algoName is to increase the soft margins and also achieve high clean accuracy. Denote a training data batch $\mathcal{B}$ of size $n$ and $\mathcal{B}^{+}_\theta$ of size $m$ to be $\{i\in\mathcal{B}: \Phi^{y_i}_{\theta}(x_i;\beta)>0\}$. Our proposed method \algoName uses the following loss function 

\begin{equation*}\label{final_loss}
    L_\theta(\mathcal{B}) = \frac{1}{n}\sum\nolimits_{i\in \mathcal{B}}l(x_i,y_i) + \frac{\lambda}{n}\sum\nolimits_{i\in \mathcal{B}^{+}_\theta} h(R_{\theta}^{\text{soft}}(x_i))
\end{equation*}

where the first term is the average cross-entropy loss on natural data points and the second term is for increasing the soft margins. The hyperparameter $\lambda$ balances the trade-off between clean and robust accuracy. By applying \eqref{eq: expression_for_grad_hR}, the gradient of the objective can be computed as

\begin{equation}\label{grad_loss}
    \nabla_\theta L_\theta(\mathcal{B}) = \frac{1}{n}\sum\nolimits_{i\in \mathcal{B}}\nabla_\theta l(x_i,y_i) + \frac{\lambda}{n}\sum\nolimits_{i\in \mathcal{B}^{+}_\theta} \frac{h'(R_{\theta}^{\text{soft}}(x_i))}{\|\nabla_x \Phi_{\theta}^{y_i} ( \hat{x}_i^{\text{soft}};\beta)\|_q} \nabla_\theta \Phi_{\theta}^{y_i} ( \hat{x}_i^{\text{soft}};\beta)
\end{equation}

Since the soft margin $R_{\theta}^{\text{soft}}(x_i)$ is only defined for $x_i$ with $\Phi^{y_i}_{\theta}(x_i;\beta)>0$, \algoName requires training on a pretrained model with a relatively high proportion of points with positive $\Phi^{y_i}_{\theta}$ values. In practice, we find that a burn-in period of several epochs of natural training is enough for such pretrained model.

\textbf{Novelty compared with prior works.} 
\textbf{(1)} \textbf{Direct} and \textbf{efficient} manipulation of the margin. 
\textbf{(1a)} In contrast to prior works that depend on indirect approximations of margins, \algoName directly operates on margins by utilizing the closed-from expression for the margin gradient in~\eqref{eq: expression_for_grad_hR} whose computation was previously unclear.
\textbf{(1b)} We significantly reduce the computational cost of computing margins and its gradients by introducing the \emph{soft margin}, a lower bound of the margin whose approximation gap is controllable. 
\textbf{(2)} \textbf{Prioritizing} the growth of smaller margins by carefully designing the cost function $h(\cdot)$ to mitigate the conflicting dynamics. 
Therefore, \algoName achieves more targeted and effective robustness improvement by directly and efficiently operating on margins as well as prioritizing the growth of smaller margins.

%% file: s6_exp.tex
\section{Experiments}\label{sec:exp}
In this section, we empirically evaluate the effectiveness and performance of the proposed \algoName on the CIFAR-10 \citep{krizhevsky2009learning} and Tiny-ImageNet \citep{deng2009imagenet} datasets. In Section~\ref{sec:exp_robust}, we evaluate the adversarial robustness of \algoName and compare it with several state-of-the-art baselines. In Section~\ref{sec:exp_effect}, we visualize the dynamics of the decision boundary under \algoName and analyze how it alleviates the conflicting dynamics. 


\subsection{Robustness Evaluation}
\label{sec:exp_robust}

\paragraph{Architectures and training parameters.} 
In the experiments on the CIFAR-10 dataset, we use the Wide Residual Network  \citep{zagoruyko2016wide} with depth 28 and width factor 10 (WRN-28-10).
On the Tiny-ImageNet dataset, we use pre-activation ResNet-18 \citep{he2016identity}. 
Models are trained using stochastic gradient descent with momentum 0.9 and weight decay 0.0005 with batch size $256$ for $200$ epochs on CIFAR-10 and $100$ epochs on Tiny-ImageNet. 
We use stochastic weight averaging \citep{izmailov2018averaging} with a decay rate of $0.995$ as in prior work \citep{gowal2020uncovering}. 
We use a cosine learning rate schedule \citep{loshchilov2016sgdr} without restarts where the initial learning rate is set to $0.1$ for all baselines and \algoName. 
To alleviate robust overfitting \citep{rice2020overfitting}, we perform early stopping on a validation set of size 1024 using projected gradient descent (PGD) attacks with 20 steps.

\textbf{Baselines.}\quad
On CIFAR-10, the baselines include: (1) standard adversarial training (AT) \citep{madry2017towards} which trains on the worst case adversarial examples; (2) TRADES \citep{zhang2019theoretically} which trades off between the clean and robust accuracy; (3) MMA \citep{Ding2020MMA} which uses cross-entropy loss on the closest boundary points; (4) GAIRAT \citep{zhang2021geometryaware} which reweights adversarial examples based on the least perturbation iterations. (5) MAIL \citep{liu2021probabilistic} which reweights adversarial examples based on their logit margins. 
(6) AWP \citep{wu2020adversarial} which adversarially perturbs both inputs and model parameters.
On Tiny-ImageNet, we compare with AT, TRADES, and MART whose hyperparameter settings are available for this dataset. 
The hyperparameters of the baselines and full experimental settings are found in Appendix~\ref{app: exp_settings}. 

\textbf{Evaluation details.}\quad
We evaluate \algoName and the baselines under $\ell_\infty$ norm constrained perturbations. The final robust accuracy is reported on AutoAttack (AA) \citep{croce2020reliable}.
For all methods, we choose the hyperparameters to achieve the best robust accuracy under the commonly used perturbation bound 
$\epsilon=\frac{8}{255}$.
To fully compare the robustness performance among different methods, we report the robust accuracy 
under four additional perturbation bounds: $\frac{2}{255},\frac{4}{255},\frac{12}{255} $ and $\frac{16}{255}$.

\textbf{Hyperparameters of \algoName.}\quad 
We use the cost function $h(\cdot)$ in \eqref{eq:exp_h}.
On CIFAR-10, we use $\alpha=3$, $r_0 = \frac{16}{255}, \lambda = 1000$ and apply gradient clipping with threshold $0.1$. On Tiny-ImageNet, we use $\alpha=5$, $r_0 = \frac{32}{255}, \lambda = 500$ and apply gradient clipping with threshold $1$. The temperature $\beta$ is set to $5$. We use 20 iterations to find the closest soft boundary points using the adapted version of FAB. We use 10 epochs of natural training as the burn-in period. 

\input{tab_cifar10_linf.tex}

\input{tab_ImgNet200_linf.tex}

\textbf{Performance.}\quad 
The evaluation results on CIFAR-10 and Tiny-ImageNet are shown in Table \ref{table: performance_Linf_cifar10} and Table \ref{table: performance_Linf_ImgeNet200}, respectively. 
On CIFAR-10, under three out of five perturbation bounds, \algoName achieves the best robustness performance among all baselines. 
On Tiny-ImageNet, \algoName obtains both the highest robust accuracy under all perturbation bounds and the highest clean accuracy. 
These results indicate the superiority of \algoName in increasing the margins. 
\textbf{(1)} Specifically, on CIFAR-10, \algoName achieves the highest robust accuracy under $\epsilon = \frac{2}{255},\frac{4}{255}$ and $\frac{8}{255}$, and achieves the second highest robust accuracy under $\epsilon = \frac{12}{255}$ and $\frac{16}{255}$, which is lower than TRADES. 
(1a) Since \algoName prioritizes increasing smaller margins which are more important, \algoName performs better than TRADES under smaller perturbation bounds and achieves much higher clean accuracy. 
(1b) Although GAIRAT and AT have higher clean accuracy than \algoName, their robustness performance is lower than \algoName under all perturbation bounds. 
(1c) Thanks to directly operating on margins in the input space and encourage robustness improvement on points with smaller margins, \algoName performs better than GAIRAT and MAIL-TRADES which use indirect approximations of the margins. 
\textbf{(2)} On Tiny-ImageNet, \algoName achieves the best clean accuracy and the best robust accuracy under all perturbation bounds.
Further experimental results using various hyperparameter settings and types of normalization layers are left to Appendix~\ref{app: further_exp_results}. We also provide results of training WRN-28-10 with additional data from generated models \citep{wang2023better} on CIFAR-10 in Appendix~\ref{app: aux_data}, where \algoName achieves 63.89\% robust accuracy under $\epsilon=\frac{8}{255}$ and 93.69\% clean accuracy.

\subsection{Dynamics of \algoName}
\label{sec:exp_effect}

In this section we provide further insights into how \algoName encourages the desirable dynamics by comparing it with adversarial training.

\textbf{Experimental setting.} \quad To compare the dynamics of the decision boundary during training using \algoName and AT, we empirically compute the margins and speed values for both methods. For fair comparison, we run \algoName and AT on the same pretrained models for one iteration on the same batch of training data points. The pretrained models include a partially trained model using natural training and a partially trained model using AT, which are the same as in Section \ref{ssec: Dynamics of adversarial training}. For all the correctly classified points in this batch, we compute the margins and speed values under both methods. Note that the speed and margins correspond to the exact decision boundary, instead of the soft decision boundary used by \algoName for robust training. Since both methods train on the same model and the same batch of data at this iteration, the margins are the same and only the speed values differ, which corresponds to the difference in dynamics of the decision boundary. 

\input{fig_VR_barplots.tex}
\textbf{\algoName mitigates the conflicting dynamics.} \quad We visualize the dynamics on both pretrained models under \algoName and AT in Figure \ref{fig: VR_bar}. Specifically, we divide the range of margins into multiple intervals and compute the proportion of positive and negative speed within all the correctly classified points. On the naturally pretrained model, most of the points have margins less than $\frac{4}{255}$ (the first bin) and are considered more vulnerable. Among these points, \algoName reduces the proportion of the negative speed from $29.2\%$ to $15.3\%$ when comparing with AT. Therefore, a higher percentage of the margins of vulnerable points will increase using \algoName. On the adversarially pretrained model, \algoName reduces the proportion of negative speed values in the first three margin intervals and therefore has better dynamics of the decision boundary. We conclude that compared with AT, \algoName leads to better dynamics of the decision boundary where increasing smaller margins is prioritized. 

%% file: tab_cifar10_linf.tex
\begin{table}[!htbp]
\centering
\noindent
\resizebox{\textwidth}{!}{\begin{tabular}{lSSSSSSS}
\centering
Defense & {Clean} & {$\epsilon=\frac{2}{255}$} & {$\epsilon=\frac{4}{255}$} & {$\epsilon=\frac{8}{255}$} & {$\epsilon=\frac{12}{255}$} & {$\epsilon=\frac{16}{255}$}  \\
\midrule
AT & {$85.65 \pm 0.25$} & {$79.08 \pm 0.12$} & {$71.24 \pm 0.28$} & {$53.20\pm 0.16$} & {$32.94 \pm 0.32$}& {$16.12 \pm 0.23$}& \\
TRADES & {$82.92 \pm 0.30$} & {$77.69 \pm 0.16$}  & {$70.68 \pm 0.15$}& {$54.28 \pm 0.19$}  & {$\boldsymbol{36.65} \pm 0.24$}  & {$\boldsymbol{21.59} \pm 0.31$} & \\
MART & {$83.37 \pm 0.25$} & {$76.58 \pm 0.24$}  & {$70.19 \pm 0.18$}& {$52.91 \pm 0.24$}  & {$35.16 \pm 0.13$}  & {$18.80 \pm 0.14$} & \\ 
MMA & {$83.22 \pm 0.38$} & {$74.24 \pm 0.52$} & {$64.42 \pm 0.29$} & {$44.02 \pm 0.33$} &{$26.45 \pm 0.21$} & {$13.78 \pm 0.25$} & \\
GAIRAT & {$\boldsymbol{86.59} \pm 0.31$} & {$76.72 \pm 0.28$}& {$64.64 \pm 0.25$} & {$38.16 \pm 0.32$}  & {$19.01 \pm 0.18$} & {$7.55 \pm 0.17$} &\\
MAIL-TRADES & {$83.96 \pm 0.52$} & {$77.65 \pm 0.33$}  & {$69.11 \pm 0.35$}& {$50.14 \pm 0.29$}  & {$31.57 \pm 0.24$}  & {$16.98 \pm 0.15$} & \\
AWP & {$84.27 \pm 0.19$} & {$78.33 \pm 0.21$}& {$70.82 \pm 0.26$} & {$53.92 \pm 0.17$}  & {$35.24 \pm 0.26$} & {$20.40 \pm 0.14$} &\\
\midrule
\algoName & {$85.55 \pm 0.24$} & {$\boldsymbol{79.21} \pm 0.14$}  & {$\boldsymbol{71.73} \pm 0.18$} & {$\boldsymbol{54.69} \pm 0.14$} & {$35.74 \pm 0.25$} & {$20.79 \pm 0.18$} & \\
\bottomrule
\end{tabular}}
\caption{Clean and robust accuracy on CIFAR-10 under AA with different perturbation sizes on WRN-28-10.} 
\label{table: performance_Linf_cifar10}
\end{table}

%% file: tab_ImgNet200_linf.tex
\begin{table}[!htbp]
\centering
\noindent
\resizebox{\textwidth}{!}{\begin{tabular}{lSSSSSSS}
\centering
Defense & {Clean} & {$\epsilon=\frac{2}{255}$} & {$\epsilon=\frac{4}{255}$} & {$\epsilon=\frac{8}{255}$} & {$\epsilon=\frac{12}{255}$} & {$\epsilon=\frac{16}{255}$}  \\
\midrule
AT & {$48.09 \pm 0.38$} & {$38.82 \pm 0.26$} & {$30.18 \pm 0.27$} & {$16.46 \pm 0.19$} & {$7.74 \pm 0.20$}& {$3.05 \pm 0.17$}& \\
TRADES & {$46.68 \pm 0.30$} & {$37.84 \pm 0.21$}  & {$29.85 \pm 0.19$}& {$16.76 \pm 0.17$}  & {$8.97 \pm 0.23$}  & {$4.43\pm 0.11$} & \\
MART & {$45.51 \pm 0.29$} & {$36.68 \pm 0.34$} & {$29.15 \pm 0.25$} & {$17.79 \pm 0.15$} &{$9.91 \pm 0.17$} & {$5.31 \pm 0.17$} & \\
\midrule
\algoName & {$\boldsymbol{49.71} \pm 0.18$} & {$\boldsymbol{39.30} \pm 0.14$}  & {$\boldsymbol{30.69} \pm 0.21$} & {$\boldsymbol{18.02} \pm 0.18$} & {$\boldsymbol{10.08} \pm 0.09$} & {$\boldsymbol{5.65} \pm 0.12$} & \\
\bottomrule
\end{tabular}}
\caption{Clean and robust accuracy on Tiny-ImageNet under AA with different perturbation sizes on ResNet-18.} 
\label{table: performance_Linf_ImgeNet200}
\end{table}

%% file: fig_VR_barplots.tex
\begin{wrapfigure}{r}{0.6\textwidth}
    \begin{center}
    \begin{subfigure}[b]{0.295\textwidth}
        \includegraphics[width=\textwidth]{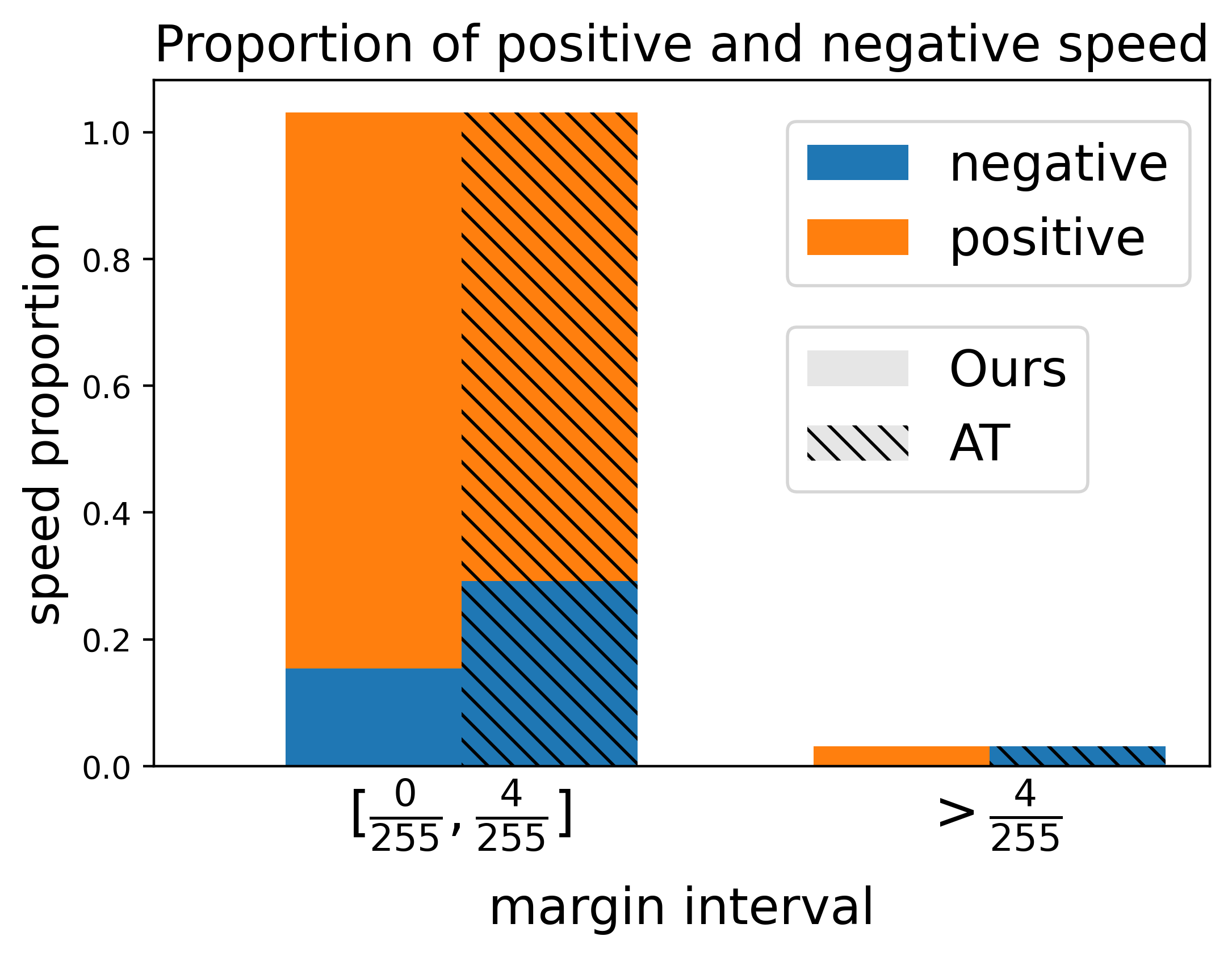}
        \caption{Naturally pretrained model}
        \label{fig: VR_bar_clean}
    \end{subfigure}%
    \begin{subfigure}[b]{0.30\textwidth}
        \includegraphics[width=\textwidth]{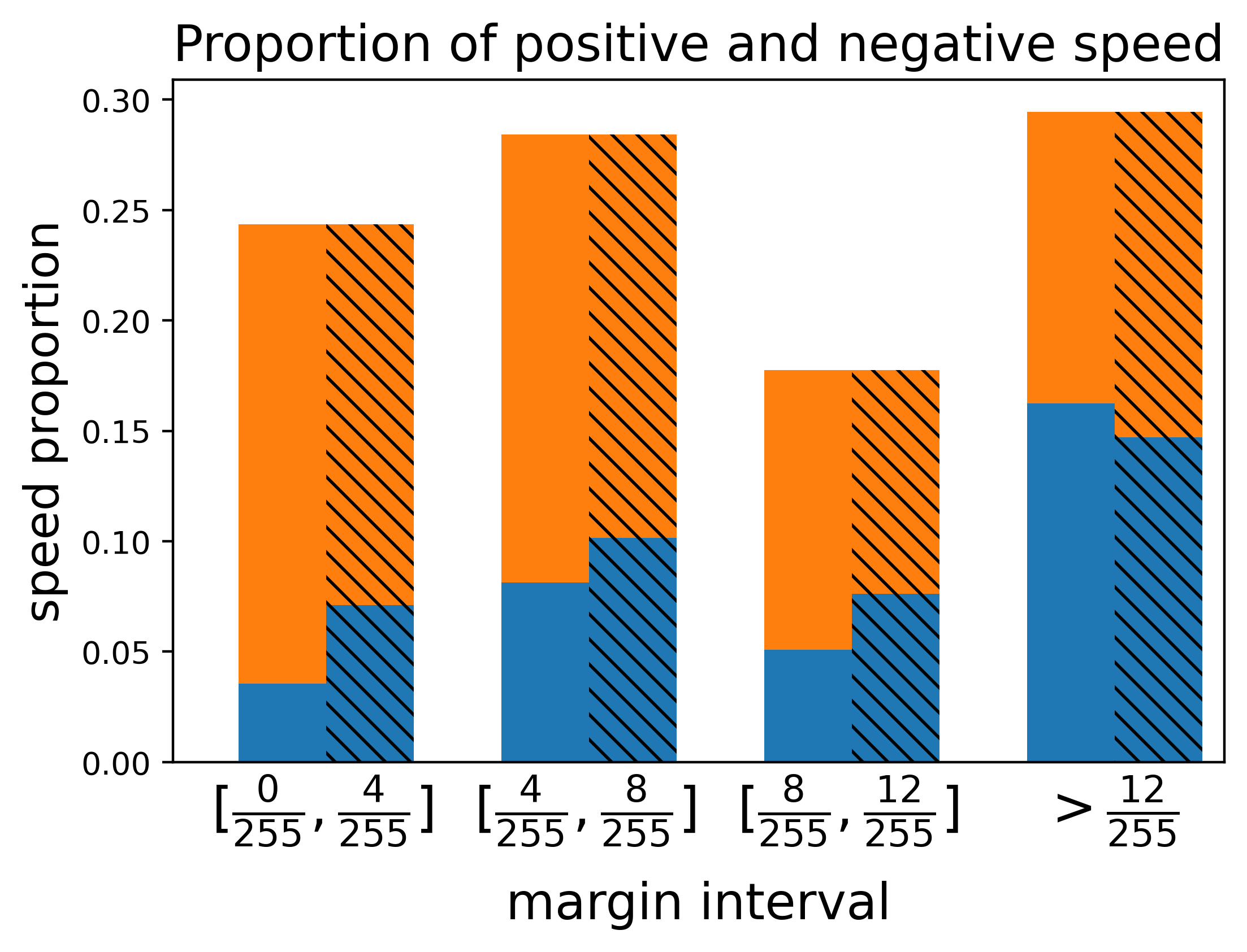}
        \caption{Partially robust model}
        \label{fig: VR_bar_robust}
    \end{subfigure}%
    \end{center}
    \vspace{-1em}
    \caption{Proportion of positive and negative speed values in each margin interval for AT and \algoName on a naturally pretrained model and a partially robust model. Observe that \algoName has lower proportion of negative speed for points with small margins ($<\frac{8}{255}$).}
    \label{fig: VR_bar}
\end{wrapfigure}

%% file: s7_conc.tex
\section{Conclusions and Discussions} \label{sec:conc}

This paper takes one more step towards understanding adversarial training by proposing a framework for studying the dynamics of the decision boundary. The phenomenon of conflicting dynamics is revealed, where the movement of decision boundary causes the margins of many vulnerable points to decrease and harms their robustness. To alleviate the conflicting dynamics, we propose \textit{\algoNameFull (\algoName)} which prioritizes moving the decision boundary away from more vulnerable points and increasing their margins. Experiments on CIFAR-10 and Tiny-ImageNet demonstrate that \algoName achieves improved robustness under various perturbation bounds. Future work includes (a) theoretical understanding of the dynamics of adversarial training; (b) developing more efficient numerical methods to find the closest boundary points for robust training.



%% file: s8_acknowledgement.tex
\subsubsection*{Acknowledgments}
The authors would like to thank Zhen Zhang, Chen Zhu and Wenxiao Wang for helpful discussions over the ideas.
This work is supported by National Science Foundation NSF-IIS-FAI program, DOD-ONR-Office of Naval Research, DOD Air Force Office of Scientific Research, DOD-DARPA-Defense Advanced Research Projects Agency Guaranteeing AI Robustness against Deception (GARD), Adobe, Capital One and JP Morgan faculty fellowships.

%% file: a4_related.tex
\section{Additional Related Work}
\label{app:add_related}

\textbf{Decision boundary analysis}\quad
In this paper, we mathematically characterizes the dynamics of decision boundaries and provide methods to directly compute and control the dynamics. 
Prior to this work, there are also some interesting studies on the dynamics of margins, though from different perspectives.
\cite{rade2022reducing} point out that adversarial training leads to a superfluous increase in the margin along the adversarial directions, which can be a reason behind the trade-off between accuracy and robustness.
\cite{ortiz2020hold} investigate the relationship between data features and decision boundaries, and reveal several properties of CNNs and adversarial training. Their results show that adversarial training exploits the sensitivity and invariance of models to improve the robustness.
\cite{tramer2020fundamental} studies invariance-based adversarial examples and expose a fundamental trade-off between commonly used sesitivity-based adversarial examples and the invariance-based ones, where the behaviors of decision boundaries are identified.

\textbf{Other Approaches to Improve Adversarial Training.}\quad
Recent works~\citep{najafi2019robustness,rebuffi2021fixing,gowal2021improving,gowal2020uncovering} have shown that the robust accuracy of adversarial training can be improved significantly with additional data from unlabeled datasets, data augmentation techniques and generative models. 
These approaches enhance the robustness of models by augmenting the dataset, which is orthogonal to our proposed algorithm that focus on how to optimize the model with the original dataset. 
\cite{wu2020adversarial} show that model robustness is related to the flatness of weight loss landscape, which is implicitly achieved by commonly used adversarial learning techniques.
Based on this insight, the authors propose to explicitly regularize the flatness of the weight loss landscape, which can improve the robust accuracy of existing adversarial training methods.
\cite{cui2021learnable} propose to use logits from a clean model to guide the learning of a robust model, which leads to both high natural accuracy and strong robustness. \\
We note that our method focuses on a different perspective of adversarial training, i.e., dynamics of decision boundary, and can be combined with these techniques to further improve the robust accuracy of the model. The investigation of such combination is out of the scope of this paper, and will be addressed in our future work.

\textbf{Certifiable Robustness.}\quad
There is an important line of work studying guaranteed robustness of neural networks. 
For example, convex relaxation of neural networks~\citep{gowal2019scalable,zhang2018finding,wong2018provable,zhang2020towards,gowal2018effectiveness} bounds the output of a network while the input data is perturbed within an $\ell_p$ norm ball. 
Randomized smoothing~\citep{cohen2019certified} is another certifiable defense which adds Gaussian noise to the input during test time. 
\citet{croce2019provable} propose a provably robust regularization for ReLU networks that maximizes the linear regions of the classifier and the distance to the decision boundary. Note that certifiable robust radius is a strict lower bound of the margin, which is the focus of our work.

%% file: a1_speed_derivation.tex
\newpage
\section{Proof of the closed-form expression for the speed 
in \eqref{eq: expression_for_speed}
and the margin gradient in \eqref{eq: expression_for_grad_hR} } \label{app: closed_expression_speed}

In this section, our goal is to prove the closed-form expression~\eqref{eq: expression_for_speed} as well as the margin gradient in \eqref{eq: expression_for_grad_hR} and provide further discussions. We first provide two preliminary lemmas and present the mathematical assumptions. Then we rigorously derive the closed-form expressions. Finally, we discuss more about the expression and its assumptions. 

\begin{lemma} \label{lemma: norm_pq_app}
For $1\leq p \leq \infty$ and let $q$ satisfies $1/q+1/p=1$. Let $a$ be any fixed vector. Then 
\begin{equation*} 
    \left\|\nabla_x \left\|x-a\right\|_p\right\|_q = 1
\end{equation*}
\end{lemma}

\begin{proof}
Without loss of generality, assume $a$ is the zero vector. Write the $k$-th component of $x$ as $x_k$.

\emph{Case 1: $1\leq p < \infty$}

By calculation, $\frac{\partial \|x\|_p}{\partial x_k} = (\frac{|x_k|}{\|x\|_p})^{p-1}\cdot \sign (x_k)$. Since $q = \frac{p}{p-1}$, we have that 
\begin{equation*}
\begin{aligned}
\sum_k |\frac{\partial \|x\|_p}{\partial x_k}|^{q} &= \sum_k  |(\frac{|x_k|}{\|x\|_p})^{p-1}\cdot \sign (x_k)|^{\frac{p}{p-1}} \\
&= \sum_k \frac{|x_k|^p}{\|x\|_p^p} \\
&= 1
\end{aligned}
\end{equation*}

Therefore, $\left\|\nabla_x \left\|x\right\|_p\right\|_q = (\sum_k |\frac{\partial \|x\|_p}{\partial x_k}|^{q})^{1/q} = 1$.

\emph{Case 2: $p = \infty$}

In this case, $\nabla_x \|x\|_\infty$ is a one-hot vector (with the one being the position of the element of $x$ with the largest absolute value). Therefore, $\left\|\nabla_x \left\|x\right\|_{\infty}\right\|_1$ = 1. 

\end{proof}

The following lemma deals with the optimality condition for $p = \infty$. Special care needs to be taken since $L_\infty$ norm is not a differentiable function.

\begin{lemma} \label{lemma: Linf_optimality_app}
Let $\hat{x}$ be a local optimum of the constrained optimization problem: 
\begin{equation*} \label{eq: closest_linf_app}
\setlength\abovedisplayskip{1pt}
\setlength\belowdisplayskip{1pt}
    \hat{x} = \argmin_{z} \|x-a\|_\infty \quad \textrm{s.t.}  \quad \phi(x) = 0,
\end{equation*}
where $a$ is any fixed vector with $\phi(a) > 0$.
Assume that $\phi$ is differentiable at point $\hat{x}$. Denote the coordinates set $\mathcal{J} = \{j: |\hat{x}_j - a_j| = \|\hat{x}-a\|_\infty\}$. Denote the $k$-th component of $\nabla_x \phi(\hat{x})$ as $\nabla_x \phi(\hat{x})_k$. Then (a) for $j \in \mathcal{J}$, $\nabla_x \phi(\hat{x})_j$ and $\hat{x}_j - a_j$ have opposite signs; (b) for $k\notin \mathcal{J}$, $\nabla_x \phi(\hat{x})_k = 0$. 
\end{lemma}

\begin{remark*}
If $\phi(a)<0$, then (a) for $j \in \mathcal{J}$, $\nabla_x \phi(\hat{x})_j$ and $\hat{x}_j - a_j$ have the same sign; (b) for $k\notin \mathcal{J}$, $\nabla_x \phi(\hat{x})_k = 0$.
\end{remark*}

\begin{proof}
\emph{(a)}
Consider the perturbation $\hat{x}(\epsilon) = \hat{x} + (0,\cdots,\epsilon_{j_1},\cdots,\epsilon_{j_m},\cdots,0)$ where $\mathcal{J} = \{j_1,\cdots,j_m\}$ and $\epsilon$ is a $m$ dimensional vector with $j$-th component $\epsilon_j$. Since $\phi(a)>0$ and $\hat{x}$ is a local optimum, $\|x-a\|_\infty < \|\hat{x}-a\|_\infty$ imply $\phi(x) > 0$ if $x$ is sufficiently close to $\hat{x}$. Therefore, if every $\epsilon_{j_i}$ is chosen so that $|\hat{x}_{j_i} + \epsilon_{j_i} - a_{j_i}| < |\hat{x}_{j_i} - a_{j_i}|$ (that is, $\epsilon_{j_i}$ has different sign from $\hat{x}_{j_i} - a_{j_i}$) and $\|\epsilon\|$ sufficiently small, then $\|\hat{x}(\epsilon) - a\| < \|\hat{x}-a\|$ and thus $\phi(\hat{x}(\epsilon))>0$. 

On the other hand, by Taylor expansion and the fact that $\phi(\hat{x})=0$, we have that 

\begin{equation*}
    \phi(\hat{x}(\epsilon)) = \sum_{j\in\mathcal{J}} \nabla \phi(\hat{x})_j \epsilon_j + \mathcal{O}(\|\epsilon\|^2)
\end{equation*}

Therefore, $\sum_{j\in\mathcal{J}} \nabla \phi(\hat{x})_j \epsilon_j > 0$ for any such $\epsilon$. By taking other $\epsilon_k \to 0$ if necessary, we obtain that $\forall j \in \mathcal{J}$, $\nabla \phi(\hat{x})_j \epsilon_j \geq 0$, where $\epsilon_j$ has different sign from $\hat{x}_{j} - a_{j}$. Therefore, $\nabla \phi(\hat{x})_j$ and $\hat{x}_{j} - a_{j}$ have different signs.

\emph{(b)}
Take any $k \notin \mathcal{J}$ and consider the perturbation $\hat{x}(\epsilon) = \hat{x} + (0,\cdots,\epsilon_{j_1},\cdots,\epsilon_k,\cdots,\epsilon_{j_k},\cdots,0)$ where $\epsilon=(\epsilon_{j_1},\cdots,\epsilon_k,\cdots,\epsilon_{j_k})$. Choose any $\epsilon$ so that $\|\epsilon\|$ is sufficiently small, each $\epsilon_{j_i}$ has the opposite sign of $\hat{x}_{j_i} - a_{j_i}$ and $\epsilon_k$ small enough (which can be positive or negative), we have that $\phi(\hat{x}(\epsilon)) > 0$ since $\|\hat{x}(\epsilon) - a\|_\infty < \|\hat{x}-a\|_\infty$. By Taylor expansion, $\sum_{j\in\mathcal{J}} \nabla \phi(\hat{x})_j \epsilon_j + \epsilon_k \nabla_x \phi(\hat{x})_k > 0$ for any such $\epsilon$. By taking $\epsilon_j \to 0$ and using the fact that $\epsilon_k$ can be positive or negative, we conclude that $\nabla_x \phi(\hat{x})_k = 0$.

\end{proof}

Now we are ready to derive the closed-form expression of the speed. We first provide the full assumptions, then derive the expression, and finally we will discuss more about the assumptions. We will write $\hat{x}_i(t)$ as $\hat{x}_i$ when the indication is clear. 
\begin{assumption} \label{assumption: differentible}
Suppose that $(x_i,y_i)$ is correctly classified by $f_{\theta(t)}$ in some time interval $t \in I$ and $\hat{x}_i(t)$ is a locally closest boundary point in the sense that for any $t\in I$, it is the local optimum of the following:
\begin{equation*} 
    \hat{x}_i(t) = \argmin\nolimits_{\hat{x}} \|\hat{x}-x_i\|_p \quad \textrm{s.t.}  \quad \phi^{y_i}(\hat{x},\theta(t)) = 0.
\end{equation*}
Assume that in the time interval $I$: (a) $\hat{x}_i(t)$ is differentiable in $t$; (b) $\phi^{y_i}$ is differentiable at point $\hat{x}_i(t)$ and at the current parameter $\theta(t)$.
\end{assumption}

\begin{proposition*}[Closed-form expression of the speed $s(x_i,t)$] For $1 \leq p \leq \infty$ and under Assumption~\ref{assumption: differentible}, define the (local) speed according to $\hat{x}_i(t)$ in Assumption~\ref{assumption: differentible} as $s(x_i,t) = \frac{d}{dt}\|\hat{x}_i(t) - x_i\|_p$, we have the following:

\begin{equation*} 
    s(x_i,t) = \frac{1}{\|\nabla_x \phi^{y_i}(\hat{x}_i(t),\theta(t))\|_q} \nabla_\theta \phi^{y_i} (\hat{x}_i(t),\theta(t)) \cdot \theta'(t)
\end{equation*}
where $q$ satisfies that $1/q+1/p=1$. In particular, $q=2$ when $p=2$ and $q=1$ when $p=\infty$.
\end{proposition*}

\begin{proof}
\emph{Case 1: $1\leq p < \infty$}

To compute $ s(x_i,t) = \frac{d}{dt} \|\hat{x}_i(t)-x_i\|_2$, we need to characterize the curve of the closest boundary point $\hat{x}_i(t)$, where two key points stand out. First, $\hat{x}_i(t)$ is on the decision boundary $\Gamma_y (t)$ and thus  $\phi^y(\hat{x}_i(t),\theta(t)) = 0$ for all $t \in I$. By taking the time derivative on both sides, we obtain the level set equation \citep{osher2004level,aghasi2011parametric}

\begin{equation} \label{eq: LF_app}
	\nabla_x \phi^{y_i}(\hat{x}_i(t),\theta(t))\cdot \hat{x}_i'(t) + \nabla_\theta \phi^{y_i} (\hat{x}_i(t),\theta(t)) \cdot \theta'(t) = 0
\end{equation}

Second, $\hat{x}_i(t)$ is the optimal solution of constrained optimization \eqref{eq: closest_pt}. Therefore, we have the following optimality condition:

\begin{equation} \label{eq: optimality_app}
    \nabla_x \phi^{y_i}(\hat{x}_i(t),\theta(t)) + \lambda(t)\nabla_x\|\hat{x}_i(t)-x_i\|_p = 0
\end{equation}

Since $x_i$ is correctly classified, $\phi^{y_i}(x_i) > 0$. Since $\hat{x}_i(t)$ is the closest point to $x_i$ whose $\phi^{y_i}$ value is zero, $\lambda(t) >0$. By taking the $L_q$ norm in Equation~\eqref{eq: optimality_app} and using Lemma~\ref{lemma: norm_pq_app}, we obtain that $\lambda(t) = \|\nabla_x \phi^{y_i}(\hat{x}_i(t),\theta(t))\|_q$.

Now, we derive $s(x_i,t)$ as follows:

\begin{equation*}
\begin{aligned}
    s(x_i,t) &= \frac{d}{dt} \|\hat{x}_i(t)-x_i\|_p  \\
    & = \nabla_x\|\hat{x}_i(t)-x_i\|_p \cdot \hat{x}_i'(t) \\
    & = -\frac{1}{\lambda} \nabla_x \phi^{y_i}(\hat{x}_i(t),\theta(t)) \cdot \hat{x}_i'(t) && \text{(By the optimality condition~\eqref{eq: optimality_app})} \\
    & = \frac{1}{\lambda} \nabla_\theta \phi^{y_i} (\hat{x}_i(t),\theta(t)) \cdot \theta'(t) && \text{(By the level set equation~\eqref{eq: LF_app})}\\
    & = \frac{\nabla_\theta \phi^{y_i} (\hat{x}_i(t),\theta(t)) \cdot \theta'(t)}{\|\nabla_x \phi^{y}(\hat{x}_i(t),\theta(t))\|_q}  
\end{aligned}
\end{equation*}

\emph{Case 2: $p = \infty$}

Note that since $L_\infty$ is not differentiable, the optimality condition in Equation~\eqref{eq: optimality_app} does not hold anymore. 

Denote the $j$-th component of $\hat{x}_i(t)$ and $x_i$ as $\hat{x}_{ij}(t)$ and $x_{ij}$. Let $\mathcal{J} = \{j: |\hat{x}_{ij}(t) - x_{ij}| = \|\hat{x}_i(t)-x_i\|_\infty\}$. By Lemma~\ref{lemma: Linf_optimality_app}, $s(x_i,t) = \frac{d}{dt}|\hat{x}_{ij}(t) - x_{ij}| = \hat{x}_{ij}^{'}(t) \sign (\hat{x}_{ij}(t)-x_{ij}) = - \hat{x}_{ij}^{'}(t) \sign (\nabla_x\phi^{y_i}(\hat{x}_{i})_j)$ for all $j \in \mathcal{J}$. Therefore, by Equation~\eqref{eq: LF_app} and Lemma~\ref{lemma: Linf_optimality_app}

\begin{equation*}
\begin{aligned}
- \nabla_\theta \phi^{y_i} (\hat{x}_i(t),\theta(t)) \cdot \theta'(t) &= \nabla_x \phi^{y_i}(\hat{x})\cdot \hat{x}_i'(t)\\
&= \sum_{j \in \mathcal{J}} \nabla_x \phi^{y_i}(\hat{x}_i)_j\cdot \hat{x}_{ij}'(t) \\
&= \sum_{j \in \mathcal{J}} - \nabla_x \phi^{y_i}(\hat{x}_i)_j \cdot \frac{s(x_i,t)}{\sign (\nabla_x\phi^{y}(\hat{x}_{i})_j)} \\
&= - \sum_{j \in \mathcal{J}} |\nabla_x\phi^{y_i}(\hat{x}_{i})_j| \cdot s(x_i,t)
\end{aligned}
\end{equation*}

Therefore, $s(x_i,t) = \frac{\nabla_\theta \phi^{y_i} (\hat{x}_i(t),\theta(t)) \cdot \theta'(t)}{\sum_{j \in \mathcal{J}} |\nabla_x\phi^{y_i}(\hat{x}_{i})_j|} = \frac{\nabla_\theta \phi^{y_i} (\hat{x}_i(t),\theta(t)) \cdot \theta'(t)}{\|\nabla_x\phi^{y_i}(\hat{x}_{i})\|_1}$, where the last equality follows from Lemma~\ref{lemma: Linf_optimality_app} that the components of $\nabla_x \phi^{y_i}(\hat{x}_i)$ are zeros if they are not in $\mathcal{J}$.
 
\end{proof}

As an corollary of the proposition we prove above, we can obtain the closed-form expression for the gradient of margin (or the gradient of any smooth function of the margin) as follows:

\begin{theorem*}
[Closed-form expression of $\nabla_\theta h(R_{\theta}(x_i))$] For $1 \leq p \leq \infty$,
\begin{equation*}
    \nabla_\theta h(R_{\theta}(x_i)) = \frac{h'(R_{\theta}(x_i))}{\|\nabla_x \phi^{y_i}(\hat{x}_i,\theta)\|_q} \nabla_\theta \phi^{y_i} (\hat{x}_i,\theta)
\end{equation*}
where $q$ satisfies that $1/q+1/p=1$.
\end{theorem*}

\begin{proof}
    In continuous time we consider $h(R(x_i,t))$ (or more rigorously, $h(R(x_i,\theta(t)))$)  and its time derivative. We use the following relationship between the gradient and the time derivative, where $\theta'(t)$ can be any update rule:

    \begin{equation*}
        \frac{d}{dt}h(R(x_i,t)) = \nabla_\theta h(R(x_i,t)) \cdot \theta'(t)
    \end{equation*}

    On the other hand:
    
    \begin{equation*}
    \begin{aligned}
        \frac{d}{dt}h(R(x_i,t)) &= h'(R(x_i,t)))\frac{d}{dt}R(x_i,t) \\ 
        &= h'(R(x_i,t))s(x_i,t) \\
        &=  \frac{h'(R(x_i))}{\|\nabla_x \phi^{y_i}(\hat{x}_i(t),\theta(t))\|_q} \nabla_\theta \phi^{y_i} (\hat{x}_i(t),\theta(t)) \cdot \theta'(t)
    \end{aligned}
    \end{equation*}

    where the last equality uses the closed-form expression for the speed $s(x_i,t)$.

    Therefore we have that for any $\theta'(t)$, $\nabla_\theta h(R(x_i,t)) \cdot \theta'(t) = \frac{h'(R(x_i))}{\|\nabla_x \phi^{y_i}(\hat{x}_i(t),\theta(t))\|_q} \nabla_\theta \phi^{y_i} (\hat{x}_i(t),\theta(t)) \cdot \theta'(t)$. We conclude that $\nabla_\theta h(R_{\theta}(x_i)) = \frac{h'(R_{\theta}(x_i))}{\|\nabla_x \phi^{y_i}(\hat{x}_i,\theta)\|_q} \nabla_\theta \phi^{y_i} (\hat{x}_i,\theta)$.
    
\end{proof}

\subsection*{Discussions on the assumptions}
Assumption~\ref{assumption: differentible} has several points that need to be explained further.

First, we only require that $\hat{x}_i(t)$ is a local closest boundary point. This is important because in practice when an algorithm for searching the closest boundary point is used (e.g. FAB), a local solution is the best one can hope for due to the non-convex nature of the optimization problem. When $\hat{x}_i(t)$ is a local solution, the speed should be interpreted as how fast the distance changes around that local solution. In this case, although $\hat{x}_i$ is not the globally closest adversarial example, the local speed around $\hat{x}_i$ still has much information on the relative movement of the decision boundary w.r.t. $x_i$, especially when the distance  $\|\hat{x}_i - x_i\|$ is relatively small and the input space is a high-dimensional space (e.g. pixel space). 

Second, we require that $\hat{x}_i(t)$ is a differentiable curve in time interval $I$. Note that if we constrain $\hat{x}_i(t)$ to be the global closest boundary point, $\hat{x}_i$ is unlikely to be differential for a large time interval $I$, especially in high dimensional space. This is because as the decision boundary moves. the closest point might switch from one point to another point that is very far away abruptly. Fortunately, this problem is alleviated because our derived closed-form expression still works when $\hat{x}_i(t)$ is a locally closest boundary point. Note that however, due to the \textit{topological change} of the decision boundary, it still can happen that $\hat{x}_i(t)$ stops existing (and thus stops being differentiable) at some time point, when typically the speed will go to infinity. From a mathematical point of view, this is related to \textit{shock} in partial differential equation (PDE) theories. More exploration on this is left to future work. In this work, we only consider the speed of the decision boundary at each discrete time step.


%% file: a2_method_details.tex
\newpage
\section{Computation of the exact and soft closest boundary point} \label{app: compute_bdr_exact_soft}

Either computing the speed of the decision boundary or using \algoName to directly optimize a function of margins requires the computation of the closest boundary point $\hat{x}$ (or the closest soft boundary point $\hat{x}^{\text{soft}}$), where we omit the subscript $i$ in this section. As discussed in Appendix~\ref{app: closed_expression_speed}, it suffices to find the locally closest (soft) boundary point in order for the closed-form expression~\ref{eq: expression_for_speed} and expression~\ref{eq: expression_for_grad_hR} for the speed and the gradient of margin to be valid.

\subsection{Closest boundary point} \label{app: compute_bdr_exact}

In this section, we will explain how to check the quality of the found $\hat{x}$ for the constrained optimization problem~\ref{eq: closest_pt} in practice. We will also give a simple analysis on how FAB\citep{croce2020minimally}, the algorithm we use in our implementation, solves the problem~\ref{eq: closest_pt} in practice. We include both $p=2$ and $p=\infty$ although in our work, only $p=\infty$ is used. We discuss both of them in order to highlight the difference in checking optimality conditions for smooth ($p=2$) and non-smooth norm ($p=\infty$).

The key points of analyzing $\hat{x}$ are that $\phi^{y}(\hat{x}) = 0$ and the KKT conditions of problem~\ref{eq: closest_pt}.

\emph{Case 1: $p=2$}

In this case, the KKT condition is given by $\nabla_x \phi^{y}(\hat{x}) + \lambda (\hat{x} - x) = 0$ for some $\lambda > 0$ (since $\phi^{y}(x) >0$). In other words, $\frac{\nabla_x \phi^{y}(\hat{x})}{\|\nabla_x \phi^{y}(\hat{x})\|} \cdot \frac{x-\hat{x}}{\|x-\hat{x}\|} = 1$. In practice, we check the following two conditions (a) $|\phi(\hat{x})| \leq 0.1$; (b) $\frac{\nabla_x \phi^{y}(\hat{x})}{\|\nabla_x \phi^{y}(\hat{x})\|} \cdot \frac{x-\hat{x}}{\|x-\hat{x}\|} > 0.8$. 
We observe in our experiments that FAB can find high-quality closest boundary points for over 90\% of the correctly classified data points. 

\emph{Case 2: $p=\infty$}

In this case, we consider the optimality condition given in Lemma~\ref{lemma: Linf_optimality_app} of Appendix~\ref{app: closed_expression_speed}. Denote $\#B$ the number of points in a set $B$. Using the notation $\mathcal{J} = \{j: |\hat{x}_j - x_j| = \|\hat{x}-x\|_\infty\}$ and $\mathcal{J}^{C}$ the complement set of $\mathcal{J}$, we check the following conditions in practice: (a) $|\phi(\hat{x})| \leq 0.1$; (b) $\frac{\#\{j\in\mathcal{J}: \nabla_x \phi^{y}(\hat{x})_j (\hat{x}_j - x_j) \leq 0\}}{\#\mathcal{J}} > 0.9$; (c) $\frac{\#\{k\notin J: |\nabla_x \phi^{y}(\hat{x})_k| < 0.1 \}}{\#\mathcal{J}^C} > 0.8$. Note the unlike $p=2$, the optimality conditions for $p=\infty$ are on each coordinate of $\hat{x}$, which is more difficult to satisfy in practice. We observe in our experiments that FAB with 100 iterations can find high-quality closest boundary points for about 85\% of the correctly classified points. However, when only 20 iterations are used, condition (3) is barely satisfied for all of the found boundary points (the first two conditions are still satisfied). 

In our visualizations of dynamics of the decision boundary for AT in Section~\ref{ssec: Dynamics of adversarial training}, we use 100 iterations for FAB and only use high-quality closest boundary points, so that the visualization results are relatively accurate. 

\subsection{Closest soft boundary point }\label{app: compute_bdr_soft}

\paragraph{Adapt FAB for soft decision boundary.} In \algoName, the closest point $\hat{x}^{\text{soft}}$ on the soft decision boundary is used. To find $\hat{x}^{\text{soft}}$, we adapt the FAB method. The original FAB method aims to find the closest point on the exact decision boundary. In particular, FAB forms linear approximations for decision boundary between the ground truth class and every other classes. The only adaptation we do on the FAB method is that now FAB only forms \emph{one} linear approximation for the soft decision boundary of the ground truth class. This is because we use the smoothed max operator in the soft logit margin, and there is no concept of the 'decision boundary between the ground truth class and another class' anymore. 

\paragraph{Computational efficiency.} By using the soft decision boundary, every iteration of FAB only requires one linear approximation of the soft decision boundary, which cost one back-propagation. In contrast, the original FAB which aims to find the closest boundary point on the exact decision boundary costs $K$ back-propagation at each iteration, where $K$ is the number of classes. Therefore, using the soft decision boundary is more efficient and is used in our proposed robust training method \algoName.

\paragraph{Local optimality condition.} The procedure of checking optimality condition is similar to the one in the last section.  Denote $\#B$ the number of points in a set $B$. Using the notation $\mathcal{J} = \{j: |\hat{x}_j - x_j| = \|\hat{x}-x\|_\infty\}$ and $\mathcal{J}^{C}$ the complement set of $\mathcal{J}$, we check the following conditions in practice: (a) $|\phi(\hat{x})| \leq 0.1$; (b) $\frac{\#\{j\in\mathcal{J}: \nabla_x \phi^{y}(\hat{x})_j (\hat{x}_j - x_j) \leq 0\}}{\#\mathcal{J}} > 0.9$; We find that when the temperature $\beta$ is relatively large (we use $\beta=5$ in all of our experiments) and 20 iterations is used, $95\%$ of the soft boundary point found for the correctly classified points satisfy these two conditions. During training, we only use these higher quality points and discard the rest of the boundary points that do not satisfy these two conditions. Note that we do not consider the third condition (c) $\frac{\#\{k\notin J: |\nabla_x \phi^{y}(\hat{x})_k| < 0.1 \}}{\#\mathcal{J}^C} > 0.8$. This is because condition (c) cannot be satisfied unless a very large iteration number is used, which is computationally prohibitive for robust training.

Experimentally \algoName achieves improved robustness over baseline methods, indicating that the closest soft boundary points used by \algoName are indeed useful for robust training. Designing faster and more reliable methods to solve the constrained optimization problem~\ref{eq: closest_pt} is left for future work.


%% file: a3_experiment.tex
\newpage
\section{Experiments}
In this section, we provide the details of experimental settings and further results of \algoName using various choices of hyperparameters. In addition, we provide experimental results when using additional data from the generated models. We also provide further analysis on the decision boundary dynamics.

\subsection{Detailed experimental settings} \label{app: exp_settings}

\paragraph{Architectures and training settings.} 
In all experiments on the CIFAR-10 dataset, we use the Wide Residual Network  \citep{zagoruyko2016wide} with depth 28 and width factor 10 (WRN-28-10) with Swish activation function \citep{ramachandran2017searching}. 
On the Tiny-ImageNet dataset, we use pre-activation ResNet-18 \citep{he2016identity}. 
In all experiments, we use stochastic weight averaging \citep{izmailov2018averaging} with a decay rate of $0.995$ as in prior work \citep{gowal2020uncovering,chen2020robust}. 
All models are trained using stochastic gradient descent with momentum 0.9 and weight decay 0.0005.
We use a cosine learning rate schedule \citep{loshchilov2016sgdr} without restarts where the initial learning rate is set to $0.1$ for baselines. 
To alleviate robust overfitting \citep{rice2020overfitting}, we compute the robust and clean accuracy at every epoch on a validation set of size 1024 using projected gradient descent (PGD) attacks with 20 steps using margin loss function. All experiments are run on NVIDIA GeForce RTX 2080 Ti GPU. 

\paragraph{Normalization layers} We consider two types of normalization layer in WRN-28-10 and ResNet-18, which are Batch Normalization (BN, used in their original architecture design) and Group Normalization (GN). When using GN, the decision boundaries are the same during training and evaluation, which is consistent with our theoretical analysis on the decision boundary dynamics. In the following sections, we will show the robustness performance on both cases: WRN-28-10 and ResNet-18 with BN and GN. We find that when applying \algoName on original WRN-28-10 and ResNet-18 with BN, gradient clipping needs to be applied in order to learn the BN parameters stably. We apply gradient clipping with norm threshold $0.1$ for experiments for CIFAR-10 on WRN-28-10 with BN and apply gradient clipping with norm threshold $1$ for Tiny-ImageNet on ResNet-18 with BN. For experiments on architectures with GN, we do not apply gradient clipping. Note that for all experiments of computing speed and margins for interpretation (Section \ref{sec: dynamics_of_bdr} and Section \ref{sec:exp_effect}), we use the ResNet-18 with GN.

\paragraph{Additional training settings}
For experiments with Group Normalization, models are run for 100 epochs on both datasets. 
For \algoName on Tiny-ImageNet, we use the cosine learning rate schedule with initial learning rate $0.05$ and on CIFAR-10, the learning rate begins at 0.1 and is decayed by a factor of 10 at the 50th and 75th epoch.
For experiments with Batch Normalization, models are run for 200 epochs on CIFAR-10 and 100 epochs on Tiny-ImageNet. For \algoName on both datasets, we use a cosine learning rate schedule \citep{loshchilov2016sgdr} without restarts where the initial learning rate is set to $0.1$, which is the same as the baselines.

\textbf{Compared baselines and their hyperparameters.}\quad
In all experiments we consider the $\ell_\infty$ perturbation setting. 
On CIFAR-10, the baseline defense methods include: 
(1) standard adversarial training (AT) \citep{madry2017towards} which trains on the worst case adversarial examples generated by 10-step PGD (PGD-10) on the cross-entropy loss. The perturbation bound is $\frac{8}{255}$and the step size of PGD is $\frac{2}{255}$; the training setting follows Rice et al.\footnote{\href{https://github.com/locuslab/robust_overfitting}{Robust Overfitting's Github}} \citep{rice2020overfitting}.
(2) TRADES \footnote{\href{https://github.com/yaodongyu/TRADES}{TRADES's Github}}\citep{zhang2019theoretically} which trades off between the clean and robust accuracy. The perturbation bound is $\frac{8}{255}$ with the step size of PGD-10 $0.007$. The regularization constant beta (or 1/lambda) is set to  6.
(3) MMA  \footnote{\href{https://github.com/BorealisAI/mma_training}{MMA's Github}}  \citep{Ding2020MMA} which trains on the closest adversarial examples (closest boundary points) with uniform weights. 
The MaxEps is set to $\frac{32}{255}$. 
(4) GAIRAT \footnote{\href{https://github.com/zjfheart/Geometry-aware-Instance-reweighted-Adversarial-Training}{GAIRAT's Github}} \citep{zhang2021geometryaware}  which reweights adversarial examples using the least perturbation steps. The perturbation bound is $\frac{8}{255}$ with step size $0.007$ using PGD-10 and the 'tanh' weight assignment function is used. 
(5) MAIL \footnote{\href{https://github.com/QizhouWang/MAIL}{MAIL's github}} \citep{liu2021probabilistic} which reweights adversarial examples using margin value. We choose its combination with TRADES (MAIL-TRADES) which provides better robustness performance than combining with AT (MAIL-AT). Its hyperparamters beta, bias and slope are set to $5.0, -1.5$ and $1.0$, respectively. 
(6) AWP \footnote{\href{https://github.com/csdongxian/AWP}{AWP's github}} \citep{wu2020adversarial} which adversarially perturbs both inputs and model parameters.
(7) FAT \footnote{\href{https://github.com/zjfheart/Friendly-Adversarial-Training}{FAT's github}} \citep{zhang2020attacks} that exploits friendly adversarial data, where the perturbation bound is set to $\frac{8}{255}$.
(8) MART~\citep{wang2019improving} which explicitly differentiates the mis-classified and correctly classified examples. 
On Tiny-ImageNet, we compare with AT, TRADES, and MART whose hyperparameter settings are available for this dataset. We follow the PyTorch implementation of \citep{gowal2020uncovering,rebuffi2021fixing} \footnote{\href{https://github.com/imrahulr/adversarial_robustness_pytorch}{UncoveringATLimits's Github}} for AT, TRADES and MART for both datasets.

\textbf{Evaluation details.}\quad
We evaluate \algoName and the baselines under $\ell_\infty$ norm constrained perturbations. The final robust accuracy is reported on AutoAttack (AA) \citep{croce2020reliable}, which uses an ensemble of selected strong attacks.
For all methods, we choose the hyperparameters to achieve the best robust accuracy under the commonly used perturbation bound 
$\epsilon=\frac{8}{255}$.
To fully compare the robustness performance among different methods, we report the robust accuracy 
under four additional perturbation bounds: $\frac{2}{255},\frac{4}{255},\frac{12}{255} $ and $\frac{16}{255}$.

\paragraph{Per-sample gradient} For computing the speed of the decision boundary in Section~\ref{ssec: Dynamics of adversarial training} and Section~\ref{sec:exp_effect}, we need to compute the per-sample gradient $\nabla_\theta \phi^{y_i}(\hat{x}_i,\theta)$ for every correctly classified point $x_i$. We use the Opacus package \citep{opacus} for computing per-sample gradients in parallel. Also, another reason why we replace BN with GN is because Opacus does not support BN for computing per-sample gradients. Although using this package will increase the memory usage, it is worth mentioning that during robust training, $\algoName$ does not need to compute the per-sample gradient and thus does not have the excessive memory issue. Per-sample gradients are only collected for computing speed, which is for interpretation of dynamics of different methods and not for robust training. 

\subsection{Hyperparameter Sensitivity Experiments}\label{app: further_exp_results}

In this section, we present the robustness performance of \algoName under different hyperparameter settings. We first show the results for architectures using Group Normalization (note that in Section \ref{sec:exp_robust} we use the original architectures using Batch Normalization) and analyze the effect of different hyperparameters. We then demonstrate more ablation experiments for architectures using Batch Normalization used in Section \ref{sec:exp_robust}.

\input{ICLR_appendix_experiments/tab_cifar10_linf_GN.tex}
\input{ICLR_appendix_experiments/tab_ImgNet200_linf_GN.tex}

\paragraph{Overall performance of \algoName on architectures with GN} In Table \ref{table: performance_Linf_cifar10_GN} and Table \ref{table: performance_Linf_ImgeNet200_GN}, the overall comparison between \algoName and baselines are demonstrated. 
Overall on both datasets, under four out of five perturbation bounds, \algoName achieves the best robustness performance. This indicates the superiority of \algoName in increasing margins. 
\textbf{(1)} Specifically, on CIFAR-10, \algoName achieves the highest clean accuracy as well as robust accuracy under all perturbation bounds among all baselines except FAT-TRADES. 
(1a) Since FAT-TRADES prevents the model from learning on highly adversarial data in order to keep clean accuracy high, it achieves the best clean accuracy and robustness under a very small perturbation bound $\frac{2}{255}$. However, its performance on larger perturbation bounds is inadequate. 
(1b) Thanks to directly operating on margins in the input space and encourage robustness improvement on points with smaller margins, \algoName performs better than GAIRAT and MAIL-TRADES which use indirect approximations of the margins. 
\textbf{(2)} On Tiny-ImageNet, \algoName achieves the best clean accuracy and the best robust accuracy under all perturbation bounds except the largest $\frac{16}{255}$. 
Although MART is the most robust under $\frac{16}{255}$, it has much lower clean accuracy ($8.93\%$ lower than $\algoName$) and worse robustness under smaller perturbation bounds. 

\paragraph{Hyperparameters of \algoName} In this paper, we use the cost function of the form $h(R) = \frac{1}{\alpha}\exp(-\alpha R)$ when $R < r_0$ and $h(R) = 0$ otherwise. We present results under different decay strengh $\alpha>0$, margin threshold $r_0$ as well as regularization constant $\lambda$ for the robustness loss.  

\textbf{Performance results.}\quad 
The evaluation results on CIFAR-10 and Tiny-ImageNet with Group Normalization are shown in Table \ref{table: performance_Linf_cifar10_GN_more} and Table \ref{table: performance_Linf_ImgeNet200_GN_more}, respectively. 
We analyze the effects of hyperparameters as follows. 

\textbf{(1)} Effect of $\alpha$: Larger $\alpha$ corresponds to a cost function $h(\cdot)$ that decays faster, and therefore prioritize improvement on even smaller margins. Therefore, it should be expected that larger $\alpha$ leads to higher clean accuracy and higher robust accuracy under smaller perturbation sizes, and results in lower robust accuracy under larger perturbation sizes. For example, on CIFAR-10, when $\alpha = 5$ is increased to $\alpha=8$ when $r_0=\frac{16}{255}, \lambda=400$, the clean accuracy as well as the robust accuracy under $\epsilon =\frac{2}{255}, \frac{4}{255}$ and $\frac{8}{255}$ improves, while the robust accuracy under larger $\epsilon$ gets lower. The same patterns can also be observed on Tiny-ImageNet, for example, when $\alpha=8$ is increased to $\alpha=10$ when $r_0=\frac{20}{255}, \lambda=1000$. 

\textbf{(2)} Effect of $r_0$: $r_0$ is from preventing \algoName from training boundary points that are too far away from clean data points. Therefore, it should be expected that training on smaller $r_0$ tends to increase the clean accuracy and the robust accuracy under relatively small perturbation sizes. Indeed, on Tiny-ImageNet, when $r_0 = \frac{24}{255}$ is decreased to $r_0=\frac{20}{255}$ when $\alpha=10$ and $\lambda=1000$, we can observe that the clean accuracy as well as robust accuracy under $\epsilon=\frac{2}{255}$ increases but the robust accuracy under larger perturbation sizes $\epsilon=\frac{12}{255}$ and $\frac{16}{255}$ decreases.

\textbf{(3)}: Effect of robust loss constant $\lambda$: A larger $\lambda$ tends to increase the robustness of the model (in particular, the robust accuracy under relatively larger perturbation sizes) while decrease the clean accuracy and the robust accuracy under relatively small perturbation sizes. For example, on Tiny-ImageNet, when $\lambda = 800$ is increased to $\lambda=1000$ when $\alpha=10$ and $r_0=\frac{20}{255}$, the clean accuracy and robust accuracy under relatively small $\epsilon = \frac{2}{255}, \frac{4}{255}$ drops but the robust accuracy under larger perturbation sizes increase.

\input{ICLR_appendix_experiments/tab_cifar10_linf_GN_more.tex}
\input{ICLR_appendix_experiments/tab_ImgNet200_linf_GN_more.tex}

\textbf{(4)} Effect of the burn-in period: A burn-in period of natural training is necessary for \algoName since its robust loss function depends on the closest boundary points, which can only be found on correctly classified points. That is, \algoName requires a descent initial clean accuracy. In our experiments, we find that the learning rate of the burn-in period is important: \algoName will train successfully if the learning rate of the burn-in period is relatively large (e.g. 0.1 for CIFAR-10 and Tiny-ImageNet). However, when the learning rate is small (such as $0.001$), \algoName sometimes drives the clean accuracy to be very low at first, and fails to train. Our suggestion is to use a larger learning rate to obtain a naturally pretrained model. 

\paragraph{More ablation on architectures with BN} In Table \ref{table: performance_Linf_cifar10_BN_more} and Table \ref{table: performance_Linf_ImgeNet200_BN_more}, we demonstrate results for more hyperparamter settings for experiments with Batch Normalization in Section \ref{sec:exp_robust}. The role of each hyperparameter is similar to  the GN case. 

\input{ICLR_appendix_experiments/tab_cifar10_linf_BN_more.tex}
\input{ICLR_appendix_experiments/tab_ImgNet200_linf_BN_more.tex}

\newpage
\subsection{Experimental results with additional data} \label{app: aux_data}
Recent work shows that generative models which are trained solely on the original training data can be used to drastically improve the adversarial robustness performance \citep{rebuffi2021fixing,gowal2021improving,wang2023better}. In this section, we demonstrate the results of \algoName on CIFAR-10 using 10M additional data from a recent diffusion model \citep{karras2022elucidating, wang2023better}. The experimental setting follows previous works \citep{rebuffi2021fixing,gowal2020uncovering} and their PyTorch implementation \footnote{\href{https://github.com/imrahulr/adversarial_robustness_pytorch}{UncoveringATLimits's Github}}, which is consistent with Section \ref{app: exp_settings} except that now we run 1200 epochs with batch size 1024 using the additional data. The initial learning rate is still set to $0.1$. Note that in the original implementation in recent works \citep{rebuffi2021fixing,gowal2020uncovering, wang2023better}, more epochs are run with possibly larger batch size and larger initial learning rate $0.4$, which could improve the performance compared with 1200 epochs and 1024 batch size with initial learning rate $0.1$ that we use for \algoName. 

\paragraph{Hyperparameters \algoName with additional data} We use WRN-28-10 (with its original batch normalization layers) and we choose $\alpha=3, r_0 = \frac{16}{255}, \lambda = 800, \beta=5$ and apply gradient clipping with threshold $0.1$. We use 10 epochs of natural training as the burn-in period. Note that we do not further tune the hyperparameters due to limited computational resources, and it is very likely that further tuning will lead to better robustness performance.

\paragraph{Performance of \algoName with additional data} The experimental results are shown in Table \ref{table: performance_Linf_cifar10_BN_AuxData}. Compared with the results in Table \ref{table: performance_Linf_cifar10_BN_more}, it is clear that the additional data can drastically improve the robust accuracy of \algoName (about $9\%$ boost in robust accuracy under $\epsilon=\frac{8}{255}$ and $8\%$ boost in clean accuracy).

\input{ICLR_appendix_experiments/tab_cifar10_linf_AuxData.tex}

\subsection{Further analysis on Decision boundary dynamics}

In Section~\ref{ssec: Dynamics of adversarial training} and Section~\ref{sec:exp_effect} we have presented the dynamics of both AT and \algoName on the same pretrained models using the same batch of data at one iteration. In this section, we demonstrate the dynamics of the decision boundary throughout the whole training process. 

\textbf{Experiment setting.}\quad To study the decision boundary dynamics throughout the training process, we train a ResNet-18 \citep{he2016deep} with group normalization (GN) \citep{wu2018group} on CIFAR-10 using (1) Adversarial Training with 10-step PGD under $\ell_\infty$ perturbation with $\epsilon=\frac{8}{255}$ from scratch; (2) \algoName with $\alpha=8, \lambda=400, r_0 = \frac{16}{255}$ from a naturally pretrained model. The models are trained with a initial learning rate of $0.01$ and the learning rate decays to $0.001$ at 20000 iteration. At each iteration, we compute the proportion of negative speed among points with margins smaller than $\frac{8}{255}$ that are regarded as vulnerable.

\textbf{Conflicting dynamics throughout training}
In Figure \ref{fig app: Dyna whole training},
the clean and robust accuracy of both methods are shown, along with the proportion of negative speed among vulnerable points. We apply curve smoothing for negative speed proportion plot for better visualization. Note that we omit the initial part of training (first 10 epochs) since at this initial stage, there are not enough correctly classified data points but speed and margin are only defined for these points. We can see that both methods exhibit some degree of robust overfitting, where the training robust accuracy becomes larger than the test robust accuracy. In addition, the conflicting dynamics exists throughout the whole training process, since the proportion of negative speed is never zero. We can see that \algoName consistently has less conflicting dynamics than AT. Interestingly, the proportion of negative speed decreases over time during training for both methods. The connection between the decreasing degree of conflicting dynamics on the training data and the robust overfitting phenomenon is left for future research. 

\input{ICLR_appendix_experiments/DynamicsWholeTraining.tex}

\subsection{Run time analysis}

In this section we provide the run time analysis. The main computational bottleneck for \algoName is finding the closest boundary points, which is an iterative algorithm adapted from FAB. Each iteration costs one back-propagation, which is the same as Projected Gradient Descent (PGD). Once we find these closest boundary point candidates, we check if the KKT condition is approximately satisfied and filter out points that do not meet the KKT condition. The computational cost of this step takes one back-propagation.

We use the torch.cuda.Event functionality in PyTorch to measure the execution time for one iteration of each method. In the case of \algoName, this means measuring the total time of finding the closest boundary points and do back propagation using the full loss function.
We use ResNet-18 with GroupNorm on a batch size of 128 on the CIFAR10 dataset. We use one NVDIA RTX A4000. The results are as follows:

\begin{itemize}
  \item Natural training: $46 \pm 0.9$ ms
  \item AT (PGD-10 on Cross-Entropy loss): $531 \pm 5.2$ ms
  \item TRADES (PGD-10 on KL divergence loss): $573 \pm 2.8$ ms
  \item DyART (10 steps for finding the closest boundary point):  $743 \pm 10.3$ ms
  \item DyART (20 steps for finding the closest boundary point): $1171 \pm 17.8$ ms
\end{itemize}

Developing faster algorithms for finding the closest boundary points is left for future research.

%% file: ICLR_appendix_experiments/tab_cifar10_linf_GN.tex
\begin{table}[!htbp]
\centering
\noindent
\resizebox{\textwidth}{!}{\begin{tabular}{lSSSSSSS}
\centering
Defense & {Clean} & {$\epsilon=\frac{2}{255}$} & {$\epsilon=\frac{4}{255}$} & {$\epsilon=\frac{8}{255}$} & {$\epsilon=\frac{12}{255}$} & {$\epsilon=\frac{16}{255}$}  \\
\midrule
AT & {$85.36 \pm 0.17$} & {$77.16 \pm 0.29$} & {$67.84 \pm 0.24$} & {$46.27 \pm 0.19$} & {$26.62 \pm 0.18$}& {$12.40 \pm 0.12$}& \\
TRADES & {$84.67 \pm 0.24$} & {$77.72 \pm 0.18$}  & {$69.38 \pm 0.15$}& {$49.29 \pm 0.15$}  & {$30.25 \pm 0.17$}  & {$16.42 \pm 0.18$} & \\
MART & {$81.02 \pm 0.17$} & {$73.04 \pm 0.21$}  & {$64.94 \pm 0.22$}& {$48.06 \pm 0.20$}  & {$30.61 \pm 0.13$}  & {$16.42 \pm 0.09$} & \\ 
MMA & {$85.52 \pm 0.36$} & {$74.78 \pm 0.42$} & {$62.21 \pm 0.39$} & {$38.61 \pm 0.47$} &{$22.13 \pm 0.29$} & {$9.95 \pm 0.20$} & \\
GAIRAT & {$83.72 \pm 0.27$} & {$73.87 \pm 0.33$}& {$61.7 \pm 0.15$} & {$37.77 \pm 0.21$}  & {$18.87 \pm 0.15$} & {$8.1 \pm 0.11$} &\\
MAIL-TRADES & {$84.48 \pm 0.22$} & {$77.18 \pm 0.26$}  & {$68.20 \pm 0.31$}& {$48.64 \pm 0.12$}  & {$29.87 \pm 0.11$}  & {$15.62 \pm 0.16$} & \\
FAT-TRADES & {$\boldsymbol{86.58} \pm 0.25$} & {$\boldsymbol{78.96} \pm 0.17$}  & {$69.54 \pm 0.12$}& {$48.07 \pm 0.19$}  & {$27.66 \pm 0.12$}  & {$13.22 \pm 0.23$} & \\
\midrule
\algoName & {$85.64 \pm 0.10$} & {$78.20 \pm 0.16$}  & {$\boldsymbol{69.59} \pm 0.19$} & {$\boldsymbol{50.03} \pm 0.16$} & {$\boldsymbol{30.87} \pm 0.20$} & {$\boldsymbol{16.55} \pm 0.12$} & \\
\bottomrule
\end{tabular}}
\caption{{Clean and robust accuracy on CIFAR-10 under AA with different perturbation sizes on WRN-28-10 with Group Normalization. The hyperparameters for \algoName is $\alpha=8,r_0=\frac{16}{255},\lambda=400$.} }
\label{table: performance_Linf_cifar10_GN}
\end{table}

%% file: ICLR_appendix_experiments/tab_ImgNet200_linf_GN.tex
\begin{table}[!htbp]
\centering
\noindent
\resizebox{\textwidth}{!}{\begin{tabular}{lSSSSSSS}
\centering
Defense & {Clean} & {$\epsilon=\frac{2}{255}$} & {$\epsilon=\frac{4}{255}$} & {$\epsilon=\frac{8}{255}$} & {$\epsilon=\frac{12}{255}$} & {$\epsilon=\frac{16}{255}$}  \\
\midrule
AT & {$43.76 \pm 0.53$} & {$35.54 \pm 0.36$} & {$28.20 \pm 0.21$} & {$16.92 \pm 0.24$} & {$9.34 \pm 0.18$}& {$4.75 \pm 0.14$}& \\
TRADES & {$46.56 \pm 0.29$} & {$37.23 \pm 0.17$}  & {$28.68 \pm 0.19$}& {$16.20 \pm 0.21$}  & {$8.38 \pm 0.10$}  & {$4.23\pm 0.06$} & \\
MART & {$38.74 \pm 0.42$} & {$32.18 \pm 0.74$} & {$26.08 \pm 0.31$} & {$16.90 \pm 0.26$} &{$10.14 \pm 0.22$} & {$\boldsymbol{6.10} \pm 0.19$} & \\
\midrule
\algoName & {$\boldsymbol{47.67} \pm 0.15$} & {$\boldsymbol{38.19} \pm 0.18$}  & {$\boldsymbol{29.59} \pm 0.14$} & {$\boldsymbol{17.79} \pm 0.18$} & {$\boldsymbol{10.24} \pm 0.13$} & {$5.41 \pm 0.11$} & \\
\bottomrule
\end{tabular}}
\caption{{Clean and robust accuracy on Tiny-ImageNet under AA with different perturbation sizes on ResNet-18 with Group Normalization. The hyperparameters for \algoName is $\alpha=3, r_0=\frac{20}{255},\lambda=500$}. }
\label{table: performance_Linf_ImgeNet200_GN}
\end{table}

%% file: ICLR_appendix_experiments/tab_cifar10_linf_GN_more.tex
\begin{table}[!htbp]
\centering
\noindent
\renewcommand{\arraystretch}{1.5}
\resizebox{\textwidth}{!}{\begin{tabular}{lSSSSSSS}
\centering
Defense & {Clean} & {$\epsilon=\frac{2}{255}$} & {$\epsilon=\frac{4}{255}$} & {$\epsilon=\frac{8}{255}$} & {$\epsilon=\frac{12}{255}$} & {$\epsilon=\frac{16}{255}$}  \\
\midrule
AT & {$85.36 \pm 0.17$} & {$77.16 \pm 0.29$} & {$67.84 \pm 0.24$} & {$46.27 \pm 0.19$} & {$26.62 \pm 0.18$}& {$12.40 \pm 0.12$}& \\
TRADES & {$84.67 \pm 0.24$} & {$77.72 \pm 0.18$}  & {$69.38 \pm 0.15$}& {$49.29 \pm 0.15$}  & {$30.25 \pm 0.17$}  & {$16.42 \pm 0.18$} & \\
MART & {$81.02 \pm 0.17$} & {$73.04 \pm 0.21$}  & {$64.94 \pm 0.22$}& {$48.06 \pm 0.20$}  & {$30.61 \pm 0.13$}  & {$16.42 \pm 0.09$} & \\ 
MMA & {$85.52 \pm 0.36$} & {$74.78 \pm 0.42$} & {$62.21 \pm 0.39$} & {$38.61 \pm 0.47$} &{$22.13 \pm 0.29$} & {$9.95 \pm 0.20$} & \\
GAIRAT & {$83.72 \pm 0.27$} & {$73.87 \pm 0.33$}& {$61.7 \pm 0.15$} & {$37.77 \pm 0.21$}  & {$18.87 \pm 0.15$} & {$8.1 \pm 0.11$} &\\
MAIL-TRADES & {$84.48 \pm 0.22$} & {$77.18 \pm 0.26$}  & {$68.20 \pm 0.31$}& {$48.64 \pm 0.12$}  & {$29.87 \pm 0.11$}  & {$15.62 \pm 0.16$} & \\
FAT-TRADES & {$\boldsymbol{86.58} \pm 0.25$} & {$\boldsymbol{78.96} \pm 0.17$}  & {$69.54 \pm 0.12$}& {$48.07 \pm 0.19$}  & {$27.66 \pm 0.12$}  & {$13.22 \pm 0.23$} & \\
\midrule
$\alpha = 10, r_0 =\frac{20}{255}, \lambda=400$ & {$84.17 \pm 0.12$} & {$76.85 \pm 0.15$}  & {$68.47 \pm 0.15$} & {$49.41 \pm 0.20$} & {$32.07 \pm 0.17$} & {$18.72 \pm 0.10$} & \\
$\alpha = 5, r_0 =\frac{20}{255}, \lambda=300$ & {$83.33 \pm 0.19$} & {$75.96 \pm 0.24$}  & {$67.82 \pm 0.22$} & {$49.55 \pm 0.24$} & {$\boldsymbol{32.29} \pm 0.14$} & {$\boldsymbol{19.16} \pm 0.15$} & \\
$\alpha = 8, r_0 =\frac{16}{255}, \lambda=400$ & {$85.64 \pm 0.10$} & {$78.20 \pm 0.16$}  & {$\boldsymbol{69.59} \pm 0.19$} & {$\boldsymbol{50.03} \pm 0.16$} & {$30.87 \pm 0.20$} & {$16.55 \pm 0.12$} & \\
$\alpha = 5, r_0 =\frac{16}{255}, \lambda=400$ & {$85.05 \pm 0.14$} & {$77.92 \pm 0.21$}  & {$69.00 \pm 0.14$} & {$49.60 \pm 0.12$} & {$30.78 \pm 0.14$} & {$17.06 \pm 0.09$} & \\
$\alpha = 0, r_0 =\frac{16}{255}, \lambda=400$ & {$83.85 \pm 0.23$} & {$76.77 \pm 0.20$}  & {$68.26 \pm 0.15$} & {$49.65 \pm 0.18$} & {$31.72 \pm 0.21$} & {$17.76 \pm 0.18$} & \\
\bottomrule
\end{tabular}}
\caption{Clean and robust accuracy on CIFAR-10 under AA with different perturbation bounds on WRN-28-10 with Group Normalization. The results on different sets of hyperparameters for \algoName starts from the eighth row.}
\label{table: performance_Linf_cifar10_GN_more}
\end{table}

%% file: ICLR_appendix_experiments/tab_ImgNet200_linf_GN_more.tex
\begin{table}[!htbp]
\centering
\noindent
\renewcommand{\arraystretch}{1.5}
\resizebox{\textwidth}{!}{\begin{tabular}{lSSSSSSS}
\centering
Defense & {Clean} & {$\epsilon=\frac{2}{255}$} & {$\epsilon=\frac{4}{255}$} & {$\epsilon=\frac{8}{255}$} & {$\epsilon=\frac{12}{255}$} & {$\epsilon=\frac{16}{255}$}  \\
\midrule
AT & {$43.76 \pm 0.53$} & {$35.54 \pm 0.36$} & {$28.20 \pm 0.21$} & {$16.92 \pm 0.24$} & {$9.34 \pm 0.18$}& {$4.75 \pm 0.14$}& \\
TRADES & {$46.56 \pm 0.29$} & {$37.23 \pm 0.17$}  & {$28.68 \pm 0.19$}& {$16.20 \pm 0.21$}  & {$8.38 \pm 0.10$}  & {$4.23\pm 0.06$} & \\
MART & {$38.74 \pm 0.42$} & {$32.18 \pm 0.74$} & {$26.08 \pm 0.31$} & {$16.90 \pm 0.26$} &{$10.14 \pm 0.22$} & {$\boldsymbol{6.10} \pm 0.19$} & \\
\midrule
$\alpha = 10, r_0 =\frac{32}{255}, \lambda=500$ & {$\boldsymbol{48.98} \pm 0.24$} & {$\boldsymbol{38.38} \pm 0.32$}  & {$\boldsymbol{29.76} \pm 0.22$} & {$17.30 \pm 0.19$} & {$9.87 \pm 0.11$} & {$5.19 \pm 0.15$} & \\
$\alpha = 10, r_0 =\frac{24}{255}, \lambda=1000$ & {$45.27 \pm 0.19$} & {$36.58 \pm 0.52$}  & {$29.03 \pm 0.32$} & {$17.35 \pm 0.24$} & {$10.03 \pm 0.18$} & {$5.61 \pm 0.16$} & \\
$\alpha = 10, r_0 =\frac{20}{255}, \lambda=1000$ & {$46.37 \pm 0.26$} & {$37.43 \pm 0.32$}  & {$29.01 \pm 0.19$} & {$17.61 \pm 0.20$} & {$9.91 \pm 0.18$} & {$5.27 \pm 0.14$} & \\
$\alpha = 10, r_0 =\frac{20}{255}, \lambda=800$ & {$47.09 \pm 0.22$} & {$38.04 \pm 0.12$}  & {$29.55 \pm 0.17$} & {$17.22 \pm 0.15$} & {$9.59 \pm 0.20$} & {$5.08 \pm 0.11$} & \\
$\alpha = 8, r_0 =\frac{20}{255}, \lambda=1000$ & {$45.69 \pm 0.17$} & {$36.74 \pm 0.20$}  & {$28.57 \pm 0.12$} & {$17.31 \pm 0.21$} & {$10.13 \pm 0.15$} & {$5.19 \pm 0.16$} & \\
$\alpha = 5, r_0 =\frac{20}{255}, \lambda=800$ & {$45.61 \pm 0.14$} & {$36.87 \pm 0.16$}  & {$29.01 \pm 0.19$} & {$17.58 \pm 0.10$} & {$\boldsymbol{10.38} \pm 0.16$} & {$5.33 \pm 0.09$} & \\
$\alpha = 3, r_0 =\frac{20}{255}, \lambda=500$ & {$47.67 \pm 0.15$} & {$38.19 \pm 0.18$}  & {$29.59 \pm 0.14$} & {$\boldsymbol{17.79} \pm 0.18$} & {$10.24 \pm 0.13$} & {$5.41 \pm 0.11$} & \\
$\alpha = 3, r_0 =\frac{16}{255}, \lambda=1000$ & {$45.27 \pm 0.20$} & {$36.71 \pm 0.16$}  & {$28.74 \pm 0.20$} & {$17.40 \pm 0.16$} & {$9.80 \pm 0.13$} & {$4.92 \pm 0.13$} & \\
\bottomrule
\end{tabular}}
\caption{Clean and robust accuracy on Tiny-ImageNet under AA with different perturbation bounds on ResNet-18 with Group Normalization. The results on different sets of hyperparameters for \algoName starts from the fourth row.} 
\label{table: performance_Linf_ImgeNet200_GN_more}
\end{table}

%% file: ICLR_appendix_experiments/tab_cifar10_linf_BN_more.tex
\begin{table}[!htbp]
\centering
\noindent
\renewcommand{\arraystretch}{1.5}
\resizebox{\textwidth}{!}{\begin{tabular}{lSSSSSSS}
\centering
Defense & {Clean} & {$\epsilon=\frac{2}{255}$} & {$\epsilon=\frac{4}{255}$} & {$\epsilon=\frac{8}{255}$} & {$\epsilon=\frac{12}{255}$} & {$\epsilon=\frac{16}{255}$}  \\
\midrule
AT & {$85.65 \pm 0.25$} & {$79.08 \pm 0.12$} & {$71.24 \pm 0.28$} & {$53.20\pm 0.16$} & {$32.94 \pm 0.32$}& {$16.12 \pm 0.23$}& \\
TRADES & {$82.92 \pm 0.30$} & {$77.69 \pm 0.16$}  & {$70.68 \pm 0.15$}& {$54.28 \pm 0.19$}  & {$36.65 \pm 0.24$}  & {$21.59 \pm 0.31$} & \\
MART & {$83.37 \pm 0.25$} & {$76.58 \pm 0.24$}  & {$70.19 \pm 0.18$}& {$52.91 \pm 0.24$}  & {$35.16 \pm 0.13$}  & {$18.80 \pm 0.14$} & \\ 
MMA & {$83.22 \pm 0.38$} & {$74.24 \pm 0.52$} & {$64.42 \pm 0.29$} & {$44.02 \pm 0.33$} &{$26.45 \pm 0.21$} & {$13.78 \pm 0.25$} & \\
GAIRAT & {$\boldsymbol{86.59} \pm 0.31$} & {$76.72 \pm 0.28$}& {$64.64 \pm 0.25$} & {$38.16 \pm 0.32$}  & {$19.01 \pm 0.18$} & {$7.55 \pm 0.17$} &\\
AWP & {$84.27 \pm 0.19$} & {$78.33 \pm 0.21$}& {$70.82 \pm 0.26$} & {$53.92 \pm 0.17$}  & {$35.24 \pm 0.26$} & {$20.40 \pm 0.14$} &\\

\midrule
$\alpha = 0, r_0 =\frac{16}{255}, \lambda=1000$ & {$85.10 \pm 0.24$} & {$78.68  \pm 0.18$}  & {$71.67  \pm 0.28$} & {$54.78  \pm 0.21$} & {$36.26 \pm 0.24$} & {$21.55 \pm 0.16$} & \\
$\alpha = 3, r_0 =\frac{16}{255}, \lambda=1000$ & {$85.55 \pm 0.24$} & {$79.21 \pm 0.14$}  & {$71.73 \pm 0.18$} & {$54.69 \pm 0.14$} & {$35.74 \pm 0.25$} & {$20.79 \pm 0.18$} & \\
$\alpha = 3, r_0 =\frac{16}{255}, \lambda=1500$ & {$85.34 \pm 0.19$} & {$78.97 \pm 0.21$}  & {$71.82 \pm 0.27$} & {$54.39 \pm 0.17$} & {$35.94 \pm 0.13$} & {$20.83 \pm 0.19$} & \\
$\alpha = 8, r_0 =\frac{16}{255}, \lambda=1000$ & {$86.36 \pm 0.32$} & {$79.84 \pm 0.25$}  & {$\boldsymbol{72.29} \pm 0.29$} & {$53.93 \pm 0.14$} & {$35.06 \pm 0.22$} & {$20.08 \pm 0.11$} & \\
$\alpha = 8, r_0 =\frac{16}{255}, \lambda=2000$ & {$86.10 \pm 0.15$} & {$79.33 \pm 0.22$}  & {$72.04 \pm 0.32$} & {$54.38 \pm 0.19$} & {$35.36 \pm 0.27$} & {$20.68 \pm 0.14$} & \\
$\alpha = 8, r_0 =\frac{16}{255}, \lambda=3000$ & {$86.05 \pm 0.27$} & {$\boldsymbol{79.64} \pm 0.25$}  & {$72.24 \pm 0.28$} & {$54.24 \pm 0.17$} & {$35.65 \pm 0.21$} & {$20.52 \pm 0.18$} & \\
$\alpha = 0, r_0 =\frac{20}{255}, \lambda=500$ & {$83.64 \pm 0.24$} & {$77.28 \pm 0.31$}  & {$70.11 \pm 0.21$} & {$54.21 \pm 0.26$} & {$37.75 \pm 0.21$} & {$23.85 \pm 0.20$} & \\
$\alpha = 0, r_0 =\frac{20}{255}, \lambda=1000$ & {$82.23 \pm 0.12$} & {$76.20 \pm 0.24$}  & {$69.48 \pm 0.26$} & {$\boldsymbol{54.82} \pm 0.25$} & {$\boldsymbol{38.47} \pm 0.21$} & {$24.80 \pm 0.20$} & \\
$\alpha = 3, r_0 =\frac{20}{255}, \lambda=400$ & {$84.51 \pm 0.08$} & {$78.26 \pm 0.14$}  & {$70.65 \pm 0.18$} & {$54.24 \pm 0.15$} & {$37.28 \pm 0.15$} & {$22.95 \pm 0.18$} & \\
$\alpha = 3, r_0 =\frac{20}{255}, \lambda=800$ & {$83.56 \pm 0.20$} & {$77.44 \pm 0.24$}  & {$70.25 \pm 0.28$} & {$54.22 \pm 0.23$} & {$37.85 \pm 0.18$} & {$23.88 \pm 0.17$} & \\
$\alpha = 5, r_0 =\frac{20}{255}, \lambda=1000$ & {$83.94 \pm 0.12$} & {$77.79 \pm 0.34$}  & {$70.77 \pm 0.28$} & {$54.39 \pm 0.27$} & {$37.59 \pm 0.19$} & {$23.61 \pm 0.23$} & \\
$\alpha = 5, r_0 =\frac{32}{255}, \lambda=500$ & {$81.52 \pm 0.31$} & {$75.58 \pm 0.21$}  & {$68.03 \pm 0.24$} & {$53.50 \pm 0.17$} & {$38.41 \pm 0.17$} & {$\boldsymbol{26.30} \pm 0.14$} & \\
\bottomrule
\end{tabular}}
\caption{Clean and robust accuracy on CIFAR-10 under AA with different perturbation bounds on WRN-28-10 (with its original Batch Normalization layer). The results on different sets of hyperparameters for \algoName starts from the seventh row.}
\label{table: performance_Linf_cifar10_BN_more}
\end{table}

%% file: ICLR_appendix_experiments/tab_ImgNet200_linf_BN_more.tex
\begin{table}[!htbp]
\centering
\noindent
\renewcommand{\arraystretch}{1.5}
\resizebox{\textwidth}{!}{\begin{tabular}{lSSSSSSS}
\centering
Defense & {Clean} & {$\epsilon=\frac{2}{255}$} & {$\epsilon=\frac{4}{255}$} & {$\epsilon=\frac{8}{255}$} & {$\epsilon=\frac{12}{255}$} & {$\epsilon=\frac{16}{255}$}  \\
\midrule
AT & {$48.09 \pm 0.38$} & {$38.82 \pm 0.26$} & {$30.18 \pm 0.27$} & {$16.46 \pm 0.19$} & {$7.74 \pm 0.20$}& {$3.05 \pm 0.17$}& \\
TRADES & {$46.68 \pm 0.30$} & {$37.84 \pm 0.21$}  & {$29.85 \pm 0.19$}& {$16.76 \pm 0.17$}  & {$8.97 \pm 0.23$}  & {$4.43\pm 0.11$} & \\
MART & {$45.51 \pm 0.29$} & {$36.68 \pm 0.34$} & {$29.15 \pm 0.25$} & {$17.79 \pm 0.15$} &{$9.91 \pm 0.17$} & {$5.31 \pm 0.17$} & \\
\midrule
$\alpha = 10, r_0 =\frac{32}{255}, \lambda=500$ & {$48.98 \pm 0.33$} & {$38.38 \pm 0.26$}  & {$29.76 \pm 0.21$} & {$17.30 \pm 0.23$} & {$9.87 \pm 0.18$} & {$5.19 \pm 0.12$} & \\
$\alpha = 0, r_0 =\frac{16}{255}, \lambda=1000$ & {$46.49 \pm 0.25$} & {$37.60 \pm 0.24$}  & {$29.17 \pm 0.16$} & {$17.0 \pm 0.19$} & {$9.02 \pm 0.15$} & {$4.59 \pm 0.14$} & \\
$\alpha = 0, r_0 =\frac{20}{255}, \lambda=500$ & {$49.17 \pm 0.21$} & {$39.53 \pm 0.24$}  & {$30.20 \pm 0.22$} & {$17.15 \pm 0.27$} & {$9.08 \pm 0.10$} & {$4.87 \pm 0.07$} & \\
$\alpha = 0, r_0 =\frac{20}{255}, \lambda=1000$ & {$43.9 \pm 0.19$} & {$36.23 \pm 0.23$}  & {$28.89 \pm 0.20$} & {$17.47 \pm 0.19$} & {$10.01 \pm 0.19$} & {$5.53 \pm 0.11$} & \\
$\alpha = 0, r_0 =\frac{20}{255}, \lambda=2000$ & {$42.21 \pm 0.28$} & {$34.79 \pm 0.19$}  & {$27.66 \pm 0.21$} & {$16.55 \pm 0.34$} & {$9.43 \pm 0.19$} & {$5.01 \pm 0.16$} & \\
$\alpha = 0, r_0 =\frac{24}{255}, \lambda=500$ & {$48.27 \pm 0.25$} & {$38.75 \pm 0.20$}  & {$30.02 \pm 0.16$} & {$18.00 \pm 0.18$} & {$10.15 \pm 0.18$} & {$5.56 \pm 0.08$} & \\
$\alpha = 5, r_0 =\frac{24}{255}, \lambda=500$ & {$\boldsymbol{50.86} \pm 0.28$} & {$\boldsymbol{39.81} \pm 0.18$}  & {$30.63 \pm 0.19$} & {$17.20 \pm 0.25$} & {$9.36 \pm 0.11$} & {$4.99 \pm 0.13$} & \\
$\alpha = 5, r_0 =\frac{24}{255}, \lambda=1000$ & {$44.59 \pm 0.32$} & {$36.3 \pm 0.30$}  & {$28.6 \pm 0.27$} & {$17.28 \pm 0.16$} & {$10.16 \pm 0.17$} & {$5.76 \pm 0.06$} & \\
$\alpha = 0, r_0 =\frac{32}{255}, \lambda=500$ & {$46.19 \pm 0.22$} & {$37.64 \pm 0.30$}  & {$29.49 \pm 0.21$} & {$\boldsymbol{18.05} \pm 0.15$} & {$\boldsymbol{10.66} \pm 0.13$} & {$\boldsymbol{6.27} \pm 0.09$} & \\
$\alpha = 3, r_0 =\frac{32}{255}, \lambda=500$ & {$48.56 \pm 0.20$} & {$39.32 \pm 0.23$}  & {$30.22 \pm 0.19$} & {$17.93 \pm 0.15$} & {$10.22 \pm 0.15$} & {$5.84 \pm 0.13$} & \\
$\alpha = 5, r_0 =\frac{32}{255}, \lambda=500$ & {$49.71 \pm 0.18$} & {$39.30 \pm 0.14$}  & {$\boldsymbol{30.69} \pm 0.21$} & {$18.02 \pm 0.18$} & {$10.08 \pm 0.09$} & {$5.65 \pm 0.12$} & \\
$\alpha = 5, r_0 =\frac{32}{255}, \lambda=800$ & {$45.90 \pm 0.34$} & {$37.68 \pm 0.25$}  & {$29.53 \pm 0.27$} & {$17.96 \pm 0.20$} & {$10.55 \pm 0.15$} & {$6.07 \pm 0.14$} & \\
\bottomrule
\end{tabular}}
\caption{Clean and robust accuracy on Tiny-ImageNet under AA with different perturbation bounds on ResNet-18 (with its original Batch Normalization layer). The results on different sets of hyperparameters for \algoName starts from the fourth row.} 
\label{table: performance_Linf_ImgeNet200_BN_more}
\end{table}

%% file: ICLR_appendix_experiments/tab_cifar10_linf_AuxData.tex
\begin{table}[!htbp]
\centering
\noindent
\resizebox{\textwidth}{!}{\begin{tabular}{lSSSSSSSS}
\centering
Defense & {Architecture}& {Clean} & {$\epsilon=\frac{2}{255}$} & {$\epsilon=\frac{4}{255}$} & {$\epsilon=\frac{8}{255}$} & {$\epsilon=\frac{12}{255}$} & {$\epsilon=\frac{16}{255}$}  \\
\midrule
\algoName & {WRN-28-10} & {$93.69$} & {$89.08$} & {$82.72$} & {$63.89$} & {$42.34$}& {$22.84$}& \\
\bottomrule
\end{tabular}}
\caption{{Clean and robust accuracy on CIFAR-10 under $l_\infty$ AutoAttack with different perturbation sizes when 10M additional generated data is used for training.} }
\label{table: performance_Linf_cifar10_BN_AuxData}
\end{table}

%% file: ICLR_appendix_experiments/DynamicsWholeTraining.tex
\begin{figure}[!htbp]
    \centering
    \begin{subfigure}[t]{0.3\textwidth}
        \centering
        \resizebox{\linewidth}{!}{\input{ICLR_appendix_experiments/AT_acc.tex}}
        \caption{\scriptsize{Accuracy of AT}}
    \end{subfigure}%
    \hfill
    \begin{subfigure}[t]{0.3\textwidth}
        \centering
        \resizebox{\linewidth}{!}{\input{ICLR_appendix_experiments/DyART_acc.tex}}
        \caption{\scriptsize{Accuracy of \algoName}}
    \end{subfigure}%
    \hfill
    \begin{subfigure}[t]{0.3\textwidth}
        \centering
        \resizebox{\linewidth}{!}{\input{ICLR_appendix_experiments/neg_ratio_wholeTraining.tex}}
        \caption{\scriptsize{Negative speed proportion
        }}
    \end{subfigure}%
    \caption{The accuracy of AT and \algoName as well as the proportion of negative speed among points whose margins are smaller than $\frac{8}{255}$.}
    \label{fig app: Dyna whole training}
\end{figure}
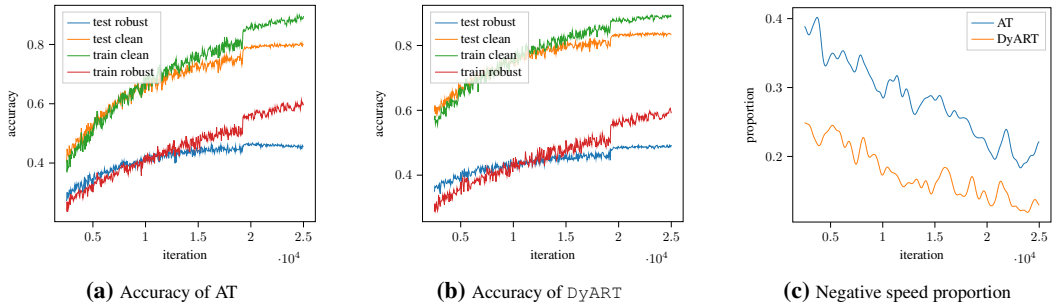

%% file: ICLR_appendix_experiments/AT_acc.tex
\begin{tikzpicture}

\definecolor{crimson2143940}{RGB}{214,39,40}
\definecolor{darkgray176}{RGB}{176,176,176}
\definecolor{darkorange25512714}{RGB}{255,127,14}
\definecolor{forestgreen4416044}{RGB}{44,160,44}
\definecolor{lightgray204}{RGB}{204,204,204}
\definecolor{steelblue31119180}{RGB}{31,119,180}

\begin{axis}[
legend cell align={left},
legend style={
  fill opacity=0.8,
  draw opacity=1,
  text opacity=1,
  at={(0.03,0.97)},
  anchor=north west,
  draw=lightgray204
},
tick align=outside,
tick pos=left,
x grid style={darkgray176},
xlabel={iteration},
xmin=1375, xmax=26125,
xtick style={color=black},
y grid style={darkgray176},
ylabel={accuracy},
ymin=0.2021728515625, ymax=0.9321044921875,
ytick style={color=black}
]
\addplot [semithick, steelblue31119180]
table {%
2500 0.271728515625
2550 0.30126953125
2600 0.285400390625
2650 0.293701171875
2700 0.286376953125
2750 0.30859375
2800 0.3076171875
2850 0.290771484375
2900 0.310302734375
2950 0.2978515625
3000 0.303466796875
3050 0.31201171875
3100 0.3076171875
3150 0.30810546875
3200 0.304443359375
3250 0.322509765625
3300 0.306640625
3350 0.31884765625
3400 0.329345703125
3450 0.314697265625
3500 0.320068359375
3550 0.326171875
3600 0.33154296875
3650 0.3173828125
3700 0.30859375
3750 0.302978515625
3800 0.330078125
3850 0.323486328125
3900 0.335693359375
3950 0.326904296875
4000 0.31689453125
4050 0.301513671875
4100 0.322265625
4150 0.332763671875
4200 0.353271484375
4250 0.310546875
4300 0.344970703125
4350 0.3408203125
4400 0.341064453125
4450 0.3310546875
4500 0.324951171875
4550 0.359375
4600 0.334228515625
4650 0.346435546875
4700 0.361572265625
4750 0.35205078125
4800 0.341552734375
4850 0.342529296875
4900 0.324462890625
4950 0.348388671875
5000 0.3125
5050 0.343017578125
5100 0.360595703125
5150 0.365478515625
5200 0.344970703125
5250 0.346923828125
5300 0.350830078125
5350 0.35595703125
5400 0.357177734375
5450 0.3466796875
5500 0.368408203125
5550 0.36376953125
5600 0.364501953125
5650 0.3642578125
5700 0.34814453125
5750 0.370849609375
5800 0.3515625
5850 0.3701171875
5900 0.362548828125
5950 0.365966796875
6000 0.365478515625
6050 0.3681640625
6100 0.381103515625
6150 0.36572265625
6200 0.365966796875
6250 0.372314453125
6300 0.36669921875
6350 0.375
6400 0.368896484375
6450 0.358642578125
6500 0.36962890625
6550 0.36669921875
6600 0.368896484375
6650 0.38134765625
6700 0.373046875
6750 0.381103515625
6800 0.377197265625
6850 0.369873046875
6900 0.374267578125
6950 0.3896484375
7000 0.3818359375
7050 0.3779296875
7100 0.3876953125
7150 0.37646484375
7200 0.391357421875
7250 0.38671875
7300 0.377685546875
7350 0.37939453125
7400 0.379150390625
7450 0.388916015625
7500 0.3828125
7550 0.393310546875
7600 0.39306640625
7650 0.3984375
7700 0.38720703125
7750 0.401611328125
7800 0.384521484375
7850 0.4033203125
7900 0.396484375
7950 0.406005859375
8000 0.38232421875
8050 0.399169921875
8100 0.396240234375
8150 0.403076171875
8200 0.385986328125
8250 0.406005859375
8300 0.39013671875
8350 0.389404296875
8400 0.39404296875
8450 0.39404296875
8500 0.39794921875
8550 0.390380859375
8600 0.396240234375
8650 0.391357421875
8700 0.397705078125
8750 0.398681640625
8800 0.394287109375
8850 0.403076171875
8900 0.39990234375
8950 0.404541015625
9000 0.406494140625
9050 0.38671875
9100 0.4013671875
9150 0.39990234375
9200 0.40771484375
9250 0.38623046875
9300 0.399169921875
9350 0.403076171875
9400 0.40380859375
9450 0.408935546875
9500 0.406005859375
9550 0.41455078125
9600 0.369384765625
9650 0.40673828125
9700 0.39697265625
9750 0.411376953125
9800 0.413818359375
9850 0.408447265625
9900 0.42138671875
9950 0.407470703125
10000 0.41455078125
10050 0.408935546875
10100 0.41455078125
10150 0.4091796875
10200 0.42333984375
10250 0.41845703125
10300 0.409423828125
10350 0.406494140625
10400 0.42822265625
10450 0.417724609375
10500 0.408203125
10550 0.410888671875
10600 0.42529296875
10650 0.423095703125
10700 0.40478515625
10750 0.40380859375
10800 0.42822265625
10850 0.418701171875
10900 0.42138671875
10950 0.42822265625
11000 0.421875
11050 0.42236328125
11100 0.431884765625
11150 0.4208984375
11200 0.412109375
11250 0.423095703125
11300 0.402587890625
11350 0.421875
11400 0.412353515625
11450 0.42138671875
11500 0.422119140625
11550 0.427490234375
11600 0.4169921875
11650 0.432373046875
11700 0.426025390625
11750 0.423583984375
11800 0.418701171875
11850 0.42578125
11900 0.426513671875
11950 0.423095703125
12000 0.41796875
12050 0.424072265625
12100 0.426025390625
12150 0.446044921875
12200 0.416015625
12250 0.436767578125
12300 0.431884765625
12350 0.43115234375
12400 0.434814453125
12450 0.418212890625
12500 0.4228515625
12550 0.42529296875
12600 0.440673828125
12650 0.4287109375
12700 0.44287109375
12750 0.432861328125
12800 0.399169921875
12850 0.436767578125
12900 0.402099609375
12950 0.434814453125
13000 0.430908203125
13050 0.4296875
13100 0.44287109375
13150 0.430908203125
13200 0.444580078125
13250 0.42626953125
13300 0.438232421875
13350 0.442138671875
13400 0.435791015625
13450 0.441650390625
13500 0.4375
13550 0.430419921875
13600 0.433837890625
13650 0.430908203125
13700 0.440185546875
13750 0.433837890625
13800 0.4384765625
13850 0.439697265625
13900 0.423095703125
13950 0.447021484375
14000 0.43212890625
14050 0.43359375
14100 0.44189453125
14150 0.4345703125
14200 0.437255859375
14250 0.43310546875
14300 0.439453125
14350 0.451171875
14400 0.44091796875
14450 0.42138671875
14500 0.443603515625
14550 0.444091796875
14600 0.425537109375
14650 0.4443359375
14700 0.426513671875
14750 0.434326171875
14800 0.441650390625
14850 0.44921875
14900 0.442138671875
14950 0.435791015625
15000 0.441650390625
15050 0.441650390625
15100 0.443115234375
15150 0.43115234375
15200 0.4541015625
15250 0.43359375
15300 0.453125
15350 0.450439453125
15400 0.449951171875
15450 0.4365234375
15500 0.44287109375
15550 0.44384765625
15600 0.451904296875
15650 0.436279296875
15700 0.44873046875
15750 0.447509765625
15800 0.4501953125
15850 0.44140625
15900 0.440673828125
15950 0.440185546875
16000 0.4501953125
16050 0.44091796875
16100 0.44189453125
16150 0.439453125
16200 0.450439453125
16250 0.46240234375
16300 0.44677734375
16350 0.46142578125
16400 0.437744140625
16450 0.445068359375
16500 0.446533203125
16550 0.432861328125
16600 0.44140625
16650 0.4462890625
16700 0.451171875
16750 0.4501953125
16800 0.438232421875
16850 0.44580078125
16900 0.447021484375
16950 0.452880859375
17000 0.438720703125
17050 0.4521484375
17100 0.44677734375
17150 0.459228515625
17200 0.4384765625
17250 0.454833984375
17300 0.451904296875
17350 0.450439453125
17400 0.44287109375
17450 0.453125
17500 0.4453125
17550 0.443603515625
17600 0.447265625
17650 0.44580078125
17700 0.435302734375
17750 0.452392578125
17800 0.460693359375
17850 0.44384765625
17900 0.4345703125
17950 0.438232421875
18000 0.447509765625
18050 0.442626953125
18100 0.442138671875
18150 0.4404296875
18200 0.44873046875
18250 0.454345703125
18300 0.450927734375
18350 0.447021484375
18400 0.4443359375
18450 0.4482421875
18500 0.43798828125
18550 0.457275390625
18600 0.441650390625
18650 0.44677734375
18700 0.44482421875
18750 0.4453125
18800 0.44580078125
18850 0.44482421875
18900 0.44677734375
18950 0.45703125
19000 0.4365234375
19050 0.44775390625
19100 0.439208984375
19150 0.450439453125
19200 0.4423828125
19250 0.458251953125
19300 0.45654296875
19350 0.46240234375
19400 0.462646484375
19450 0.463623046875
19500 0.46142578125
19550 0.4638671875
19600 0.46533203125
19650 0.466064453125
19700 0.4609375
19750 0.45947265625
19800 0.4658203125
19850 0.459228515625
19900 0.4599609375
19950 0.46142578125
20000 0.464599609375
20050 0.460205078125
20100 0.462646484375
20150 0.468017578125
20200 0.465576171875
20250 0.46240234375
20300 0.46044921875
20350 0.463134765625
20400 0.4609375
20450 0.45947265625
20500 0.465576171875
20550 0.4599609375
20600 0.46435546875
20650 0.46240234375
20700 0.46826171875
20750 0.462158203125
20800 0.458984375
20850 0.462890625
20900 0.463623046875
20950 0.452880859375
21000 0.46337890625
21050 0.458984375
21100 0.455322265625
21150 0.459228515625
21200 0.4599609375
21250 0.46435546875
21300 0.463134765625
21350 0.46240234375
21400 0.46044921875
21450 0.455322265625
21500 0.46435546875
21550 0.462646484375
21600 0.45751953125
21650 0.46240234375
21700 0.462158203125
21750 0.458984375
21800 0.456787109375
21850 0.4619140625
21900 0.45947265625
21950 0.45751953125
22000 0.45703125
22050 0.46484375
22100 0.46435546875
22150 0.4599609375
22200 0.456787109375
22250 0.4580078125
22300 0.462646484375
22350 0.46240234375
22400 0.458740234375
22450 0.4599609375
22500 0.45751953125
22550 0.4638671875
22600 0.4609375
22650 0.460205078125
22700 0.457763671875
22750 0.461181640625
22800 0.45751953125
22850 0.46044921875
22900 0.463623046875
22950 0.45751953125
23000 0.45947265625
23050 0.453125
23100 0.458740234375
23150 0.4541015625
23200 0.453369140625
23250 0.45556640625
23300 0.456298828125
23350 0.4609375
23400 0.45703125
23450 0.45703125
23500 0.456787109375
23550 0.460693359375
23600 0.455078125
23650 0.455322265625
23700 0.458740234375
23750 0.457275390625
23800 0.45361328125
23850 0.45947265625
23900 0.456298828125
23950 0.453857421875
24000 0.455810546875
24050 0.45556640625
24100 0.4609375
24150 0.45849609375
24200 0.45849609375
24250 0.45458984375
24300 0.454833984375
24350 0.455322265625
24400 0.4580078125
24450 0.45068359375
24500 0.45458984375
24550 0.459716796875
24600 0.45458984375
24650 0.4541015625
24700 0.4482421875
24750 0.45751953125
24800 0.448974609375
24850 0.447998046875
24900 0.458740234375
24950 0.455078125
25000 0.454833984375
};
\addlegendentry{test robust}
\addplot [semithick, darkorange25512714]
table {%
2500 0.40234375
2550 0.44775390625
2600 0.430419921875
2650 0.42431640625
2700 0.42236328125
2750 0.458740234375
2800 0.442626953125
2850 0.41943359375
2900 0.445068359375
2950 0.442626953125
3000 0.448486328125
3050 0.46240234375
3100 0.465087890625
3150 0.464111328125
3200 0.45068359375
3250 0.458251953125
3300 0.462890625
3350 0.453125
3400 0.480712890625
3450 0.477294921875
3500 0.473388671875
3550 0.4833984375
3600 0.4833984375
3650 0.469482421875
3700 0.473876953125
3750 0.474609375
3800 0.48779296875
3850 0.463134765625
3900 0.517822265625
3950 0.48681640625
4000 0.483642578125
4050 0.462646484375
4100 0.496826171875
4150 0.49609375
4200 0.507568359375
4250 0.46435546875
4300 0.512451171875
4350 0.4931640625
4400 0.52685546875
4450 0.51220703125
4500 0.51025390625
4550 0.534912109375
4600 0.51611328125
4650 0.49951171875
4700 0.5224609375
4750 0.496826171875
4800 0.53515625
4850 0.520263671875
4900 0.511474609375
4950 0.5224609375
5000 0.50244140625
5050 0.533447265625
5100 0.5322265625
5150 0.546630859375
5200 0.549072265625
5250 0.5517578125
5300 0.554443359375
5350 0.534912109375
5400 0.563720703125
5450 0.52783203125
5500 0.555908203125
5550 0.545166015625
5600 0.551513671875
5650 0.552978515625
5700 0.538818359375
5750 0.576171875
5800 0.538818359375
5850 0.573974609375
5900 0.56005859375
5950 0.555419921875
6000 0.54931640625
6050 0.57275390625
6100 0.554443359375
6150 0.56005859375
6200 0.55615234375
6250 0.577392578125
6300 0.565185546875
6350 0.561279296875
6400 0.573974609375
6450 0.583740234375
6500 0.59375
6550 0.583740234375
6600 0.598876953125
6650 0.572265625
6700 0.57666015625
6750 0.575927734375
6800 0.588134765625
6850 0.59619140625
6900 0.6025390625
6950 0.586669921875
7000 0.59130859375
7050 0.6064453125
7100 0.599365234375
7150 0.607666015625
7200 0.601318359375
7250 0.599365234375
7300 0.60205078125
7350 0.61279296875
7400 0.615478515625
7450 0.622802734375
7500 0.60791015625
7550 0.60595703125
7600 0.599853515625
7650 0.63232421875
7700 0.62646484375
7750 0.608154296875
7800 0.58154296875
7850 0.638427734375
7900 0.619873046875
7950 0.622802734375
8000 0.601806640625
8050 0.643798828125
8100 0.617919921875
8150 0.62158203125
8200 0.6162109375
8250 0.624755859375
8300 0.619384765625
8350 0.612548828125
8400 0.62109375
8450 0.630126953125
8500 0.640869140625
8550 0.629638671875
8600 0.6396484375
8650 0.598876953125
8700 0.64697265625
8750 0.63330078125
8800 0.62451171875
8850 0.638427734375
8900 0.646484375
8950 0.6396484375
9000 0.6650390625
9050 0.6416015625
9100 0.660888671875
9150 0.661865234375
9200 0.65625
9250 0.663818359375
9300 0.66845703125
9350 0.648681640625
9400 0.635498046875
9450 0.63427734375
9500 0.652099609375
9550 0.636474609375
9600 0.622802734375
9650 0.666259765625
9700 0.63525390625
9750 0.644287109375
9800 0.669921875
9850 0.677978515625
9900 0.672607421875
9950 0.63623046875
10000 0.6669921875
10050 0.6474609375
10100 0.6796875
10150 0.670166015625
10200 0.6640625
10250 0.6865234375
10300 0.642578125
10350 0.668212890625
10400 0.683349609375
10450 0.661376953125
10500 0.66455078125
10550 0.67578125
10600 0.64013671875
10650 0.67333984375
10700 0.684326171875
10750 0.686767578125
10800 0.674072265625
10850 0.677490234375
10900 0.6806640625
10950 0.66162109375
11000 0.673095703125
11050 0.6875
11100 0.648681640625
11150 0.65869140625
11200 0.684814453125
11250 0.670166015625
11300 0.679443359375
11350 0.704833984375
11400 0.674560546875
11450 0.696044921875
11500 0.68798828125
11550 0.6455078125
11600 0.687255859375
11650 0.671630859375
11700 0.689453125
11750 0.677734375
11800 0.69140625
11850 0.69384765625
11900 0.663818359375
11950 0.7060546875
12000 0.67919921875
12050 0.706298828125
12100 0.6875
12150 0.69677734375
12200 0.7099609375
12250 0.700927734375
12300 0.697021484375
12350 0.69873046875
12400 0.69775390625
12450 0.70361328125
12500 0.696044921875
12550 0.682373046875
12600 0.691650390625
12650 0.69873046875
12700 0.68994140625
12750 0.7177734375
12800 0.695068359375
12850 0.71484375
12900 0.71484375
12950 0.69970703125
13000 0.6982421875
13050 0.705322265625
13100 0.68359375
13150 0.703125
13200 0.69140625
13250 0.697021484375
13300 0.695556640625
13350 0.69287109375
13400 0.723388671875
13450 0.7109375
13500 0.711181640625
13550 0.70166015625
13600 0.703369140625
13650 0.6923828125
13700 0.706787109375
13750 0.706787109375
13800 0.715576171875
13850 0.7001953125
13900 0.7060546875
13950 0.731201171875
14000 0.7109375
14050 0.71923828125
14100 0.7080078125
14150 0.7041015625
14200 0.72119140625
14250 0.713134765625
14300 0.726806640625
14350 0.724365234375
14400 0.704833984375
14450 0.73583984375
14500 0.725830078125
14550 0.726318359375
14600 0.723388671875
14650 0.724365234375
14700 0.7353515625
14750 0.732421875
14800 0.73681640625
14850 0.71875
14900 0.73388671875
14950 0.718017578125
15000 0.7255859375
15050 0.72021484375
15100 0.711669921875
15150 0.7216796875
15200 0.7197265625
15250 0.7265625
15300 0.735595703125
15350 0.733154296875
15400 0.74365234375
15450 0.731689453125
15500 0.716064453125
15550 0.719970703125
15600 0.736328125
15650 0.734375
15700 0.7314453125
15750 0.709716796875
15800 0.717041015625
15850 0.7216796875
15900 0.734130859375
15950 0.728271484375
16000 0.735595703125
16050 0.73388671875
16100 0.73876953125
16150 0.743408203125
16200 0.737060546875
16250 0.736328125
16300 0.72412109375
16350 0.738525390625
16400 0.730224609375
16450 0.736572265625
16500 0.724365234375
16550 0.7412109375
16600 0.74462890625
16650 0.742919921875
16700 0.747314453125
16750 0.737548828125
16800 0.744384765625
16850 0.737060546875
16900 0.704833984375
16950 0.7431640625
17000 0.749267578125
17050 0.72216796875
17100 0.715087890625
17150 0.721923828125
17200 0.74609375
17250 0.73583984375
17300 0.750732421875
17350 0.73974609375
17400 0.752685546875
17450 0.73779296875
17500 0.759521484375
17550 0.744873046875
17600 0.756103515625
17650 0.744873046875
17700 0.763427734375
17750 0.7470703125
17800 0.755615234375
17850 0.7470703125
17900 0.7451171875
17950 0.751953125
18000 0.744140625
18050 0.752685546875
18100 0.748779296875
18150 0.759521484375
18200 0.75439453125
18250 0.73193359375
18300 0.73779296875
18350 0.74560546875
18400 0.73095703125
18450 0.74072265625
18500 0.749755859375
18550 0.750732421875
18600 0.74169921875
18650 0.75048828125
18700 0.748046875
18750 0.7587890625
18800 0.757080078125
18850 0.758544921875
18900 0.7587890625
18950 0.744384765625
19000 0.767333984375
19050 0.753173828125
19100 0.7529296875
19150 0.732421875
19200 0.7451171875
19250 0.776611328125
19300 0.78173828125
19350 0.78271484375
19400 0.791259765625
19450 0.788330078125
19500 0.78564453125
19550 0.788818359375
19600 0.7919921875
19650 0.7900390625
19700 0.7900390625
19750 0.78759765625
19800 0.78759765625
19850 0.788330078125
19900 0.789794921875
19950 0.789794921875
20000 0.785888671875
20050 0.794677734375
20100 0.78857421875
20150 0.791015625
20200 0.795654296875
20250 0.789794921875
20300 0.793212890625
20350 0.79443359375
20400 0.796142578125
20450 0.78857421875
20500 0.793212890625
20550 0.789794921875
20600 0.793701171875
20650 0.7900390625
20700 0.788818359375
20750 0.792236328125
20800 0.796875
20850 0.792724609375
20900 0.79833984375
20950 0.795654296875
21000 0.79296875
21050 0.79345703125
21100 0.7958984375
21150 0.79638671875
21200 0.79443359375
21250 0.793212890625
21300 0.798828125
21350 0.795654296875
21400 0.7958984375
21450 0.79833984375
21500 0.79345703125
21550 0.7939453125
21600 0.79541015625
21650 0.79443359375
21700 0.796142578125
21750 0.79345703125
21800 0.797607421875
21850 0.788818359375
21900 0.798095703125
21950 0.7998046875
22000 0.79736328125
22050 0.79931640625
22100 0.79443359375
22150 0.791748046875
22200 0.7939453125
22250 0.803466796875
22300 0.798095703125
22350 0.79443359375
22400 0.79296875
22450 0.796630859375
22500 0.8017578125
22550 0.793701171875
22600 0.79150390625
22650 0.796875
22700 0.79736328125
22750 0.80224609375
22800 0.794921875
22850 0.799072265625
22900 0.79736328125
22950 0.796630859375
23000 0.79541015625
23050 0.79736328125
23100 0.796630859375
23150 0.796875
23200 0.79248046875
23250 0.79736328125
23300 0.79296875
23350 0.801025390625
23400 0.798583984375
23450 0.793701171875
23500 0.79638671875
23550 0.795654296875
23600 0.797607421875
23650 0.801025390625
23700 0.79931640625
23750 0.7958984375
23800 0.7978515625
23850 0.79248046875
23900 0.794921875
23950 0.79931640625
24000 0.802734375
24050 0.791259765625
24100 0.79443359375
24150 0.797119140625
24200 0.80126953125
24250 0.79833984375
24300 0.8017578125
24350 0.798828125
24400 0.79931640625
24450 0.80078125
24500 0.796875
24550 0.796630859375
24600 0.797119140625
24650 0.80029296875
24700 0.80517578125
24750 0.802001953125
24800 0.797607421875
24850 0.8046875
24900 0.795166015625
24950 0.799560546875
25000 0.800537109375
};
\addlegendentry{test clean}
\addplot [semithick, forestgreen4416044]
table {%
2500 0.368896484375
2550 0.41064453125
2600 0.37841796875
2650 0.3876953125
2700 0.38671875
2750 0.426513671875
2800 0.407470703125
2850 0.386474609375
2900 0.4287109375
2950 0.414794921875
3000 0.440673828125
3050 0.43408203125
3100 0.431396484375
3150 0.423095703125
3200 0.417724609375
3250 0.420166015625
3300 0.45068359375
3350 0.424072265625
3400 0.451171875
3450 0.454345703125
3500 0.443359375
3550 0.454833984375
3600 0.463134765625
3650 0.4482421875
3700 0.450439453125
3750 0.457763671875
3800 0.467041015625
3850 0.44140625
3900 0.5029296875
3950 0.470458984375
4000 0.45458984375
4050 0.4375
4100 0.47314453125
4150 0.474365234375
4200 0.485595703125
4250 0.445556640625
4300 0.489501953125
4350 0.45703125
4400 0.49609375
4450 0.49365234375
4500 0.469970703125
4550 0.50732421875
4600 0.486328125
4650 0.473388671875
4700 0.48828125
4750 0.47509765625
4800 0.507080078125
4850 0.4990234375
4900 0.4921875
4950 0.508544921875
5000 0.492919921875
5050 0.523681640625
5100 0.51611328125
5150 0.531982421875
5200 0.554443359375
5250 0.536865234375
5300 0.5302734375
5350 0.5224609375
5400 0.5751953125
5450 0.52880859375
5500 0.5498046875
5550 0.546875
5600 0.5458984375
5650 0.5341796875
5700 0.528076171875
5750 0.55322265625
5800 0.516357421875
5850 0.55126953125
5900 0.521240234375
5950 0.531005859375
6000 0.52978515625
6050 0.554931640625
6100 0.54931640625
6150 0.548095703125
6200 0.54052734375
6250 0.5615234375
6300 0.55029296875
6350 0.551025390625
6400 0.56005859375
6450 0.57958984375
6500 0.571533203125
6550 0.57373046875
6600 0.59521484375
6650 0.555419921875
6700 0.555908203125
6750 0.552001953125
6800 0.577880859375
6850 0.56787109375
6900 0.574462890625
6950 0.589599609375
7000 0.57568359375
7050 0.59033203125
7100 0.58740234375
7150 0.602783203125
7200 0.58642578125
7250 0.5927734375
7300 0.5869140625
7350 0.599609375
7400 0.600341796875
7450 0.623291015625
7500 0.61572265625
7550 0.6064453125
7600 0.605224609375
7650 0.625244140625
7700 0.630859375
7750 0.61572265625
7800 0.591064453125
7850 0.636962890625
7900 0.62109375
7950 0.624755859375
8000 0.6005859375
8050 0.6337890625
8100 0.62255859375
8150 0.627197265625
8200 0.6201171875
8250 0.6259765625
8300 0.6025390625
8350 0.5966796875
8400 0.604248046875
8450 0.639404296875
8500 0.647705078125
8550 0.632080078125
8600 0.652587890625
8650 0.60986328125
8700 0.65087890625
8750 0.62060546875
8800 0.6162109375
8850 0.649169921875
8900 0.652099609375
8950 0.6357421875
9000 0.656005859375
9050 0.65234375
9100 0.6591796875
9150 0.657470703125
9200 0.651123046875
9250 0.668212890625
9300 0.6572265625
9350 0.64111328125
9400 0.639404296875
9450 0.633544921875
9500 0.653564453125
9550 0.63330078125
9600 0.619873046875
9650 0.667724609375
9700 0.63720703125
9750 0.64306640625
9800 0.68115234375
9850 0.6884765625
9900 0.6748046875
9950 0.6484375
10000 0.691650390625
10050 0.669677734375
10100 0.68505859375
10150 0.681884765625
10200 0.691650390625
10250 0.697265625
10300 0.64599609375
10350 0.6865234375
10400 0.69677734375
10450 0.6865234375
10500 0.6845703125
10550 0.6845703125
10600 0.658447265625
10650 0.679443359375
10700 0.694091796875
10750 0.69970703125
10800 0.689208984375
10850 0.68212890625
10900 0.684326171875
10950 0.673095703125
11000 0.677978515625
11050 0.697265625
11100 0.66650390625
11150 0.7021484375
11200 0.72412109375
11250 0.7021484375
11300 0.711181640625
11350 0.73193359375
11400 0.69189453125
11450 0.7099609375
11500 0.694580078125
11550 0.675537109375
11600 0.704833984375
11650 0.689697265625
11700 0.7001953125
11750 0.714599609375
11800 0.713134765625
11850 0.71728515625
11900 0.68994140625
11950 0.724853515625
12000 0.703125
12050 0.7177734375
12100 0.697998046875
12150 0.711669921875
12200 0.71240234375
12250 0.716064453125
12300 0.72705078125
12350 0.72509765625
12400 0.711669921875
12450 0.70556640625
12500 0.734619140625
12550 0.72119140625
12600 0.71630859375
12650 0.735107421875
12700 0.73486328125
12750 0.738525390625
12800 0.728271484375
12850 0.740478515625
12900 0.751708984375
12950 0.72021484375
13000 0.71728515625
13050 0.724853515625
13100 0.712890625
13150 0.725830078125
13200 0.70361328125
13250 0.71142578125
13300 0.7177734375
13350 0.71435546875
13400 0.74658203125
13450 0.7314453125
13500 0.74072265625
13550 0.729736328125
13600 0.725830078125
13650 0.737060546875
13700 0.736328125
13750 0.744384765625
13800 0.75048828125
13850 0.736572265625
13900 0.739013671875
13950 0.753173828125
14000 0.73779296875
14050 0.73876953125
14100 0.738525390625
14150 0.711181640625
14200 0.738525390625
14250 0.75390625
14300 0.763427734375
14350 0.753173828125
14400 0.743896484375
14450 0.775146484375
14500 0.76171875
14550 0.75146484375
14600 0.752197265625
14650 0.744873046875
14700 0.761962890625
14750 0.744384765625
14800 0.775390625
14850 0.746826171875
14900 0.763916015625
14950 0.741943359375
15000 0.772216796875
15050 0.74267578125
15100 0.73828125
15150 0.74853515625
15200 0.761474609375
15250 0.762451171875
15300 0.76220703125
15350 0.761962890625
15400 0.787841796875
15450 0.77880859375
15500 0.759765625
15550 0.76123046875
15600 0.779052734375
15650 0.768798828125
15700 0.7666015625
15750 0.75
15800 0.7607421875
15850 0.76123046875
15900 0.7734375
15950 0.783935546875
16000 0.78271484375
16050 0.78955078125
16100 0.78759765625
16150 0.79296875
16200 0.776611328125
16250 0.776123046875
16300 0.759033203125
16350 0.78173828125
16400 0.77197265625
16450 0.778564453125
16500 0.763427734375
16550 0.7861328125
16600 0.789306640625
16650 0.783447265625
16700 0.765869140625
16750 0.797119140625
16800 0.783447265625
16850 0.78076171875
16900 0.75439453125
16950 0.79541015625
17000 0.79833984375
17050 0.76220703125
17100 0.77099609375
17150 0.77587890625
17200 0.795166015625
17250 0.78662109375
17300 0.808349609375
17350 0.783203125
17400 0.79541015625
17450 0.771728515625
17500 0.82421875
17550 0.791259765625
17600 0.803466796875
17650 0.795654296875
17700 0.824951171875
17750 0.80029296875
17800 0.810546875
17850 0.791015625
17900 0.800537109375
17950 0.801513671875
18000 0.77978515625
18050 0.79443359375
18100 0.802001953125
18150 0.8095703125
18200 0.802734375
18250 0.785400390625
18300 0.798095703125
18350 0.80126953125
18400 0.77783203125
18450 0.80419921875
18500 0.802001953125
18550 0.812744140625
18600 0.79052734375
18650 0.81640625
18700 0.796630859375
18750 0.809814453125
18800 0.793212890625
18850 0.81396484375
18900 0.810302734375
18950 0.79638671875
19000 0.8154296875
19050 0.811767578125
19100 0.8046875
19150 0.7841796875
19200 0.8017578125
19250 0.841796875
19300 0.848876953125
19350 0.8447265625
19400 0.84912109375
19450 0.844970703125
19500 0.84716796875
19550 0.853515625
19600 0.860595703125
19650 0.858154296875
19700 0.858154296875
19750 0.8505859375
19800 0.8525390625
19850 0.85693359375
19900 0.8583984375
19950 0.853271484375
20000 0.865478515625
20050 0.870849609375
20100 0.8583984375
20150 0.864990234375
20200 0.85205078125
20250 0.84765625
20300 0.84912109375
20350 0.85595703125
20400 0.868896484375
20450 0.856689453125
20500 0.86083984375
20550 0.86328125
20600 0.85693359375
20650 0.861328125
20700 0.854248046875
20750 0.869140625
20800 0.86962890625
20850 0.86669921875
20900 0.8642578125
20950 0.866455078125
21000 0.855224609375
21050 0.856201171875
21100 0.8681640625
21150 0.86328125
21200 0.861083984375
21250 0.859375
21300 0.86474609375
21350 0.869140625
21400 0.86376953125
21450 0.869384765625
21500 0.86376953125
21550 0.866943359375
21600 0.865478515625
21650 0.864501953125
21700 0.87451171875
21750 0.880126953125
21800 0.877685546875
21850 0.866943359375
21900 0.880615234375
21950 0.884765625
22000 0.877685546875
22050 0.87255859375
22100 0.87060546875
22150 0.872314453125
22200 0.869873046875
22250 0.87255859375
22300 0.877197265625
22350 0.86865234375
22400 0.869384765625
22450 0.869384765625
22500 0.882568359375
22550 0.880126953125
22600 0.87744140625
22650 0.874755859375
22700 0.8818359375
22750 0.87744140625
22800 0.87353515625
22850 0.8828125
22900 0.888427734375
22950 0.883544921875
23000 0.879638671875
23050 0.881103515625
23100 0.88232421875
23150 0.88037109375
23200 0.88134765625
23250 0.868408203125
23300 0.87109375
23350 0.871337890625
23400 0.8642578125
23450 0.8818359375
23500 0.882080078125
23550 0.87646484375
23600 0.884765625
23650 0.8916015625
23700 0.880615234375
23750 0.87158203125
23800 0.88330078125
23850 0.88330078125
23900 0.87646484375
23950 0.875244140625
24000 0.88330078125
24050 0.87939453125
24100 0.87890625
24150 0.8818359375
24200 0.888427734375
24250 0.876953125
24300 0.880126953125
24350 0.875732421875
24400 0.897216796875
24450 0.889404296875
24500 0.88818359375
24550 0.888671875
24600 0.89892578125
24650 0.893310546875
24700 0.892333984375
24750 0.883056640625
24800 0.89306640625
24850 0.890380859375
24900 0.88623046875
24950 0.886962890625
25000 0.89599609375
};
\addlegendentry{train clean}
\addplot [semithick, crimson2143940]
table {%
2500 0.236328125
2550 0.269775390625
2600 0.2353515625
2650 0.267333984375
2700 0.253173828125
2750 0.272705078125
2800 0.2802734375
2850 0.25927734375
2900 0.295166015625
2950 0.282958984375
3000 0.27880859375
3050 0.28857421875
3100 0.28515625
3150 0.271728515625
3200 0.27734375
3250 0.285400390625
3300 0.281494140625
3350 0.290771484375
3400 0.294677734375
3450 0.28515625
3500 0.296630859375
3550 0.2998046875
3600 0.299560546875
3650 0.306396484375
3700 0.290771484375
3750 0.28466796875
3800 0.306396484375
3850 0.30224609375
3900 0.30615234375
3950 0.302978515625
4000 0.293212890625
4050 0.28125
4100 0.293701171875
4150 0.30908203125
4200 0.3095703125
4250 0.298828125
4300 0.309814453125
4350 0.316650390625
4400 0.316650390625
4450 0.31201171875
4500 0.287353515625
4550 0.31689453125
4600 0.31005859375
4650 0.31640625
4700 0.3251953125
4750 0.322265625
4800 0.31298828125
4850 0.322509765625
4900 0.30419921875
4950 0.32861328125
5000 0.301513671875
5050 0.32421875
5100 0.32421875
5150 0.33251953125
5200 0.33740234375
5250 0.33251953125
5300 0.330078125
5350 0.330810546875
5400 0.367431640625
5450 0.343505859375
5500 0.345947265625
5550 0.352783203125
5600 0.351806640625
5650 0.341552734375
5700 0.3291015625
5750 0.342041015625
5800 0.31884765625
5850 0.346923828125
5900 0.3330078125
5950 0.3369140625
6000 0.334228515625
6050 0.341796875
6100 0.3515625
6150 0.337646484375
6200 0.349365234375
6250 0.356689453125
6300 0.350341796875
6350 0.35791015625
6400 0.34716796875
6450 0.345458984375
6500 0.345947265625
6550 0.3544921875
6600 0.353515625
6650 0.34619140625
6700 0.357421875
6750 0.362548828125
6800 0.348388671875
6850 0.338134765625
6900 0.3427734375
6950 0.375732421875
7000 0.36083984375
7050 0.356201171875
7100 0.36669921875
7150 0.360595703125
7200 0.361083984375
7250 0.370361328125
7300 0.370361328125
7350 0.358154296875
7400 0.36083984375
7450 0.367431640625
7500 0.379638671875
7550 0.38134765625
7600 0.384033203125
7650 0.3837890625
7700 0.391357421875
7750 0.388671875
7800 0.385986328125
7850 0.392578125
7900 0.3974609375
7950 0.38623046875
8000 0.367919921875
8050 0.387451171875
8100 0.37841796875
8150 0.385986328125
8200 0.36865234375
8250 0.39306640625
8300 0.373046875
8350 0.361328125
8400 0.376708984375
8450 0.376708984375
8500 0.388427734375
8550 0.374755859375
8600 0.38671875
8650 0.392578125
8700 0.3935546875
8750 0.38232421875
8800 0.396240234375
8850 0.395263671875
8900 0.385009765625
8950 0.38623046875
9000 0.40087890625
9050 0.398681640625
9100 0.39501953125
9150 0.399169921875
9200 0.40478515625
9250 0.382568359375
9300 0.392333984375
9350 0.3935546875
9400 0.39111328125
9450 0.405517578125
9500 0.400634765625
9550 0.402099609375
9600 0.357177734375
9650 0.40478515625
9700 0.38232421875
9750 0.397705078125
9800 0.415771484375
9850 0.4091796875
9900 0.41650390625
9950 0.401123046875
10000 0.425048828125
10050 0.426025390625
10100 0.41796875
10150 0.4072265625
10200 0.430908203125
10250 0.427001953125
10300 0.4111328125
10350 0.40283203125
10400 0.42138671875
10450 0.41845703125
10500 0.412841796875
10550 0.40283203125
10600 0.42138671875
10650 0.42138671875
10700 0.410888671875
10750 0.40869140625
10800 0.433349609375
10850 0.39892578125
10900 0.42578125
10950 0.4248046875
11000 0.427001953125
11050 0.421630859375
11100 0.4208984375
11150 0.4560546875
11200 0.431396484375
11250 0.43798828125
11300 0.406982421875
11350 0.43701171875
11400 0.423583984375
11450 0.42431640625
11500 0.42919921875
11550 0.432861328125
11600 0.425537109375
11650 0.4443359375
11700 0.425537109375
11750 0.443359375
11800 0.43603515625
11850 0.441650390625
11900 0.429443359375
11950 0.441650390625
12000 0.42431640625
12050 0.4296875
12100 0.44189453125
12150 0.461181640625
12200 0.439453125
12250 0.4404296875
12300 0.465576171875
12350 0.447509765625
12400 0.443359375
12450 0.434326171875
12500 0.447509765625
12550 0.438720703125
12600 0.458740234375
12650 0.452392578125
12700 0.47705078125
12750 0.455078125
12800 0.42236328125
12850 0.453369140625
12900 0.4384765625
12950 0.443603515625
13000 0.44482421875
13050 0.439208984375
13100 0.450927734375
13150 0.44091796875
13200 0.423095703125
13250 0.44921875
13300 0.462646484375
13350 0.458251953125
13400 0.45361328125
13450 0.474853515625
13500 0.4677734375
13550 0.46142578125
13600 0.454345703125
13650 0.470703125
13700 0.473876953125
13750 0.46533203125
13800 0.455810546875
13850 0.49072265625
13900 0.459228515625
13950 0.470458984375
14000 0.452392578125
14050 0.462890625
14100 0.45947265625
14150 0.4384765625
14200 0.43896484375
14250 0.47607421875
14300 0.480712890625
14350 0.48095703125
14400 0.45654296875
14450 0.47314453125
14500 0.482666015625
14550 0.46923828125
14600 0.45458984375
14650 0.4697265625
14700 0.44775390625
14750 0.46142578125
14800 0.48583984375
14850 0.485107421875
14900 0.457763671875
14950 0.449462890625
15000 0.490234375
15050 0.47509765625
15100 0.463623046875
15150 0.452880859375
15200 0.49560546875
15250 0.4755859375
15300 0.476318359375
15350 0.4794921875
15400 0.49853515625
15450 0.490234375
15500 0.487548828125
15550 0.485107421875
15600 0.50341796875
15650 0.482177734375
15700 0.486083984375
15750 0.491943359375
15800 0.475830078125
15850 0.480224609375
15900 0.47412109375
15950 0.49609375
16000 0.503662109375
16050 0.49462890625
16100 0.48486328125
16150 0.498779296875
16200 0.482177734375
16250 0.498779296875
16300 0.475830078125
16350 0.501953125
16400 0.485107421875
16450 0.47607421875
16500 0.48193359375
16550 0.4921875
16600 0.495849609375
16650 0.48974609375
16700 0.49169921875
16750 0.52001953125
16800 0.487060546875
16850 0.485107421875
16900 0.495849609375
16950 0.50439453125
17000 0.505126953125
17050 0.496826171875
17100 0.497314453125
17150 0.51708984375
17200 0.475830078125
17250 0.50634765625
17300 0.513427734375
17350 0.500732421875
17400 0.493408203125
17450 0.492431640625
17500 0.524169921875
17550 0.518310546875
17600 0.498779296875
17650 0.50146484375
17700 0.513671875
17750 0.496337890625
17800 0.507568359375
17850 0.47802734375
17900 0.50048828125
17950 0.498046875
18000 0.494384765625
18050 0.498291015625
18100 0.512451171875
18150 0.494873046875
18200 0.5009765625
18250 0.518310546875
18300 0.517578125
18350 0.51611328125
18400 0.493408203125
18450 0.53076171875
18500 0.5078125
18550 0.51513671875
18600 0.50048828125
18650 0.513916015625
18700 0.5
18750 0.49169921875
18800 0.48974609375
18850 0.5244140625
18900 0.502685546875
18950 0.50048828125
19000 0.4951171875
19050 0.51904296875
19100 0.498291015625
19150 0.50439453125
19200 0.50048828125
19250 0.553955078125
19300 0.55908203125
19350 0.552978515625
19400 0.560546875
19450 0.5517578125
19500 0.555419921875
19550 0.5576171875
19600 0.550537109375
19650 0.557861328125
19700 0.5546875
19750 0.54833984375
19800 0.56396484375
19850 0.560546875
19900 0.56201171875
19950 0.562255859375
20000 0.5654296875
20050 0.553466796875
20100 0.567138671875
20150 0.56787109375
20200 0.561767578125
20250 0.547607421875
20300 0.55322265625
20350 0.551025390625
20400 0.573486328125
20450 0.5654296875
20500 0.568603515625
20550 0.578369140625
20600 0.578857421875
20650 0.5703125
20700 0.5712890625
20750 0.56884765625
20800 0.575927734375
20850 0.572998046875
20900 0.567626953125
20950 0.559814453125
21000 0.564453125
21050 0.55126953125
21100 0.557373046875
21150 0.573974609375
21200 0.56982421875
21250 0.575439453125
21300 0.5712890625
21350 0.56298828125
21400 0.555908203125
21450 0.55419921875
21500 0.560791015625
21550 0.56494140625
21600 0.568359375
21650 0.557373046875
21700 0.576416015625
21750 0.58251953125
21800 0.573974609375
21850 0.5732421875
21900 0.581787109375
21950 0.57177734375
22000 0.564208984375
22050 0.576416015625
22100 0.584716796875
22150 0.576904296875
22200 0.5751953125
22250 0.569580078125
22300 0.588623046875
22350 0.57666015625
22400 0.577392578125
22450 0.5732421875
22500 0.592529296875
22550 0.586181640625
22600 0.587646484375
22650 0.58154296875
22700 0.577392578125
22750 0.572265625
22800 0.571533203125
22850 0.591064453125
22900 0.597900390625
22950 0.59326171875
23000 0.599853515625
23050 0.59521484375
23100 0.6005859375
23150 0.576416015625
23200 0.580322265625
23250 0.591064453125
23300 0.579833984375
23350 0.57861328125
23400 0.57666015625
23450 0.6015625
23500 0.602783203125
23550 0.5927734375
23600 0.58740234375
23650 0.5927734375
23700 0.5859375
23750 0.5908203125
23800 0.58544921875
23850 0.5830078125
23900 0.57763671875
23950 0.577392578125
24000 0.57568359375
24050 0.576416015625
24100 0.5849609375
24150 0.579345703125
24200 0.593994140625
24250 0.602294921875
24300 0.588623046875
24350 0.590087890625
24400 0.60302734375
24450 0.59375
24500 0.59619140625
24550 0.58935546875
24600 0.593994140625
24650 0.594970703125
24700 0.581298828125
24750 0.596435546875
24800 0.61181640625
24850 0.59814453125
24900 0.604248046875
24950 0.6005859375
25000 0.594970703125
};
\addlegendentry{train robust}
\end{axis}

\end{tikzpicture}

%% file: ICLR_appendix_experiments/DyART_acc.tex
\begin{tikzpicture}

\definecolor{crimson2143940}{RGB}{214,39,40}
\definecolor{darkgray176}{RGB}{176,176,176}
\definecolor{darkorange25512714}{RGB}{255,127,14}
\definecolor{forestgreen4416044}{RGB}{44,160,44}
\definecolor{lightgray204}{RGB}{204,204,204}
\definecolor{steelblue31119180}{RGB}{31,119,180}

\begin{axis}[
legend cell align={left},
legend style={
  fill opacity=0.8,
  draw opacity=1,
  text opacity=1,
  at={(0.03,0.97)},
  anchor=north west,
  draw=lightgray204
},
tick align=outside,
tick pos=left,
x grid style={darkgray176},
xlabel={iteration},
xmin=1375, xmax=26125,
xtick style={color=black},
y grid style={darkgray176},
ylabel={accuracy},
ymin=0.25653076171875, ymax=0.92388916015625,
ytick style={color=black}
]
\addplot [semithick, steelblue31119180]
table {%
2500 0.348603515625
2550 0.364716796875
2600 0.36154296875
2650 0.36349609375
2700 0.367890625
2750 0.365693359375
2800 0.359345703125
2850 0.361298828125
2900 0.367646484375
2950 0.358857421875
3000 0.3766796875
3050 0.342744140625
3100 0.362763671875
3150 0.3766796875
3200 0.382294921875
3250 0.37716796875
3300 0.375947265625
3350 0.380830078125
3400 0.375458984375
3450 0.3659375
3500 0.3786328125
3550 0.385224609375
3600 0.378388671875
3650 0.374482421875
3700 0.37375
3750 0.379853515625
3800 0.383759765625
3850 0.378388671875
3900 0.375458984375
3950 0.38498046875
4000 0.38693359375
4050 0.36984375
4100 0.380830078125
4150 0.3805859375
4200 0.374970703125
4250 0.377412109375
4300 0.388642578125
4350 0.39962890625
4400 0.390595703125
4450 0.390107421875
4500 0.397431640625
4550 0.388154296875
4600 0.38693359375
4650 0.389375
4700 0.396943359375
4750 0.39572265625
4800 0.3883984375
4850 0.389130859375
4900 0.383515625
4950 0.39279296875
5000 0.40841796875
5050 0.359345703125
5100 0.396455078125
5150 0.41330078125
5200 0.412568359375
5250 0.39767578125
5300 0.395478515625
5350 0.419892578125
5400 0.39669921875
5450 0.406220703125
5500 0.39865234375
5550 0.39669921875
5600 0.411103515625
5650 0.39083984375
5700 0.380341796875
5750 0.4157421875
5800 0.408173828125
5850 0.40646484375
5900 0.3981640625
5950 0.39669921875
6000 0.399384765625
6050 0.41037109375
6100 0.411103515625
6150 0.408662109375
6200 0.4079296875
6250 0.409638671875
6300 0.413056640625
6350 0.4059765625
6400 0.40060546875
6450 0.4176953125
6500 0.41427734375
6550 0.41818359375
6600 0.421357421875
6650 0.42697265625
6700 0.423798828125
6750 0.421357421875
6800 0.414521484375
6850 0.4079296875
6900 0.422578125
6950 0.410615234375
7000 0.420625
7050 0.4274609375
7100 0.41818359375
7150 0.414033203125
7200 0.42111328125
7250 0.405732421875
7300 0.40841796875
7350 0.430146484375
7400 0.413544921875
7450 0.4157421875
7500 0.42404296875
7550 0.4284375
7600 0.42208984375
7650 0.426484375
7700 0.42111328125
7750 0.41525390625
7800 0.408173828125
7850 0.40890625
7900 0.411103515625
7950 0.42599609375
8000 0.4313671875
8050 0.403291015625
8100 0.416962890625
8150 0.418916015625
8200 0.425751953125
8250 0.422333984375
8300 0.43478515625
8350 0.42208984375
8400 0.427216796875
8450 0.434052734375
8500 0.425263671875
8550 0.42892578125
8600 0.426484375
8650 0.41232421875
8700 0.412568359375
8750 0.426728515625
8800 0.42501953125
8850 0.41818359375
8900 0.424287109375
8950 0.427705078125
9000 0.42208984375
9050 0.436005859375
9100 0.42892578125
9150 0.436494140625
9200 0.43185546875
9250 0.43576171875
9300 0.4274609375
9350 0.437470703125
9400 0.430634765625
9450 0.423798828125
9500 0.415986328125
9550 0.4440625
9600 0.416474609375
9650 0.427705078125
9700 0.426484375
9750 0.448212890625
9800 0.4313671875
9850 0.436005859375
9900 0.432099609375
9950 0.428681640625
10000 0.44259765625
10050 0.434296875
10100 0.4255078125
10150 0.423310546875
10200 0.440888671875
10250 0.431611328125
10300 0.443330078125
10350 0.43625
10400 0.4235546875
10450 0.438203125
10500 0.442353515625
10550 0.4352734375
10600 0.442353515625
10650 0.445283203125
10700 0.43185546875
10750 0.440400390625
10800 0.432587890625
10850 0.426728515625
10900 0.432099609375
10950 0.440400390625
11000 0.4352734375
11050 0.436005859375
11100 0.44162109375
11150 0.43380859375
11200 0.43673828125
11250 0.4430859375
11300 0.43576171875
11350 0.433076171875
11400 0.4333203125
11450 0.439423828125
11500 0.438203125
11550 0.449921875
11600 0.442841796875
11650 0.447724609375
11700 0.446748046875
11750 0.45431640625
11800 0.427705078125
11850 0.4333203125
11900 0.44552734375
11950 0.444794921875
12000 0.4430859375
12050 0.437958984375
12100 0.449677734375
12150 0.432099609375
12200 0.4391796875
12250 0.455537109375
12300 0.447236328125
12350 0.442353515625
12400 0.44552734375
12450 0.435029296875
12500 0.44943359375
12550 0.449677734375
12600 0.439423828125
12650 0.45333984375
12700 0.450166015625
12750 0.44064453125
12800 0.4469921875
12850 0.440400390625
12900 0.449677734375
12950 0.430634765625
13000 0.44455078125
13050 0.449921875
13100 0.449921875
13150 0.454560546875
13200 0.45822265625
13250 0.453828125
13300 0.459931640625
13350 0.446259765625
13400 0.4440625
13450 0.454072265625
13500 0.443818359375
13550 0.458466796875
13600 0.453828125
13650 0.44552734375
13700 0.453095703125
13750 0.457734375
13800 0.4548046875
13850 0.45236328125
13900 0.451875
13950 0.46017578125
14000 0.446015625
14050 0.454560546875
14100 0.4489453125
14150 0.455048828125
14200 0.455048828125
14250 0.44015625
14300 0.451142578125
14350 0.46212890625
14400 0.44796875
14450 0.46017578125
14500 0.45919921875
14550 0.448212890625
14600 0.450166015625
14650 0.450166015625
14700 0.45431640625
14750 0.453828125
14800 0.448212890625
14850 0.45919921875
14900 0.451142578125
14950 0.452119140625
15000 0.463837890625
15050 0.45138671875
15100 0.45041015625
15150 0.443818359375
15200 0.45578125
15250 0.449189453125
15300 0.448212890625
15350 0.460908203125
15400 0.44845703125
15450 0.4508984375
15500 0.449677734375
15550 0.460419921875
15600 0.448701171875
15650 0.45041015625
15700 0.44650390625
15750 0.454560546875
15800 0.453583984375
15850 0.44943359375
15900 0.450654296875
15950 0.45822265625
16000 0.4675
16050 0.449189453125
16100 0.463837890625
16150 0.442353515625
16200 0.461640625
16250 0.451630859375
16300 0.4528515625
16350 0.4675
16400 0.466767578125
16450 0.45529296875
16500 0.459443359375
16550 0.453583984375
16600 0.457001953125
16650 0.4596875
16700 0.4665234375
16750 0.457734375
16800 0.459931640625
16850 0.459931640625
16900 0.462861328125
16950 0.4626171875
17000 0.44650390625
17050 0.453583984375
17100 0.454560546875
17150 0.465546875
17200 0.453828125
17250 0.462373046875
17300 0.4626171875
17350 0.465791015625
17400 0.468720703125
17450 0.458955078125
17500 0.46408203125
17550 0.45919921875
17600 0.460908203125
17650 0.456025390625
17700 0.451142578125
17750 0.454560546875
17800 0.4675
17850 0.45138671875
17900 0.46505859375
17950 0.457978515625
18000 0.46408203125
18050 0.457734375
18100 0.467744140625
18150 0.4626171875
18200 0.457001953125
18250 0.46408203125
18300 0.46212890625
18350 0.451875
18400 0.467255859375
18450 0.46994140625
18500 0.4704296875
18550 0.467255859375
18600 0.460419921875
18650 0.457001953125
18700 0.4567578125
18750 0.453583984375
18800 0.459443359375
18850 0.453095703125
18900 0.46212890625
18950 0.458955078125
19000 0.452607421875
19050 0.468232421875
19100 0.46115234375
19150 0.47091796875
19200 0.446015625
19250 0.47970703125
19300 0.476533203125
19350 0.477021484375
19400 0.478974609375
19450 0.480927734375
19500 0.4860546875
19550 0.48458984375
19600 0.484833984375
19650 0.48458984375
19700 0.482880859375
19750 0.481904296875
19800 0.485078125
19850 0.485078125
19900 0.481904296875
19950 0.48849609375
20000 0.485810546875
20050 0.483857421875
20100 0.48361328125
20150 0.485078125
20200 0.487275390625
20250 0.485078125
20300 0.489228515625
20350 0.48263671875
20400 0.48458984375
20450 0.4880078125
20500 0.4841015625
20550 0.48458984375
20600 0.48751953125
20650 0.485810546875
20700 0.4880078125
20750 0.4841015625
20800 0.48166015625
20850 0.485078125
20900 0.481171875
20950 0.484345703125
21000 0.485322265625
21050 0.48654296875
21100 0.484345703125
21150 0.48556640625
21200 0.487275390625
21250 0.487275390625
21300 0.4909375
21350 0.4860546875
21400 0.485810546875
21450 0.485078125
21500 0.4880078125
21550 0.4841015625
21600 0.488251953125
21650 0.48849609375
21700 0.4880078125
21750 0.4880078125
21800 0.4899609375
21850 0.483857421875
21900 0.492646484375
21950 0.488251953125
22000 0.485322265625
22050 0.48654296875
22100 0.488251953125
22150 0.49044921875
22200 0.484345703125
22250 0.487763671875
22300 0.488251953125
22350 0.48703125
22400 0.485322265625
22450 0.487275390625
22500 0.487275390625
22550 0.48751953125
22600 0.48556640625
22650 0.48751953125
22700 0.482392578125
22750 0.485810546875
22800 0.48361328125
22850 0.486787109375
22900 0.491669921875
22950 0.488984375
23000 0.488251953125
23050 0.49044921875
23100 0.487763671875
23150 0.484345703125
23200 0.488251953125
23250 0.489228515625
23300 0.48556640625
23350 0.48166015625
23400 0.480927734375
23450 0.4880078125
23500 0.487275390625
23550 0.490693359375
23600 0.490693359375
23650 0.492158203125
23700 0.48947265625
23750 0.48361328125
23800 0.48654296875
23850 0.488740234375
23900 0.48849609375
23950 0.488984375
24000 0.485078125
24050 0.486787109375
24100 0.4899609375
24150 0.4909375
24200 0.486787109375
24250 0.49044921875
24300 0.491181640625
24350 0.486787109375
24400 0.48703125
24450 0.4860546875
24500 0.4860546875
24550 0.48751953125
24600 0.492646484375
24650 0.48849609375
24700 0.48556640625
24750 0.4880078125
24800 0.4860546875
24850 0.49240234375
24900 0.48947265625
24950 0.48751953125
25000 0.493134765625
};
\addlegendentry{test robust}
\addplot [semithick, darkorange25512714]
table {%
2500 0.59666015625
2550 0.614970703125
2600 0.58884765625
2650 0.601787109375
2700 0.590556640625
2750 0.61033203125
2800 0.592509765625
2850 0.60203125
2900 0.60349609375
2950 0.611552734375
3000 0.60056640625
3050 0.607646484375
3100 0.6088671875
3150 0.629130859375
3200 0.621806640625
3250 0.6205859375
3300 0.6127734375
3350 0.607646484375
3400 0.63181640625
3450 0.609111328125
3500 0.634990234375
3550 0.623515625
3600 0.618876953125
3650 0.62888671875
3700 0.6381640625
3750 0.63181640625
3800 0.638408203125
3850 0.64060546875
3900 0.625224609375
3950 0.636943359375
4000 0.64841796875
4050 0.63767578125
4100 0.649150390625
4150 0.635966796875
4200 0.658671875
4250 0.646220703125
4300 0.654765625
4350 0.65134765625
4400 0.637919921875
4450 0.64890625
4500 0.647685546875
4550 0.663798828125
4600 0.660380859375
4650 0.6371875
4700 0.64744140625
4750 0.654033203125
4800 0.678935546875
4850 0.66892578125
4900 0.675029296875
4950 0.66453125
5000 0.674296875
5050 0.658671875
5100 0.667705078125
5150 0.671123046875
5200 0.6655078125
5250 0.66697265625
5300 0.670146484375
5350 0.65671875
5400 0.6752734375
5450 0.668681640625
5500 0.690654296875
5550 0.68552734375
5600 0.68796875
5650 0.684306640625
5700 0.67380859375
5750 0.688701171875
5800 0.670634765625
5850 0.685283203125
5900 0.67771484375
5950 0.700419921875
6000 0.69529296875
6050 0.6996875
6100 0.680888671875
6150 0.68162109375
6200 0.69822265625
6250 0.69431640625
6300 0.698466796875
6350 0.698466796875
6400 0.711162109375
6450 0.69431640625
6500 0.70994140625
6550 0.703349609375
6600 0.69724609375
6650 0.681376953125
6700 0.705791015625
6750 0.7075
6800 0.70798828125
6850 0.717265625
6900 0.69578125
6950 0.70212890625
7000 0.701884765625
7050 0.712138671875
7100 0.71580078125
7150 0.713115234375
7200 0.7241015625
7250 0.699443359375
7300 0.705546875
7350 0.71140625
7400 0.69529296875
7450 0.713359375
7500 0.725810546875
7550 0.709208984375
7600 0.724345703125
7650 0.7084765625
7700 0.713359375
7750 0.722880859375
7800 0.691875
7850 0.72849609375
7900 0.71091796875
7950 0.71775390625
8000 0.736064453125
8050 0.7123828125
8100 0.715068359375
8150 0.711650390625
8200 0.720927734375
8250 0.7221484375
8300 0.728740234375
8350 0.728740234375
8400 0.727763671875
8450 0.725810546875
8500 0.73826171875
8550 0.721904296875
8600 0.714091796875
8650 0.734599609375
8700 0.754619140625
8750 0.72166015625
8800 0.730693359375
8850 0.7455859375
8900 0.74119140625
8950 0.738017578125
9000 0.742900390625
9050 0.7436328125
9100 0.760234375
9150 0.7397265625
9200 0.74607421875
9250 0.757060546875
9300 0.758037109375
9350 0.750712890625
9400 0.739970703125
9450 0.7377734375
9500 0.7436328125
9550 0.748271484375
9600 0.725078125
9650 0.747783203125
9700 0.7436328125
9750 0.756572265625
9800 0.7533984375
9850 0.746806640625
9900 0.74900390625
9950 0.74509765625
10000 0.74314453125
10050 0.74705078125
10100 0.76560546875
10150 0.735087890625
10200 0.762919921875
10250 0.759501953125
10300 0.75095703125
10350 0.74900390625
10400 0.761455078125
10450 0.739482421875
10500 0.759990234375
10550 0.769755859375
10600 0.75779296875
10650 0.758037109375
10700 0.76267578125
10750 0.76658203125
10800 0.76072265625
10850 0.79001953125
10900 0.762919921875
10950 0.7612109375
11000 0.75828125
11050 0.752177734375
11100 0.767314453125
11150 0.774638671875
11200 0.75486328125
11250 0.779521484375
11300 0.76658203125
11350 0.780986328125
11400 0.787578125
11450 0.760966796875
11500 0.779033203125
11550 0.768291015625
11600 0.7631640625
11650 0.76755859375
11700 0.76853515625
11750 0.7768359375
11800 0.759013671875
11850 0.78220703125
11900 0.778056640625
11950 0.774150390625
12000 0.785380859375
12050 0.78611328125
12100 0.776591796875
12150 0.75779296875
12200 0.771953125
12250 0.785625
12300 0.771953125
12350 0.7709765625
12400 0.780498046875
12450 0.785869140625
12500 0.78513671875
12550 0.77634765625
12600 0.782939453125
12650 0.769267578125
12700 0.768779296875
12750 0.7846484375
12800 0.779521484375
12850 0.790751953125
12900 0.78513671875
12950 0.76951171875
13000 0.76267578125
13050 0.7768359375
13100 0.787578125
13150 0.789287109375
13200 0.7866015625
13250 0.795390625
13300 0.79734375
13350 0.791484375
13400 0.789775390625
13450 0.781474609375
13500 0.7826953125
13550 0.79099609375
13600 0.781474609375
13650 0.779521484375
13700 0.781962890625
13750 0.788310546875
13800 0.79392578125
13850 0.7846484375
13900 0.772685546875
13950 0.797587890625
14000 0.788798828125
14050 0.773662109375
14100 0.7866015625
14150 0.78513671875
14200 0.80271484375
14250 0.78171875
14300 0.795146484375
14350 0.788798828125
14400 0.7905078125
14450 0.792216796875
14500 0.802958984375
14550 0.783427734375
14600 0.807841796875
14650 0.796611328125
14700 0.791240234375
14750 0.805888671875
14800 0.80857421875
14850 0.80515625
14900 0.7885546875
14950 0.81296875
15000 0.7885546875
15050 0.790751953125
15100 0.79392578125
15150 0.811015625
15200 0.804423828125
15250 0.800517578125
15300 0.804423828125
15350 0.801005859375
15400 0.807109375
15450 0.803203125
15500 0.8041796875
15550 0.804423828125
15600 0.800517578125
15650 0.798076171875
15700 0.792216796875
15750 0.7944140625
15800 0.79490234375
15850 0.80125
15900 0.81248046875
15950 0.800517578125
16000 0.800517578125
16050 0.78904296875
16100 0.79880859375
16150 0.801494140625
16200 0.7963671875
16250 0.80564453125
16300 0.799052734375
16350 0.80759765625
16400 0.8080859375
16450 0.790751953125
16500 0.78416015625
16550 0.79197265625
16600 0.8119921875
16650 0.797099609375
16700 0.80466796875
16750 0.805888671875
16800 0.800517578125
16850 0.80955078125
16900 0.795634765625
16950 0.79783203125
17000 0.80515625
17050 0.803935546875
17100 0.795390625
17150 0.803447265625
17200 0.8119921875
17250 0.780986328125
17300 0.8217578125
17350 0.802958984375
17400 0.813212890625
17450 0.81248046875
17500 0.811748046875
17550 0.80564453125
17600 0.813701171875
17650 0.817607421875
17700 0.80125
17750 0.8158984375
17800 0.80515625
17850 0.8022265625
17900 0.80759765625
17950 0.78904296875
18000 0.8080859375
18050 0.806865234375
18100 0.8041796875
18150 0.816142578125
18200 0.815654296875
18250 0.80515625
18300 0.799296875
18350 0.802470703125
18400 0.8100390625
18450 0.80857421875
18500 0.816875
18550 0.812236328125
18600 0.815166015625
18650 0.8139453125
18700 0.81541015625
18750 0.814921875
18800 0.80466796875
18850 0.79978515625
18900 0.79978515625
18950 0.79392578125
19000 0.813701171875
19050 0.81248046875
19100 0.801982421875
19150 0.79880859375
19200 0.79783203125
19250 0.8217578125
19300 0.826396484375
19350 0.82810546875
19400 0.8295703125
19450 0.83103515625
19500 0.82859375
19550 0.833232421875
19600 0.834208984375
19650 0.8295703125
19700 0.83396484375
19750 0.82908203125
19800 0.833720703125
19850 0.8334765625
19900 0.831767578125
19950 0.82810546875
20000 0.8256640625
20050 0.834208984375
20100 0.830302734375
20150 0.83298828125
20200 0.834453125
20250 0.833720703125
20300 0.83494140625
20350 0.83201171875
20400 0.83396484375
20450 0.8295703125
20500 0.83005859375
20550 0.83494140625
20600 0.833720703125
20650 0.83494140625
20700 0.82859375
20750 0.8315234375
20800 0.834697265625
20850 0.830546875
20900 0.829326171875
20950 0.834697265625
21000 0.830546875
21050 0.82810546875
21100 0.839580078125
21150 0.8325
21200 0.83005859375
21250 0.830302734375
21300 0.83591796875
21350 0.833232421875
21400 0.832255859375
21450 0.83103515625
21500 0.83396484375
21550 0.83494140625
21600 0.833720703125
21650 0.836162109375
21700 0.83494140625
21750 0.834453125
21800 0.837138671875
21850 0.8315234375
21900 0.8354296875
21950 0.837138671875
22000 0.833720703125
22050 0.837138671875
22100 0.83396484375
22150 0.82859375
22200 0.83591796875
22250 0.83884765625
22300 0.83640625
22350 0.833720703125
22400 0.834208984375
22450 0.832744140625
22500 0.829814453125
22550 0.837626953125
22600 0.83494140625
22650 0.83689453125
22700 0.831767578125
22750 0.838359375
22800 0.83494140625
22850 0.83787109375
22900 0.836162109375
22950 0.832744140625
23000 0.835185546875
23050 0.837138671875
23100 0.834453125
23150 0.83591796875
23200 0.838359375
23250 0.839091796875
23300 0.830546875
23350 0.830791015625
23400 0.838603515625
23450 0.83591796875
23500 0.837626953125
23550 0.832744140625
23600 0.83689453125
23650 0.838115234375
23700 0.838603515625
23750 0.834697265625
23800 0.835673828125
23850 0.834208984375
23900 0.83201171875
23950 0.8354296875
24000 0.837138671875
24050 0.83396484375
24100 0.8325
24150 0.8354296875
24200 0.83982421875
24250 0.828837890625
24300 0.83494140625
24350 0.83201171875
24400 0.837626953125
24450 0.835673828125
24500 0.836162109375
24550 0.829814453125
24600 0.833720703125
24650 0.834697265625
24700 0.83640625
24750 0.837626953125
24800 0.835673828125
24850 0.8373828125
24900 0.836162109375
24950 0.834453125
25000 0.8325
};
\addlegendentry{test clean}
\addplot [semithick, forestgreen4416044]
table {%
2500 0.571044921875
2550 0.583251953125
2600 0.55810546875
2650 0.568359375
2700 0.564697265625
2750 0.568359375
2800 0.55517578125
2850 0.5703125
2900 0.576171875
2950 0.581787109375
3000 0.57470703125
3050 0.579345703125
3100 0.575927734375
3150 0.60498046875
3200 0.5947265625
3250 0.591552734375
3300 0.592041015625
3350 0.5888671875
3400 0.61181640625
3450 0.58837890625
3500 0.60693359375
3550 0.602294921875
3600 0.5810546875
3650 0.6064453125
3700 0.6279296875
3750 0.614013671875
3800 0.622314453125
3850 0.617919921875
3900 0.607421875
3950 0.617919921875
4000 0.63134765625
4050 0.61376953125
4100 0.633544921875
4150 0.612060546875
4200 0.631591796875
4250 0.627685546875
4300 0.62890625
4350 0.628662109375
4400 0.609130859375
4450 0.626708984375
4500 0.623291015625
4550 0.63037109375
4600 0.62890625
4650 0.616943359375
4700 0.62548828125
4750 0.63037109375
4800 0.6591796875
4850 0.653564453125
4900 0.66748046875
4950 0.64208984375
5000 0.67138671875
5050 0.650634765625
5100 0.652099609375
5150 0.664306640625
5200 0.664794921875
5250 0.660400390625
5300 0.650634765625
5350 0.64599609375
5400 0.68359375
5450 0.6640625
5500 0.697265625
5550 0.680908203125
5600 0.689453125
5650 0.66748046875
5700 0.663330078125
5750 0.681640625
5800 0.656494140625
5850 0.667236328125
5900 0.65234375
5950 0.673095703125
6000 0.69189453125
6050 0.69384765625
6100 0.66796875
6150 0.6650390625
6200 0.6875
6250 0.684814453125
6300 0.68603515625
6350 0.696044921875
6400 0.70947265625
6450 0.68359375
6500 0.70068359375
6550 0.701171875
6600 0.697998046875
6650 0.65673828125
6700 0.69189453125
6750 0.695068359375
6800 0.695556640625
6850 0.70458984375
6900 0.683837890625
6950 0.71142578125
7000 0.697021484375
7050 0.7080078125
7100 0.701171875
7150 0.704345703125
7200 0.71337890625
7250 0.68603515625
7300 0.6982421875
7350 0.709716796875
7400 0.68603515625
7450 0.707763671875
7500 0.72900390625
7550 0.718017578125
7600 0.7216796875
7650 0.70654296875
7700 0.722412109375
7750 0.72216796875
7800 0.69482421875
7850 0.72900390625
7900 0.7109375
7950 0.710693359375
8000 0.732421875
8050 0.710205078125
8100 0.712158203125
8150 0.702392578125
8200 0.72412109375
8250 0.721435546875
8300 0.717041015625
8350 0.72412109375
8400 0.716796875
8450 0.733642578125
8500 0.7509765625
8550 0.732421875
8600 0.723876953125
8650 0.729736328125
8700 0.747802734375
8750 0.705322265625
8800 0.72119140625
8850 0.748779296875
8900 0.740234375
8950 0.738037109375
9000 0.748291015625
9050 0.7509765625
9100 0.756591796875
9150 0.734375
9200 0.740234375
9250 0.76220703125
9300 0.75048828125
9350 0.743896484375
9400 0.730712890625
9450 0.736328125
9500 0.742919921875
9550 0.746337890625
9600 0.716064453125
9650 0.74609375
9700 0.744873046875
9750 0.74658203125
9800 0.771484375
9850 0.751708984375
9900 0.748291015625
9950 0.751220703125
10000 0.7578125
10050 0.745361328125
10100 0.76611328125
10150 0.74609375
10200 0.78125
10250 0.778564453125
10300 0.767578125
10350 0.771728515625
10400 0.781982421875
10450 0.74462890625
10500 0.76025390625
10550 0.77392578125
10600 0.7646484375
10650 0.757080078125
10700 0.77197265625
10750 0.7734375
10800 0.771240234375
10850 0.797607421875
10900 0.769287109375
10950 0.775634765625
11000 0.773193359375
11050 0.7646484375
11100 0.77490234375
11150 0.79638671875
11200 0.7861328125
11250 0.80078125
11300 0.790771484375
11350 0.795166015625
11400 0.789306640625
11450 0.778076171875
11500 0.78369140625
11550 0.78369140625
11600 0.775634765625
11650 0.779541015625
11700 0.783203125
11750 0.798583984375
11800 0.773681640625
11850 0.797607421875
11900 0.79833984375
11950 0.786865234375
12000 0.794189453125
12050 0.7958984375
12100 0.791748046875
12150 0.77294921875
12200 0.78173828125
12250 0.785888671875
12300 0.79248046875
12350 0.7890625
12400 0.791015625
12450 0.7861328125
12500 0.820068359375
12550 0.802490234375
12600 0.80224609375
12650 0.79150390625
12700 0.79248046875
12750 0.798583984375
12800 0.792724609375
12850 0.795166015625
12900 0.804443359375
12950 0.785888671875
13000 0.7822265625
13050 0.78857421875
13100 0.7890625
13150 0.795166015625
13200 0.794189453125
13250 0.815185546875
13300 0.811767578125
13350 0.8115234375
13400 0.80078125
13450 0.806396484375
13500 0.8076171875
13550 0.804443359375
13600 0.794677734375
13650 0.82080078125
13700 0.81640625
13750 0.820556640625
13800 0.8251953125
13850 0.8173828125
13900 0.810546875
13950 0.820556640625
14000 0.8115234375
14050 0.793701171875
14100 0.8037109375
14150 0.796875
14200 0.81982421875
14250 0.818115234375
14300 0.82470703125
14350 0.8173828125
14400 0.806396484375
14450 0.815673828125
14500 0.828857421875
14550 0.80078125
14600 0.827880859375
14650 0.8203125
14700 0.8134765625
14750 0.82275390625
14800 0.83056640625
14850 0.810546875
14900 0.80322265625
14950 0.8232421875
15000 0.818359375
15050 0.8154296875
15100 0.814208984375
15150 0.82470703125
15200 0.83154296875
15250 0.82861328125
15300 0.82958984375
15350 0.8251953125
15400 0.84130859375
15450 0.824462890625
15500 0.8330078125
15550 0.825927734375
15600 0.845703125
15650 0.833984375
15700 0.822509765625
15750 0.822265625
15800 0.82177734375
15850 0.826171875
15900 0.831787109375
15950 0.845703125
16000 0.841796875
16050 0.8232421875
16100 0.836181640625
16150 0.839111328125
16200 0.82763671875
16250 0.836669921875
16300 0.826171875
16350 0.8427734375
16400 0.835205078125
16450 0.81982421875
16500 0.821533203125
16550 0.8310546875
16600 0.84619140625
16650 0.8251953125
16700 0.829345703125
16750 0.845458984375
16800 0.834716796875
16850 0.8427734375
16900 0.8330078125
16950 0.8427734375
17000 0.850341796875
17050 0.84814453125
17100 0.849853515625
17150 0.855712890625
17200 0.857421875
17250 0.828369140625
17300 0.85595703125
17350 0.840087890625
17400 0.83544921875
17450 0.839599609375
17500 0.848388671875
17550 0.84228515625
17600 0.84423828125
17650 0.846435546875
17700 0.856201171875
17750 0.86083984375
17800 0.848876953125
17850 0.848876953125
17900 0.8466796875
17950 0.825927734375
18000 0.8427734375
18050 0.847900390625
18100 0.849365234375
18150 0.862060546875
18200 0.862548828125
18250 0.848876953125
18300 0.849609375
18350 0.847900390625
18400 0.85595703125
18450 0.853759765625
18500 0.8603515625
18550 0.849853515625
18600 0.85205078125
18650 0.856689453125
18700 0.84130859375
18750 0.852294921875
18800 0.844970703125
18850 0.857421875
18900 0.8466796875
18950 0.83447265625
19000 0.85205078125
19050 0.85693359375
19100 0.84033203125
19150 0.8388671875
19200 0.85302734375
19250 0.875732421875
19300 0.875732421875
19350 0.87353515625
19400 0.876953125
19450 0.8818359375
19500 0.876220703125
19550 0.879638671875
19600 0.8759765625
19650 0.87548828125
19700 0.8740234375
19750 0.874267578125
19800 0.88232421875
19850 0.875732421875
19900 0.875244140625
19950 0.874267578125
20000 0.875
20050 0.884521484375
20100 0.87744140625
20150 0.875244140625
20200 0.879638671875
20250 0.87060546875
20300 0.875732421875
20350 0.87451171875
20400 0.882080078125
20450 0.884765625
20500 0.88232421875
20550 0.876953125
20600 0.878662109375
20650 0.87939453125
20700 0.875732421875
20750 0.8740234375
20800 0.876220703125
20850 0.8818359375
20900 0.87744140625
20950 0.8779296875
21000 0.8701171875
21050 0.875244140625
21100 0.87158203125
21150 0.87353515625
21200 0.874755859375
21250 0.8720703125
21300 0.87353515625
21350 0.88427734375
21400 0.8779296875
21450 0.879150390625
21500 0.884521484375
21550 0.87841796875
21600 0.87548828125
21650 0.877685546875
21700 0.882568359375
21750 0.888427734375
21800 0.886962890625
21850 0.88671875
21900 0.888671875
21950 0.884033203125
22000 0.887451171875
22050 0.88720703125
22100 0.880126953125
22150 0.881591796875
22200 0.883544921875
22250 0.88818359375
22300 0.88525390625
22350 0.880126953125
22400 0.884765625
22450 0.88232421875
22500 0.884521484375
22550 0.884033203125
22600 0.888427734375
22650 0.88623046875
22700 0.8916015625
22750 0.8876953125
22800 0.882568359375
22850 0.885498046875
22900 0.884033203125
22950 0.88134765625
23000 0.89013671875
23050 0.888916015625
23100 0.8935546875
23150 0.88916015625
23200 0.892822265625
23250 0.888916015625
23300 0.876220703125
23350 0.88134765625
23400 0.880859375
23450 0.8916015625
23500 0.890380859375
23550 0.88818359375
23600 0.8916015625
23650 0.8896484375
23700 0.890625
23750 0.888916015625
23800 0.88818359375
23850 0.884521484375
23900 0.883056640625
23950 0.884033203125
24000 0.88671875
24050 0.8857421875
24100 0.88232421875
24150 0.88916015625
24200 0.892333984375
24250 0.88818359375
24300 0.886962890625
24350 0.88671875
24400 0.892578125
24450 0.892822265625
24500 0.89208984375
24550 0.88818359375
24600 0.89111328125
24650 0.8857421875
24700 0.891845703125
24750 0.887939453125
24800 0.8896484375
24850 0.888671875
24900 0.893310546875
24950 0.88916015625
25000 0.89208984375
};
\addlegendentry{train clean}
\addplot [semithick, crimson2143940]
table {%
2500 0.286865234375
2550 0.31201171875
2600 0.29052734375
2650 0.297607421875
2700 0.318115234375
2750 0.302490234375
2800 0.2900390625
2850 0.29638671875
2900 0.339111328125
2950 0.310546875
3000 0.33154296875
3050 0.303466796875
3100 0.31591796875
3150 0.32080078125
3200 0.320556640625
3250 0.325439453125
3300 0.338134765625
3350 0.324951171875
3400 0.332275390625
3450 0.30908203125
3500 0.340576171875
3550 0.33544921875
3600 0.32568359375
3650 0.334228515625
3700 0.32958984375
3750 0.3466796875
3800 0.342041015625
3850 0.33935546875
3900 0.317138671875
3950 0.34814453125
4000 0.3388671875
4050 0.320556640625
4100 0.34033203125
4150 0.3349609375
4200 0.326904296875
4250 0.343505859375
4300 0.344482421875
4350 0.35498046875
4400 0.346435546875
4450 0.343505859375
4500 0.343994140625
4550 0.339599609375
4600 0.338623046875
4650 0.359619140625
4700 0.350830078125
4750 0.35400390625
4800 0.35205078125
4850 0.353271484375
4900 0.356689453125
4950 0.36083984375
5000 0.378662109375
5050 0.340576171875
5100 0.37939453125
5150 0.382568359375
5200 0.390869140625
5250 0.364013671875
5300 0.348876953125
5350 0.378173828125
5400 0.38916015625
5450 0.38037109375
5500 0.37841796875
5550 0.380126953125
5600 0.39208984375
5650 0.36181640625
5700 0.350341796875
5750 0.383056640625
5800 0.3681640625
5850 0.364990234375
5900 0.357177734375
5950 0.359375
6000 0.37353515625
6050 0.37353515625
6100 0.37451171875
6150 0.37939453125
6200 0.37890625
6250 0.37548828125
6300 0.37353515625
6350 0.3857421875
6400 0.3779296875
6450 0.380859375
6500 0.3828125
6550 0.390380859375
6600 0.383544921875
6650 0.38427734375
6700 0.3818359375
6750 0.39013671875
6800 0.37060546875
6850 0.369873046875
6900 0.38623046875
6950 0.396728515625
7000 0.400390625
7050 0.396728515625
7100 0.393798828125
7150 0.392333984375
7200 0.394287109375
7250 0.376953125
7300 0.3876953125
7350 0.4033203125
7400 0.3818359375
7450 0.38916015625
7500 0.4208984375
7550 0.410400390625
7600 0.404296875
7650 0.40185546875
7700 0.418701171875
7750 0.41943359375
7800 0.3876953125
7850 0.3974609375
7900 0.399169921875
7950 0.41259765625
8000 0.4013671875
8050 0.381591796875
8100 0.405517578125
8150 0.401123046875
8200 0.407470703125
8250 0.40185546875
8300 0.41015625
8350 0.3798828125
8400 0.40185546875
8450 0.412353515625
8500 0.412841796875
8550 0.400146484375
8600 0.4033203125
8650 0.41259765625
8700 0.402099609375
8750 0.418701171875
8800 0.41455078125
8850 0.4189453125
8900 0.408203125
8950 0.410400390625
9000 0.406005859375
9050 0.438720703125
9100 0.427734375
9150 0.424072265625
9200 0.421875
9250 0.428955078125
9300 0.401123046875
9350 0.41259765625
9400 0.397705078125
9450 0.409423828125
9500 0.414794921875
9550 0.42529296875
9600 0.39453125
9650 0.414794921875
9700 0.40625
9750 0.423095703125
9800 0.439453125
9850 0.439208984375
9900 0.424560546875
9950 0.41015625
10000 0.452392578125
10050 0.430419921875
10100 0.4345703125
10150 0.4208984375
10200 0.45556640625
10250 0.42724609375
10300 0.44775390625
10350 0.425048828125
10400 0.432373046875
10450 0.4287109375
10500 0.4287109375
10550 0.418701171875
10600 0.441650390625
10650 0.442626953125
10700 0.435302734375
10750 0.43359375
10800 0.44140625
10850 0.433837890625
10900 0.43798828125
10950 0.434814453125
11000 0.438720703125
11050 0.432373046875
11100 0.43603515625
11150 0.472412109375
11200 0.453857421875
11250 0.447265625
11300 0.446533203125
11350 0.43701171875
11400 0.43603515625
11450 0.426513671875
11500 0.430908203125
11550 0.459716796875
11600 0.443603515625
11650 0.4404296875
11700 0.4404296875
11750 0.469482421875
11800 0.428955078125
11850 0.435546875
11900 0.440185546875
11950 0.4423828125
12000 0.44580078125
12050 0.43408203125
12100 0.451904296875
12150 0.4453125
12200 0.45263671875
12250 0.449951171875
12300 0.475830078125
12350 0.46826171875
12400 0.448974609375
12450 0.439697265625
12500 0.48681640625
12550 0.460205078125
12600 0.4482421875
12650 0.468505859375
12700 0.481689453125
12750 0.458984375
12800 0.452392578125
12850 0.45556640625
12900 0.4775390625
12950 0.429443359375
13000 0.44384765625
13050 0.4541015625
13100 0.450439453125
13150 0.463134765625
13200 0.446533203125
13250 0.455322265625
13300 0.486328125
13350 0.465087890625
13400 0.44677734375
13450 0.484130859375
13500 0.454345703125
13550 0.48583984375
13600 0.46484375
13650 0.470458984375
13700 0.482421875
13750 0.478271484375
13800 0.477294921875
13850 0.48876953125
13900 0.477294921875
13950 0.48388671875
14000 0.46826171875
14050 0.4619140625
14100 0.4462890625
14150 0.451904296875
14200 0.458984375
14250 0.478271484375
14300 0.489990234375
14350 0.4892578125
14400 0.46533203125
14450 0.48974609375
14500 0.4853515625
14550 0.454833984375
14600 0.46728515625
14650 0.460693359375
14700 0.471435546875
14750 0.462158203125
14800 0.4736328125
14850 0.476806640625
14900 0.468505859375
14950 0.467529296875
15000 0.50732421875
15050 0.486083984375
15100 0.472900390625
15150 0.46728515625
15200 0.485595703125
15250 0.47802734375
15300 0.487060546875
15350 0.48046875
15400 0.494384765625
15450 0.49609375
15500 0.477783203125
15550 0.495361328125
15600 0.494140625
15650 0.487548828125
15700 0.46875
15750 0.476806640625
15800 0.489990234375
15850 0.470947265625
15900 0.478515625
15950 0.5126953125
16000 0.510009765625
16050 0.484375
16100 0.502197265625
16150 0.491455078125
16200 0.48828125
16250 0.47705078125
16300 0.472900390625
16350 0.505126953125
16400 0.494873046875
16450 0.467041015625
16500 0.4873046875
16550 0.501220703125
16600 0.501220703125
16650 0.490478515625
16700 0.49755859375
16750 0.508544921875
16800 0.49365234375
16850 0.495849609375
16900 0.510498046875
16950 0.5078125
17000 0.490966796875
17050 0.48876953125
17100 0.504150390625
17150 0.51513671875
17200 0.491943359375
17250 0.49072265625
17300 0.516845703125
17350 0.501953125
17400 0.501953125
17450 0.493408203125
17500 0.523681640625
17550 0.5146484375
17600 0.5107421875
17650 0.5107421875
17700 0.513427734375
17750 0.5029296875
17800 0.500732421875
17850 0.48681640625
17900 0.517578125
17950 0.48974609375
18000 0.497802734375
18050 0.492919921875
18100 0.500244140625
18150 0.504150390625
18200 0.500244140625
18250 0.521728515625
18300 0.510498046875
18350 0.49267578125
18400 0.51123046875
18450 0.52978515625
18500 0.52197265625
18550 0.508056640625
18600 0.50244140625
18650 0.5146484375
18700 0.501708984375
18750 0.492919921875
18800 0.486572265625
18850 0.51025390625
18900 0.510498046875
18950 0.50634765625
19000 0.49072265625
19050 0.5146484375
19100 0.51123046875
19150 0.504638671875
19200 0.5009765625
19250 0.55126953125
19300 0.5556640625
19350 0.55615234375
19400 0.554443359375
19450 0.565185546875
19500 0.552978515625
19550 0.565673828125
19600 0.5546875
19650 0.56005859375
19700 0.552978515625
19750 0.557373046875
19800 0.56591796875
19850 0.553955078125
19900 0.559326171875
19950 0.557373046875
20000 0.56201171875
20050 0.5732421875
20100 0.566162109375
20150 0.56591796875
20200 0.5576171875
20250 0.5537109375
20300 0.560302734375
20350 0.548095703125
20400 0.57763671875
20450 0.56884765625
20500 0.56689453125
20550 0.563720703125
20600 0.564208984375
20650 0.569091796875
20700 0.57275390625
20750 0.568359375
20800 0.571533203125
20850 0.570556640625
20900 0.565185546875
20950 0.571533203125
21000 0.56103515625
21050 0.55908203125
21100 0.55322265625
21150 0.576171875
21200 0.571533203125
21250 0.585205078125
21300 0.57763671875
21350 0.559326171875
21400 0.568115234375
21450 0.560546875
21500 0.5576171875
21550 0.564208984375
21600 0.557861328125
21650 0.56884765625
21700 0.579833984375
21750 0.59130859375
21800 0.5888671875
21850 0.58203125
21900 0.57275390625
21950 0.5810546875
22000 0.570068359375
22050 0.579345703125
22100 0.579833984375
22150 0.581787109375
22200 0.570068359375
22250 0.576171875
22300 0.582763671875
22350 0.578369140625
22400 0.568115234375
22450 0.57373046875
22500 0.58642578125
22550 0.578125
22600 0.586669921875
22650 0.583984375
22700 0.5673828125
22750 0.571044921875
22800 0.569091796875
22850 0.589599609375
22900 0.591552734375
22950 0.59765625
23000 0.5888671875
23050 0.5927734375
23100 0.587646484375
23150 0.5859375
23200 0.5849609375
23250 0.589111328125
23300 0.58447265625
23350 0.57373046875
23400 0.57568359375
23450 0.592529296875
23500 0.590576171875
23550 0.587646484375
23600 0.584228515625
23650 0.594482421875
23700 0.59326171875
23750 0.584228515625
23800 0.580810546875
23850 0.578125
23900 0.572265625
23950 0.574462890625
24000 0.572509765625
24050 0.574462890625
24100 0.57421875
24150 0.57763671875
24200 0.585205078125
24250 0.58349609375
24300 0.588134765625
24350 0.578125
24400 0.58154296875
24450 0.584228515625
24500 0.589599609375
24550 0.58447265625
24600 0.590576171875
24650 0.590087890625
24700 0.593505859375
24750 0.591796875
24800 0.6005859375
24850 0.603271484375
24900 0.60595703125
24950 0.59814453125
25000 0.5908203125
};
\addlegendentry{train robust}
\end{axis}

\end{tikzpicture}

%% file: ICLR_appendix_experiments/neg_ratio_wholeTraining.tex
\begin{tikzpicture}

\definecolor{darkgray176}{RGB}{176,176,176}
\definecolor{darkorange25512714}{RGB}{255,127,14}
\definecolor{lightgray204}{RGB}{204,204,204}
\definecolor{steelblue31119180}{RGB}{31,119,180}

\begin{axis}[
legend cell align={left},
legend style={fill opacity=0.8, draw opacity=1, text opacity=1, draw=lightgray204},
tick align=outside,
tick pos=left,
x grid style={darkgray176},
xlabel={iteration},
xmin=1375, xmax=26125,
xtick style={color=black},
y grid style={darkgray176},
ylabel={proportion},
ymin=0.104806615998552, ymax=0.41586644651451,
ytick style={color=black}
]
\addplot [semithick, steelblue31119180]
table {%
2500 0.388691232107264
2550 0.387947228237588
2600 0.38654897804616
2650 0.384668183553244
2700 0.382533679839356
2750 0.380406928643635
2800 0.378553927755436
2850 0.377197663840732
2900 0.376500309795516
2950 0.376538408742158
3000 0.37729523864608
3050 0.378670718063892
3100 0.380507424845948
3150 0.382622034817419
3200 0.384839368244665
3250 0.387028922966793
3300 0.389113376840266
3350 0.391076822693447
3400 0.392950281907564
3450 0.394778121230432
3500 0.396577624538953
3550 0.3983164906776
3600 0.399885226484923
3650 0.401101452162176
3700 0.401727363309239
3750 0.401489694559751
3800 0.400142447178375
3850 0.397506603662251
3900 0.393509281905015
3950 0.38820856413535
4000 0.381794597872437
4050 0.374588630203971
4100 0.36698008308363
4150 0.35939964921944
4200 0.352259413357628
4250 0.345909858324373
4300 0.340610719350916
4350 0.336509380002635
4400 0.333641377155794
4450 0.331952349530708
4500 0.331308829464052
4550 0.331534551134644
4600 0.332436995363385
4650 0.333827774790053
4700 0.335535379028371
4750 0.337432511469905
4800 0.339410120141427
4850 0.341386148070524
4900 0.343302775850046
4950 0.345095544281582
5000 0.346702917361776
5050 0.34805948637634
5100 0.349092403670659
5150 0.349736887464351
5200 0.349943181045578
5250 0.349696697431234
5300 0.349018312810731
5350 0.347976134135488
5400 0.346674431983904
5450 0.345257115850686
5500 0.343893749943163
5550 0.342748064181212
5600 0.341967341334533
5650 0.341652782724593
5700 0.341852353559825
5750 0.34254892382411
5800 0.343682034455262
5850 0.345135350948839
5900 0.346782684002811
5950 0.34848533276487
6000 0.350104375430058
6050 0.351510590016056
6100 0.352580932628988
6150 0.353218723275297
6200 0.353344635499984
6250 0.352911098286413
6300 0.351917245946666
6350 0.350420479382204
6400 0.348524105247355
6450 0.346375643657557
6500 0.344143759246146
6550 0.341982326516272
6600 0.340021967641743
6650 0.338326501884014
6700 0.33690072702012
6750 0.335691575836952
6800 0.334599683286255
6850 0.333516156953828
6900 0.332343406520642
6950 0.331019358530879
7000 0.329535199337974
7050 0.327937732565831
7100 0.326322749864643
7150 0.324825538727924
7200 0.323598298510707
7250 0.322801262601513
7300 0.322578431561774
7350 0.323040392751754
7400 0.324252462109885
7450 0.326216177347062
7500 0.3288554507392
7550 0.332015209570461
7600 0.335467710815772
7650 0.338923148112051
7700 0.342080026083173
7750 0.344646382573502
7800 0.346398107654425
7850 0.347204578394269
7900 0.347057808029451
7950 0.346053564582187
8000 0.344388926831144
8050 0.342310566806658
8100 0.340070784401418
8150 0.337878507161385
8200 0.335868580315174
8250 0.334080099685706
8300 0.332489712376359
8350 0.331003250115204
8400 0.329514362722562
8450 0.327949216039154
8500 0.326272767832362
8550 0.324500983609638
8600 0.322698004497083
8650 0.320939766516027
8700 0.319289354734014
8750 0.317778866850312
8800 0.316391747438781
8850 0.315084814617164
8900 0.313786215475735
8950 0.312419801224316
9000 0.310927633674206
9050 0.309275627166645
9100 0.30745674877665
9150 0.305499097951015
9200 0.303459609865525
9250 0.301428591266222
9300 0.299509391540611
9350 0.297819201069251
9400 0.296457112226246
9450 0.295482083422786
9500 0.294884093119869
9550 0.294578603007559
9600 0.294403447666973
9650 0.294160649477687
9700 0.293655855219654
9750 0.292751346999242
9800 0.291418795468572
9850 0.289752191511024
9900 0.287961533898074
9950 0.286340907961113
10000 0.285203882129304
10050 0.284814679893427
10100 0.285339942240399
10150 0.286810378784421
10200 0.289125864141522
10250 0.292078981229613
10300 0.295395935099867
10350 0.298790671005197
10400 0.302015247686207
10450 0.304876288153927
10500 0.307251315974102
10550 0.30908426834376
10600 0.310366569247711
10650 0.311120969487951
10700 0.311395073004642
10750 0.311267193695049
10800 0.310840092080577
10850 0.310251143114324
10900 0.309662414618957
10950 0.309257677842006
11000 0.309203300957342
11050 0.309619997578449
11100 0.310546381289061
11150 0.311929249767297
11200 0.313594307444047
11250 0.315274161611283
11300 0.316629769889375
11350 0.317320697348162
11400 0.317047372079246
11450 0.315619800004125
11500 0.313002550182479
11550 0.309334282899316
11600 0.304913678951076
11650 0.300169433558078
11700 0.295587808365077
11750 0.291639799377752
11800 0.288711821368672
11850 0.287035853528642
11900 0.286667762614378
11950 0.287485327544266
12000 0.289207701566945
12050 0.291455962898956
12100 0.293809190309239
12150 0.295874723833786
12200 0.297336967264292
12250 0.29798764779368
12300 0.297737612311153
12350 0.29661139766458
12400 0.294702913434007
12450 0.292168548860887
12500 0.289173973719738
12550 0.28587559151879
12600 0.282417995780806
12650 0.278922397335293
12700 0.275498212492108
12750 0.272242522339203
12800 0.269252886628936
12850 0.266615999159996
12900 0.264416714230157
12950 0.26272669924515
13000 0.261592165542818
13050 0.261040046722502
13100 0.261057382114453
13150 0.261605377291356
13200 0.262597840427188
13250 0.263929515800866
13300 0.26547820467079
13350 0.267123414114637
13400 0.268754300161161
13450 0.270299536359474
13500 0.271726915121713
13550 0.273035682176221
13600 0.274264840355339
13650 0.27545788077472
13700 0.276659793082909
13750 0.277889145738664
13800 0.279135796591457
13850 0.28036866663985
13900 0.281533570085654
13950 0.282573754695601
14000 0.283437481328705
14050 0.284108110256743
14100 0.284592223679297
14150 0.284930985141682
14200 0.285187550820161
14250 0.285434127800499
14300 0.285738161383502
14350 0.28613695759859
14400 0.28663226661914
14450 0.28718504874462
14500 0.287707148215641
14550 0.288095998294319
14600 0.288246856589117
14650 0.288075781540436
14700 0.287549909893027
14750 0.286692845114821
14800 0.285591305681609
14850 0.284380655960127
14900 0.283232156188466
14950 0.282308634164922
15000 0.281750699506207
15050 0.281644153404119
15100 0.282009981714589
15150 0.282803474009297
15200 0.283907960164329
15250 0.28517400157012
15300 0.286423698805755
15350 0.287476213546441
15400 0.288183071773957
15450 0.288430728444268
15500 0.288153514045087
15550 0.28733787926867
15600 0.28601639458337
15650 0.284255766156449
15700 0.282155725359651
15750 0.279822491721451
15800 0.277370288422625
15850 0.274899082491358
15900 0.27248891338838
15950 0.270201396671287
16000 0.268070418653378
16050 0.266103006394052
16100 0.2642946266617
16150 0.262640707700302
16200 0.26115068853175
16250 0.259857134520382
16300 0.258820523509861
16350 0.258115470592996
16400 0.257821051393245
16450 0.257990872788041
16500 0.258632846485257
16550 0.259701330440306
16600 0.261080569620306
16650 0.262606246640704
16700 0.26409050494832
16750 0.265341258504313
16800 0.266202440828476
16850 0.266564936931826
16900 0.266390396377861
16950 0.265696306361435
17000 0.264569865526592
17050 0.263136408054699
17100 0.261549579530018
17150 0.259962706368318
17200 0.258513958125243
17250 0.257311588026119
17300 0.256412056057863
17350 0.255831023228006
17400 0.255545603240902
17450 0.255490346902102
17500 0.25559280388196
17550 0.255772311380893
17600 0.255961808980743
17650 0.256112322293449
17700 0.25618930078108
17750 0.256183013947915
17800 0.256089946270973
17850 0.255913899673207
17900 0.255654975526182
17950 0.25529852856233
18000 0.254828352060395
18050 0.254209409180138
18100 0.253414764971685
18150 0.252405636583507
18200 0.251162769149717
18250 0.249675704127396
18300 0.247959066703815
18350 0.246042977928912
18400 0.243985394618034
18450 0.241852921061639
18500 0.239725406049139
18550 0.237678835408772
18600 0.235777898094316
18650 0.23407687359144
18700 0.232603226473699
18750 0.231359522820232
18800 0.230334100541664
18850 0.229495814694097
18900 0.228816604480678
18950 0.228263015794179
19000 0.227821211828408
19050 0.2274770459023
19100 0.227223503610965
19150 0.227056666936576
19200 0.226971321709121
19250 0.22695643428487
19300 0.226995754261871
19350 0.227071517402374
19400 0.227157734754418
19450 0.227217804679008
19500 0.227216293992854
19550 0.227105245920377
19600 0.22683035563859
19650 0.226333234019134
19700 0.225564682940367
19750 0.224491952554077
19800 0.223104199793326
19850 0.221428205881162
19900 0.219524265611934
19950 0.217482478043471
20000 0.215409828407757
20050 0.213411966687083
20100 0.211576794608873
20150 0.209947844560121
20200 0.208529429047803
20250 0.207278371583351
20300 0.206117578921482
20350 0.204955557430633
20400 0.203723911959062
20450 0.202379902157302
20500 0.200935510837307
20550 0.199460492882027
20600 0.198073539821819
20650 0.196922338316698
20700 0.196153284524457
20750 0.195900239103746
20800 0.196243580526036
20850 0.197204983380517
20900 0.198739477704736
20950 0.200745249176819
21000 0.203073502818811
21050 0.205560049242291
21100 0.208055442629764
21150 0.210445620376916
21200 0.212675854432685
21250 0.214753987814333
21300 0.21674900585213
21350 0.218765672334424
21400 0.220906357683481
21450 0.223250191810774
21500 0.225818527100264
21550 0.228557425119966
21600 0.231340584087531
21650 0.233992797868448
21700 0.236298940776919
21750 0.238048357209676
21800 0.239057661904832
21850 0.239202694693643
21900 0.238429146232465
21950 0.236751335267571
22000 0.234258067127397
22050 0.231085087459012
22100 0.227402414612343
22150 0.223396169921597
22200 0.219250375651468
22250 0.215133781828384
22300 0.211203631712636
22350 0.207598740298507
22400 0.204436030035816
22450 0.201798046913897
22500 0.199725154252943
22550 0.198205950441653
22600 0.197152720537418
22650 0.196427903933237
22700 0.195843514779053
22750 0.195205845256609
22800 0.1943409311123
22850 0.193149051144346
22900 0.191618830611479
22950 0.189836280623285
23000 0.187960738992934
23050 0.186198884739057
23100 0.184758477038244
23150 0.183804623004788
23200 0.183433765335971
23250 0.183647977674286
23300 0.184362488069279
23350 0.185427037323333
23400 0.186651724044421
23450 0.187847903680685
23500 0.188866412382225
23550 0.189621816316106
23600 0.190111271221809
23650 0.190416621683281
23700 0.190664293373565
23750 0.191018680848485
23800 0.191613881636085
23850 0.192536926593044
23900 0.193783135411483
23950 0.195273185668398
24000 0.196861737842416
24050 0.198369431117194
24100 0.19963605961018
24150 0.20055117854077
24200 0.201091698750148
24250 0.201314749349956
24300 0.20135885089186
24350 0.201395837395951
24400 0.201610356569258
24450 0.202148139031823
24500 0.203107991882848
24550 0.204524263090178
24600 0.206360076221872
24650 0.208536232318368
24700 0.210925670619196
24750 0.213384480315161
24800 0.215756248283756
24850 0.217887708309392
24900 0.219632652603701
24950 0.220868357733531
25000 0.221507307519578
};
\addlegendentry{AT}
\addplot [semithick, darkorange25512714]
table {%
2500 0.248843283104575
2550 0.248666945983544
2600 0.248354016183095
2650 0.247974235380596
2700 0.24759922464819
2750 0.247275757776261
2800 0.247005514597907
2850 0.246732531918683
2900 0.246358004031924
2950 0.245760184870202
3000 0.244826895264664
3050 0.24348652464595
3100 0.241732931271443
3150 0.239624812734792
3200 0.237280147506512
3250 0.234842281378757
3300 0.232452240667781
3350 0.230214930033507
3400 0.228186704423989
3450 0.226365759482349
3500 0.224713424431829
3550 0.223164686355681
3600 0.221667977696298
3650 0.220207265112376
3700 0.218804368694935
3750 0.217530430158655
3800 0.216487599955105
3850 0.215793706371611
3900 0.215549248333142
3950 0.21583383306478
4000 0.216670769654304
4050 0.21802669291218
4100 0.21981290582607
4150 0.221890708399445
4200 0.224104105129829
4250 0.226285621910881
4300 0.228298009800943
4350 0.23004051798248
4400 0.231469784597958
4450 0.232594424302606
4500 0.233466152887206
4550 0.234167297629821
4600 0.234799191836078
4650 0.235446854255221
4700 0.236177892072011
4750 0.237037919001312
4800 0.23803911080241
4850 0.239166639684553
4900 0.240383620957404
4950 0.241621366147147
5000 0.242803503354394
5050 0.243838640869558
5100 0.244629683030242
5150 0.245093773501838
5200 0.245166413263126
5250 0.244827154969776
5300 0.244101777520679
5350 0.243072004023137
5400 0.241856111456499
5450 0.240601817536928
5500 0.239451448700046
5550 0.238519384970079
5600 0.237863929063857
5650 0.237473172043789
5700 0.237263915505451
5750 0.237087636182564
5800 0.236758857155624
5850 0.236089081908554
5900 0.234918303217014
5950 0.233152542473534
6000 0.230781823600878
6050 0.227902167247162
6100 0.22470059963728
6150 0.2214329877436
6200 0.218395452087625
6250 0.215869702061676
6300 0.214086365973898
6350 0.213178821865799
6400 0.213172818693837
6450 0.213973395183739
6500 0.215382535554009
6550 0.217131312883984
6600 0.218918448338066
6650 0.220454532815369
6700 0.221498451996172
6750 0.221883936027457
6800 0.221525382258787
6850 0.220414483910615
6900 0.218601757673992
6950 0.216180837390028
7000 0.213260550097819
7050 0.20995539329828
7100 0.206384620328443
7150 0.202691052701424
7200 0.199037092986359
7250 0.195624576281471
7300 0.192672332778547
7350 0.190407697298111
7400 0.189031613906745
7450 0.188689362528463
7500 0.189436135069338
7550 0.191235137695449
7600 0.193958707714889
7650 0.197413779169799
7700 0.201368540223879
7750 0.205588113253563
7800 0.209857685659109
7850 0.213991253453146
7900 0.217820725497545
7950 0.221202086364543
8000 0.223994166977378
8050 0.226052013371525
8100 0.227244415131805
8150 0.22746591812288
8200 0.226660530701483
8250 0.224846532237192
8300 0.222136779043293
8350 0.218738018170233
8400 0.214934405338703
8450 0.211046419772162
8500 0.207392691701685
8550 0.204246724493263
8600 0.201794256769771
8650 0.200117650492554
8700 0.199196836943193
8750 0.198931979374355
8800 0.199170250811242
8850 0.199742358985582
8900 0.200491499233888
8950 0.201279853245767
9000 0.202005711044288
9050 0.202592039859088
9100 0.202983296473534
9150 0.20313461188905
9200 0.203004140149393
9250 0.202560914225321
9300 0.201775366149091
9350 0.200622752676392
9400 0.199088476244713
9450 0.197159365266819
9500 0.194845257451303
9550 0.192177155523707
9600 0.189217816497187
9650 0.186077357872847
9700 0.182906384869437
9750 0.179889836327158
9800 0.177229323599893
9850 0.175111949215473
9900 0.173686616272475
9950 0.173039787509465
10000 0.173177108130538
10050 0.174022937780588
10100 0.175432734396083
10150 0.177204337830849
10200 0.179115971516634
10250 0.180947631120483
10300 0.182509177534958
10350 0.183657949210551
10400 0.184313991794847
10450 0.184458022665378
10500 0.184130034721338
10550 0.183418538454915
10600 0.182435932625425
10650 0.181307940111742
10700 0.180153330956124
10750 0.179073570265752
10800 0.17813956492429
10850 0.177390199460957
10900 0.176828830414832
10950 0.17642722176088
11000 0.176143035559148
11050 0.175910628734663
11100 0.175670754007847
11150 0.175362690717312
11200 0.174944357330811
11250 0.174381484160068
11300 0.173649724712531
11350 0.172742247189929
11400 0.171651751983106
11450 0.170377771526275
11500 0.168931477974097
11550 0.167333883268256
11600 0.165621496380445
11650 0.163850200771796
11700 0.162091272841794
11750 0.160423324430888
11800 0.158931309221322
11850 0.157693794764189
11900 0.156771323894468
11950 0.15620272781326
12000 0.156003593107358
12050 0.15615498890379
12100 0.156618426562501
12150 0.157326570486132
12200 0.158198434218678
12250 0.159137795301808
12300 0.16004616363767
12350 0.160837624421057
12400 0.161439576432942
12450 0.161811865258503
12500 0.161939300499695
12550 0.161845155858959
12600 0.161582665719116
12650 0.161229719843022
12700 0.160883190879266
12750 0.160645716949425
12800 0.160606558850756
12850 0.160831245390555
12900 0.161351198835174
12950 0.162152254148626
13000 0.163170158251552
13050 0.164293082741515
13100 0.16538343071221
13150 0.166277476080694
13200 0.166823214076578
13250 0.16689307376261
13300 0.166410574223028
13350 0.165367241138402
13400 0.163823353630273
13450 0.161919220585065
13500 0.159851085693081
13550 0.157855020396329
13600 0.156171236878102
13650 0.155012283893904
13700 0.154527415329734
13750 0.154776565057705
13800 0.155724156229611
13850 0.157228822175023
13900 0.159072326602999
13950 0.16099631030951
14000 0.162723141148576
14050 0.164000679335429
14100 0.164638047582776
14150 0.164515209586619
14200 0.163599854149625
14250 0.161947881329128
14300 0.159689086264589
14350 0.157019164376397
14400 0.154166612754554
14450 0.151392310669611
14500 0.148944888897876
14550 0.14704891306355
14600 0.145873498097265
14650 0.145513942629397
14700 0.145983860857058
14750 0.14720855803607
14800 0.149038856558228
14850 0.151279317904946
14900 0.153708085270285
14950 0.156115606210821
15000 0.15833758470431
15050 0.160270093641651
15100 0.161878640908944
15150 0.163195665209768
15200 0.164309075203857
15250 0.165332965322942
15300 0.166388101009257
15350 0.167576203784823
15400 0.168964807440298
15450 0.170572399828844
15500 0.172372821100515
15550 0.174300085625559
15600 0.17626118386709
15650 0.178154661556111
15700 0.17988155615493
15750 0.181360916982304
15800 0.182535012408715
15850 0.183376168917928
15900 0.183885032078992
15950 0.184085694661124
16000 0.184011815857428
16050 0.183713688653984
16100 0.183227857600324
16150 0.182574153729296
16200 0.18174933135862
16250 0.18072912864313
16300 0.179467103939955
16350 0.177913469793872
16400 0.176034807123959
16450 0.173825612594706
16500 0.171307845236448
16550 0.168539366790428
16600 0.165601564350995
16650 0.162580038959796
16700 0.159560485857143
16750 0.15662449483073
16800 0.153840110326582
16850 0.151280741014173
16900 0.149020561449628
16950 0.147137868089238
17000 0.145692623666912
17050 0.144720957985428
17100 0.144216171225376
17150 0.144118430561137
17200 0.144320574733675
17250 0.144677074064009
17300 0.145037973224926
17350 0.145275350519295
17400 0.145304406155465
17450 0.145091922202151
17500 0.144664798996303
17550 0.144085625451101
17600 0.143442007426953
17650 0.142830359046117
17700 0.142337033435051
17750 0.14203807845116
17800 0.141993104393228
17850 0.142251983212509
17900 0.142856101569239
17950 0.143840186276247
18000 0.145225219934472
18050 0.147024680432525
18100 0.149226115574414
18150 0.151798964143379
18200 0.154674906038581
18250 0.157748707604312
18300 0.160879574833987
18350 0.163892275776938
18400 0.166588015963928
18450 0.168754263783722
18500 0.170205006825026
18550 0.170790553114715
18600 0.170433046636734
18650 0.169128847532564
18700 0.166961928510503
18750 0.164082618126998
18800 0.160683713424446
18850 0.156988242882653
18900 0.153211152679585
18950 0.149551932460377
19000 0.146186917129757
19050 0.14327238503997
19100 0.140940417164279
19150 0.13929759856315
19200 0.138418683841045
19250 0.138329335644747
19300 0.138996013803318
19350 0.140308632092546
19400 0.142095000586943
19450 0.144122286168514
19500 0.14614187226626
19550 0.147911948253998
19600 0.149235286161485
19650 0.149968177418093
19700 0.1500459595717
19750 0.149469734084869
19800 0.148297369026556
19850 0.14662143178321
19900 0.144565131959127
19950 0.142256528205327
20000 0.139828995162046
20050 0.137416114204104
20100 0.135137947126501
20150 0.133101742215682
20200 0.131388648055372
20250 0.130051428751385
20300 0.12911189280184
20350 0.128554993263021
20400 0.12834487947708
20450 0.12844159416248
20500 0.128808633812644
20550 0.129424205337526
20600 0.130297877049869
20650 0.131462320053156
20700 0.132966659373944
20750 0.134863528949477
20800 0.13719054484448
20850 0.139957712383994
20900 0.143117864958636
20950 0.146580502992841
21000 0.150199995280481
21050 0.153786588256394
21100 0.157126178911953
21150 0.159996009288082
21200 0.162200890590471
21250 0.163598493123203
21300 0.164119587547175
21350 0.163789258145815
21400 0.162718038275469
21450 0.161089422084748
21500 0.159134227090999
21550 0.157081627857917
21600 0.155130530077155
21650 0.153412737302725
21700 0.151982001397424
21750 0.150800890886523
21800 0.149766102926202
21850 0.14873103296989
21900 0.147546139947738
21950 0.14608426645784
22000 0.144273897561994
22050 0.142112723271631
22100 0.139674187950407
22150 0.137093460179472
22200 0.134546280396912
22250 0.132220827904014
22300 0.130274842615244
22350 0.128816764952995
22400 0.127886426658248
22450 0.127442398952985
22500 0.127382841136068
22550 0.127559200410936
22600 0.127813417859927
22650 0.127999579728026
22700 0.128005871954309
22750 0.127774499196154
22800 0.12729238043344
22850 0.126592166254489
22900 0.125737504080422
22950 0.124811142856179
23000 0.123890813884166
23050 0.123050603529472
23100 0.122351374054005
23150 0.121836945244853
23200 0.121525605646799
23250 0.121414734254858
23300 0.121475810961094
23350 0.121655778217872
23400 0.12188094066213
23450 0.122062964660585
23500 0.122122472093093
23550 0.121997705983046
23600 0.121663581557056
23650 0.121139712084661
23700 0.120494202131578
23750 0.119831693474109
23800 0.119273641061151
23850 0.118945699203823
23900 0.118948374216012
23950 0.119343960113184
24000 0.120152586190235
24050 0.121353738973542
24100 0.122889752913874
24150 0.124680049652744
24200 0.126636588442572
24250 0.128662799862614
24300 0.130666254009087
24350 0.132558802496569
24400 0.134249285116262
24450 0.135655639167612
24500 0.13670302920044
24550 0.137325232467208
24600 0.137479569244504
24650 0.137154971855996
24700 0.13637641178829
24750 0.135225776770122
24800 0.133831988452647
24850 0.132370249558345
24900 0.131040968804341
24950 0.130024833367415
25000 0.129474886098489
};
\addlegendentry{DyART}
\end{axis}

\end{tikzpicture}